\documentclass{article}
\usepackage{kr}

\usepackage{times}
\usepackage{soul}
\usepackage{url}
\usepackage[hidelinks]{hyperref}
\usepackage[utf8]{inputenc}
\usepackage[small]{caption}
\usepackage{graphicx}
\usepackage{amsmath}
\usepackage{amsthm}
\usepackage{booktabs}
\usepackage{algorithm}
\usepackage{algorithmic}
\newtheorem{example}{Example}
\newtheorem{theorem}{Theorem}

\title{Normalisations of Existential Rules: Not so Innocuous!}

\author{%
David Carral$^1$ \and
Lucas Larroque$^2$\footnote{Our work started when Lucas was intern at LIRMM-Inria.} \and
Marie-Laure Mugnier$^{1}$ \and
Micha\"el Thomazo$^3$ \\
\affiliations
$^1$LIRMM, Inria, University of Montpellier, CNRS, Montpellier, France\\
$^2$DI ENS, ENS, CNRS, PSL University, Paris, France\\
$^3$Inria, DI ENS, ENS, CNRS, PSL University, Paris, France\\
\emails
\{david.carral, michael.thomazo\}@inria.fr,
lucas.larroque@ens.psl.eu,
mugnier@lirmm.fr
}

\usepackage{amsfonts}
\usepackage{amsmath}
\usepackage{amssymb} 
\usepackage{tikz}
\usetikzlibrary{arrows,automata}
\usepackage{todonotes}
\usepackage{xspace}

\usepackage{comment}
\includecomment{tr}
\excludecomment{paper}

 \allowdisplaybreaks

\makeatletter
\newtheorem*{rep@theorem}{\rep@title}
\newcommand{\newreptheorem}[2]{%
\newenvironment{rep#1}[1]{%
 \def\rep@title{#2 \ref{##1}}%
 \begin{rep@theorem}}%
 {\end{rep@theorem}}}
\makeatother

\newreptheorem{theorem}{Theorem}
\newreptheorem{lemma}{Lemma}
\newreptheorem{proposition}{Proposition}

\theoremstyle{definition}
\newtheorem{definition}[theorem]{Definition}
\theoremstyle{plain}
\newtheorem{proposition}[theorem]{Proposition}
\newtheorem{lemma}[theorem]{Lemma}

\newcommand{\FirstItem}{(i)\xspace}
\newcommand{\SecondItem}{(ii)\xspace}
\newcommand{\ThirdItem}{(iii)\xspace}
\newcommand{\FourthItem}{(iv)\xspace}

\newcommand{\FormatFunction}[1]{\ensuremath{\mathsf{#1}}\xspace}
\newcommand{\FormatEntitySet}[1]{\ensuremath{\texttt{#1}}\xspace}
\newcommand{\FormatPredicate}[1]{\ensuremath{\mathtt{#1}}\xspace}
\newcommand{\FP}[1]{\FormatPredicate{#1}}

\newcommand{\Variables}{\FormatEntitySet{Vars}}
\newcommand{\Vars}{\Variables}
\newcommand{\Predicates}{\FormatEntitySet{Preds}}
\newcommand{\Preds}{\Predicates}

\newcommand{\Constants}{\FormatEntitySet{Cons}}
\newcommand{\C}{\Constants}
\newcommand{\Terms}{\FormatEntitySet{Terms}}
\newcommand{\EntitiesIn}[2]{\ensuremath{#1(#2)}\xspace}
\newcommand{\EI}[2]{\EntitiesIn{#1}{#2}}
\newcommand{\arity}[1]{\ensuremath{\FormatFunction{ar}(#1)}\xspace}

\newcommand{\vx}{{\ensuremath{\vec{x}}}\xspace}
\newcommand{\vy}{\ensuremath{\vec{y}}\xspace}
\newcommand{\vz}{\ensuremath{\vec{z}}\xspace}
\newcommand{\vt}{{\ensuremath{\vec{t}}}\xspace}

\newcommand{\Formula}{\ensuremath{U}\xspace}

\newcommand{\FB}{\ensuremath{F}\xspace}
\newcommand{\Query}{\ensuremath{Q}\xspace}
\newcommand{\BCQ}{\ensuremath{Q}\xspace}

\newcommand{\Rule}{\ensuremath{R}\xspace}
\newcommand{\Homomorphism}{\ensuremath{\pi}\xspace}
\newcommand{\Hom}{\Homomorphism}
\newcommand{\Tuple}[1]{\ensuremath{\langle #1 \rangle}\xspace}

\newcommand{\ens}[1]{\ensuremath{\left\{#1\right\}}}
\newcommand{\dotq}{.\ }

\newcommand{\KB}{\ensuremath{\mathcal{K}}\xspace}
\newcommand{\K}{\KB}
\newcommand{\setF}{\ensuremath{\mathcal{F}}\xspace}
\newcommand{\setG}{\ensuremath{\mathcal{G}}\xspace}
\newcommand{\setR}{\ensuremath{\mathcal{R}}\xspace}
\newcommand{\setM}{\ensuremath{\mathcal{M}}\xspace}
\newcommand{\setN}{\ensuremath{\mathcal{N}}\xspace}
\newcommand{\setU}{\ensuremath{\mathcal{U}}\xspace}
\newcommand{\setV}{\ensuremath{\mathcal{V}}\xspace}
\newcommand{\spt}{\ensuremath{\FormatFunction{sp}}}
\newcommand{\oad}{\ensuremath{\textit{1ad}}}
\newcommand{\tad}{\ensuremath{\textit{2ad}}}

\newcommand{\spR}{\ensuremath{\spt(\setR)}\xspace}
\newcommand{\oadR}{\ensuremath{\textit{1ad}(\setR)}\xspace}
\newcommand{\tadR}{\ensuremath{\textit{2ad}(\setR)}\xspace}

\newcommand{\der}{\ensuremath{\mathcal{D}}\xspace}

\newcommand{\FormatChaseVariant}[1]{\ensuremath{\mathbb{#1}}\xspace}
\newcommand{\X}{\FormatChaseVariant{X}}
\newcommand{\Y}{\FormatChaseVariant{Y}}
\newcommand{\Ob}{\FormatChaseVariant{O}}
\newcommand{\SO}{\FormatChaseVariant{SO}}
\newcommand{\R}{\FormatChaseVariant{R}}
\newcommand{\E}{\FormatChaseVariant{E}}
\newcommand{\DatalogFirst}[1]{\ensuremath{\FormatChaseVariant{DF}\text{-}#1\xspace}}
\newcommand{\DF}[1]{\DatalogFirst{#1}}

\newcommand{\funsup}[1]{\ensuremath{\FormatFunction{support}(#1)}}
\newcommand{\funout}[1]{\ensuremath{\FormatFunction{output}(#1)}}
\newcommand{\funres}[1]{\ensuremath{\FormatFunction{res}(#1)}}
\newcommand{\funtrig}[1]{\ensuremath{\FormatFunction{triggers}(#1)}}
\newcommand{\funlen}[1]{\ensuremath{\FormatFunction{length}(#1)}}
\newcommand{\funfr}[1]{\ensuremath{\FormatFunction{fr}(#1)}}
\newcommand{\funCh}[2]{\ensuremath{\FormatFunction{Ch}_{#1}(#2)}}
\newcommand{\funchase}[2]{\ensuremath{\FormatFunction{chase}_{#1}(#2)}}

\newcommand{\chaseterm}[2]{\ensuremath{\mathit{CT}^{#1}_{\forall#2}}}
\newcommand{\chaseterminst}[3]{\ensuremath{\mathit{CT}^{#1}_{#2#3}}}

\newcommand{\cmark}{\ensuremath{\boldsymbol{+}}\xspace}
\newcommand{\xmark}{\ensuremath{\boldsymbol{-}}\xspace}
\newcommand{\emark}{\ensuremath{\boldsymbol{=}}\xspace}
\newcommand{\nemark}{\ensuremath{\boldsymbol{\neq}}\xspace}

\newcommand{\FormatTM}[1]{\ensuremath{\mathit{#1}}\xspace}

\newcommand{\TM}{\FormatTM{M}}
\newcommand{\States}{\ensuremath{Q}\xspace}
\newcommand{\StartingState}{\ensuremath{q_S}\xspace}
\newcommand{\AcceptingState}{\ensuremath{q_A}\xspace}
\newcommand{\RejectingState}{\ensuremath{q_R}\xspace}
\newcommand{\Alphabet}{\ensuremath{\Gamma}\xspace}
\newcommand{\TransitionFunction}{\ensuremath{\delta}\xspace}

\newcommand{\NonFinal}{\FP{NF}}
\newcommand{\Final}{\FP{F}}
\newcommand{\Brake}{\FP{B}}
\newcommand{\Real}{\FP{R}}
\newcommand{\Done}{\FP{D}}
\newcommand{\Last}{\FP{Lst}}
\newcommand{\First}{\FP{Frst}}
\newcommand{\Init}{\FP{Int}}

\newcommand{\HeadState}[1]{\FP{Hd_{#1}}}
\newcommand{\Next}{\FP{Nxt}}
\newcommand{\NextPlus}{\Next^+}
\newcommand{\Step}{\FP{Stp}}
\newcommand{\End}{\FP{End}}
\newcommand{\Content}[1]{\FP{S}_{#1}}
\newcommand{\Blank}{\text{\textvisiblespace}}

\begin{document}

\maketitle

\begin{abstract}
Existential rules are an expressive knowledge representation language mainly developed to query data. In the literature, they are often supposed to be in some normal form that simplifies technical developments. For instance, a common assumption is that rule heads are atomic, i.e., restricted to a single atom. Such assumptions are considered to be made without loss of generality as long as all sets of rules can be normalised
while preserving entailment. However, an important question is whether the properties that ensure the decidability of reasoning are preserved as well. We provide a systematic study of the impact of these procedures on the different chase variants with respect to chase (non-)termination and FO-rewritability. This also leads us to study open problems related to chase termination of independent interest.
\end{abstract}



\section{Introduction}
\label{section:introduction}

\noindent

Existential rules are an expressive knowledge representation language mainly developed to query data \cite{DBLP:conf/ijcai/BagetLMS09,DBLP:conf/pods/CaliGL09}. Such rules are an extension of first-order function-free Horn rules (like those of Datalog) with existentially quantified variables in the rule heads, which allows to infer the existence of unknown individuals. 

Querying a knowledge base (KB) $\KB = \Tuple{\setR, \FB}$, where $\setR$ is a set of existential rules and $\FB$ a set of facts, consists in computing all the answers to queries that are logically entailed from $\KB$. 
Two main techniques have been developed, particularly in the context of the fundamental (Boolean) conjunctive queries. The \emph{chase} is a bottom-up process that expands $\FB$ by rule applications from $\setR$ towards a fixpoint. It produces a \emph{universal} model of $\KB$, i.e., a model of $\KB$ that homomorphically maps to all models of $\KB$, which is therefore sufficient to decide query entailment. \emph{Query rewriting} is a dual technique, which consists in rewriting a query $q$ with the rules in $\setR$ into a query $q'$ such that $q$ is entailed by $\KB$ if and only if $q'$ is entailed by $\FB$ solely. 

Conjunctive query answering being undecidable for existential rules \cite{beeri-vardi-81}, both the chase and query rewriting may not terminate. There is however a wide range of rule subclasses defined by syntactic restrictions that ensure \emph{chase termination} on any set of facts (see, e.g., various acyclicity notions in \cite{DBLP:journals/jair/GrauHKKMMW13}) or the existence, for any conjunctive query, of a (finite) rewriting into a first-order query, a property referred to as \emph{FO-rewritability} \cite{DBLP:journals/jar/CalvaneseGLLR07}. 

%
In the literature, existential rules are often supposed to be in some normal form that simplifies technical developments.
For instance, a common assumption is that rule heads are atomic, i.e., restricted to a single atom.
On the one hand, the use of single-head rules greatly simplifies the presentation of theoretical arguments (e.g., \cite{DBLP:journals/ai/CaliGP12}).
On the other hand, this restriction may also simplify implementations; e.g., the optimisation procedure presented in \cite{MaterKBviaTrigGraph} exploits single-head rules to clearly establish the provenance of each fact computed during the chase.
Moreover, after normalisation, we can apply existing methods to effectively determine if the chase terminates for an input 
single-head 
existential rule set if this set is linear \cite{DBLP:conf/icdt/LeclereMTU19} or guarded \cite{DBLP:conf/pods/GogaczMP20}.
Normal form assumptions are often made without loss of generality as long as all sets of rules can be normalised
while preserving all interesting entailments. However, an important question is whether the properties that ensure the decidability of reasoning are preserved as well. In particular, what is the impact of common normalisation procedures on fundamental properties like chase termination or FO-rewritability? 

%
In fact, the chase is a family of algorithms, which differ from each other in their termination properties. Here, we consider the four main chase variants, namely: 
the \emph{oblivious} chase \cite{cgk08}, 
the \emph{semi-oblivious} (aka skolem) chase \cite{SchemaMapTermToTract},
 the \emph{restricted} (aka standard) chase \cite{DataExSemQAns}
 and the \emph{core} chase \cite{chaserevisited}. 
As the core chase has the inconvenience of being non-monotonic (i.e., the produced set of facts does not grow monotonically), we actually study a monotonic variant that behaves similarly regarding termination, namely the \emph{equivalent} chase \cite{Rocher2016}. 
The ability of a chase variant to halt on a given KB is directly related to its power of 
reducing logical redundancies introduced by rules. 
The oblivious chase blindly performs all possible rule applications, while the equivalent chase terminates exactly when the KB admits a finite universal model. The other variants lie between these two extremes. 
For practical efficiency reasons, the most implemented variant is the restricted chase. However, it is the only variant sensitive to the order of rule applications:
for a given KB, there may be sequences of rule applications that terminate, while others do not. We study a natural strategy, called \emph{Datalog-first restricted chase}, which prioritises Datalog rules (whose head does not include existential quantifiers) thus achieving termination in many real-world cases \cite{DBLP:conf/ijcai/CarralDK17}. Moreover, experiments have shown that it is indeed a very efficient strategy \cite{DBLP:conf/cade/UrbaniKJDC18}. 

On the other hand, we consider two well-known normalisation procedures of a set of rules: \emph{single-piece}-decomposition, which breaks rule heads into subsets called pieces 
and outputs a logically equivalent rule set
\cite{walkDecLine,SCMUCQRewriting}; and \emph{atomic}-decomposition, which requires to introduce fresh predicates and outputs a set of atomic-head rules that form a conservative extension of the original set, hence preserve entailment \cite{tamingInfChase,DBLP:journals/ai/CaliGP12}.

\paragraph{Contributions.} We provide a systematic study of the impact of these procedures on the different chase variants with respect to chase (non-)termination and FO-rewritability. This also leads us to solve some open problems related to chase termination, which are of independent interest.

Although the relationships between most chase variants with respect to chase termination are well understood \cite{anatomychase}, the question remained open regarding the restricted chase and its Datalog-first version. Unexpectedly, we found that Datalog-first strategies are not always optimal:
 we exhibit a rule set $\setR$ such that the restricted chase has a terminating sequence on any KB $\Tuple{\setR, \FB}$ but there is a KB $\Tuple{\setR, \FB}$ on which no Datalog-first strategy terminates (Section 3). 

\begin{table}
\setlength{\tabcolsep}{3.5pt}
\renewcommand{\arraystretch}{1.5}
\begin{tabular}{c|c|c|c|c|c|c|c|c|c|c|c|}
\cline{2-8}
							& \Ob & \SO & $\exists$-\R & $\forall$-\R & $\exists$-\DF{\R} & $\forall$-\DF{\R} & \E \\\hline
\multicolumn{1}{|c|}{Single-Piece}	&\emark&\cmark&\nemark&\nemark&\nemark&\nemark&\emark\\ \hline
\multicolumn{1}{|c|}{One-Way}		&\emark&\emark&\xmark&\xmark&\xmark&\xmark&\xmark\\ \hline
\multicolumn{1}{|c|}{Two-Way}		&\emark&\emark&\cmark&\xmark&\emark&\emark&\emark\\ \hline
\end{tabular}
\caption{Impact on chase termination. Chase variants are denoted as follows: \Ob: Oblivious;  \SO: semi-oblivious; (\DF){\R}: (Datalog-first) restricted;  \E: equivalent.}\label{table-chase} 
\end{table}

While it appears that none of the considered decompositions influences FO-rewritability, the situation is very different concerning chase termination, as summarized in Table \ref{table-chase}. Note that we distinguish between two behaviors for restricted chases: $\exists$ means that at least one chase sequence terminates on any KB (``sometimes-termination'') and $\forall$ that all sequences terminate on any KB (``termination''). 
Single-piece-decomposition (Section 4) has no impact on the oblivious and equivalent chases (noted =), a positive impact on the semi-oblivious chase (noted +), and an erratic impact on (Datalog-first-)restricted chase (noted $\neq$). The standard atomic-decomposition (Section 5), called \emph{one-way} in the table, has a negative impact on all chase variants, except for the (semi-)oblivious ones. Looking for a well-behaved atomic-decomposition procedure, we study a new one, named \emph{two-way} (Section 6). A salient property of this decomposition is that it preserves the existence of a finite universal model. As shown in the table, two-way behaves better than one-way: it preserves (sometimes-)termination of the Datalog-first restricted chase and may even improve the sometimes-termination of the restricted chase. However, the negative impact on the termination of the restricted chase remains. 
 These findings led us to an intriguing question: does a computable normalisation procedure exist that produces atomic-head rules and exactly preserves the termination of the restricted chase? We show that the answer is negative by a complexity argument (Section 7). 
More specifically, we study the decidability status of the following problem: Given a KB $\mathcal K = \Tuple{\setR, \FB}$, does the restricted chase terminate on $\mathcal K$? We show that the associated membership problem is at least at the second level of the arithmetical hierarchy (precisely $\Pi^0_2$-hard) when there is no restriction on $\setR$, while it is recursively enumerable (in $\Sigma^0_1$) when $\setR$ is a set of atomic-head rules. Since $\Sigma^0_1 \subsetneq \Pi^0_2$, we obtain the negative answer to our question. 

\begin{paper}
The complete proofs for all of the results in this paper can be found on an arXiv submission with the same name.
\end{paper}
\begin{tr}
This extended report includes an appendix with additional details on proofs omitted
from the conference version.
\end{tr}

\section{Preliminaries}\label{sec:prelim}

\paragraph{First-Order Logic (FOL)}
We define \Preds, \C, and \Vars to be mutually disjoint, countably infinite sets of \emph{predicates}, \emph{constants}, and \emph{variables}, respectively.
Every $P \in \Preds$ has an \emph{arity} $\arity{P} \geq 0$.
Let $\Terms = \C \cup \Vars$ be the set of \emph{terms}.
We write lists $t_1, \ldots, t_n$ of terms as $\vt$ and often treat them as sets.
For a formula or set thereof \Formula, let \EntitiesIn{\Preds}{\Formula}, \EntitiesIn{\C}{\Formula}, \EntitiesIn{\Vars}{\Formula}, and \EntitiesIn{\Terms}{\Formula} be the sets of all predicates, constants, variables, and terms that occur in \Formula, respectively.
 
An \emph{atom} is a FOL formula $P(\vt)$ with $P$ a $\vert \vt \vert$-ary predicate and $\vt \in \Terms$.
For a formula $\Formula$, we write $\Formula[\vx]$ to indicate that \vx is the set of all free variables that occur in $\Formula$.

\begin{definition}
\label{def:exist-rule}
An \emph{(existential) rule} \Rule is a FOL formula 
\begin{align}
\forall \vx \forall \vy . \big(B[\vx, \vy] \rightarrow \exists \vz . H[\vx, \vz]\big) \label{rule}
\end{align}
where \vx, \vy, and \vz are pairwise disjoint lists of variables; and $B$ and $H$ are (finite) non-empty conjunctions of atoms, called the \emph{body} and the \emph{head} of $R$, respectively.
The set \vx is the \emph{frontier} of \Rule.
If \vz is empty, then \Rule is a \emph{Datalog} rule.
\end{definition}

Next, we often denote a rule such as $R$ above by $B \to H$ or $B \to \exists \vz . H$, omitting all or some quantifiers.

 A \emph{factbase} \FB is an existentially closed (finite) conjunction of atoms. 
 A \emph{Boolean conjunctive query} (BCQ) has the same form as a factbase, and we often identify both notions.
A \emph{knowledge base} (KB) \K is a tuple $\Tuple{\setR, \FB}$ with \setR a rule set and \FB a factbase.
We often identify rule bodies, rule heads, and factbases with (finite) sets of atoms. 
 
 Given atom sets \FB and $F'$, a \emph{homomorphism} \Hom from \FB to $F'$
  is a function with domain $\EI{\Vars}{F}$ such that $\Hom(F) \subseteq F'$; 
\Hom is an \emph{isomorphism} from  \FB to $F'$ if additionally, $\Hom$ is injective and $\Hom^{-1}$ is a homomorphism from \FB to $F'$.
A homomorphism \Hom from $F$ to $F'$ is a \emph{retraction} if \Hom is the identity over $\EI{\Vars}{F}\cap\EI{\Vars}{F'}$ (next, we often use this notion with $F' \subseteq F$).

We identify logical interpretations with atom sets. 
An atom set \FB \emph{satisfies} a rule $\Rule = B \to H$ if, for every homomorphism \Hom from $B$ to \FB, there is an extension $\hat{\Hom}$ of $\Hom$ with $\hat{\Hom}(H)\subseteq F$; equivalently, \FB is a \emph{model} of $\Rule$. 
An atom set \setM is a \emph{model of a factbase} \FB if there is a homomorphism from \FB to \setM, and it is a \emph{model of a KB} $\Tuple{\setR, F}$ if it is a model of \FB and satisfies all rules in \setR.
Given KBs or atom sets $A$ and $B$, $A$ \emph{entails} $B$, written $A \models B$, if every model of $A$ is a model of $B$; $A$ and $B$ are \emph{equivalent}
if $A \models B$ and $B \models A$. 
Given atom sets $F$ and $F'$, it is known that $F \models F'$ iff there is a homomorphism from $F'$ to $F$. 

\begin{definition}
A model \setM of a KB \K is \emph{universal} if there is a homomorphism from \setM to every model of \K.
\end{definition}

Every KB \K admits some (possibly infinite) universal model. Hence, $\K \models \Query$ for any BCQ \Query iff there is a homomorphism from a universal model of \K to \Query. 
The \emph{BCQ entailment} problem takes as input a KB \K and a BCQ $Q$ and asks if $\K \models Q$; it is undecidable \cite{beeri-vardi-81}. 

Next, we will consider transformations of rule sets that introduce fresh predicates. To specify the relationships between a rule set and its decomposition, we will rely on the notion of conservative extension:
\begin{definition}[Conservative extension]\label{def:ce}
    Let \setR and $\setR'$ be two rule sets such that $\EI{\Preds}{\setR} \subseteq \EI{\Preds}{\setR'}$. The set $\setR'$ is a \emph{conservative extension} of the set \setR if (1) the restriction of any model of $\setR'$ to the predicates in $\EI{\Preds}{\setR}$ is a model of $\setR$, and (2) any model \setM of $\setR$ can be extended to a model $\mathcal{M'}$ of $\setR'$ that has the same domain (i.e., \EntitiesIn{\Terms}{\setM} = \EntitiesIn{\Terms}{\setM'}) and agrees with \setM on the interpretation of the predicates in $\EI{\Preds}{\setR}$ (i.e., they have the same atoms with predicates in $\EI{\Preds}{\setR}$). 
\end{definition}
When $\setR'$ is a conservative extension of $\setR$, for any factbase $\FB$ the KBs $\Tuple{\setR, \FB}$ and $\Tuple{\setR', \FB}$
entail the same (closed) formulas on $\EI{\Preds}{\setR}$, in particular BCQs.

\paragraph{The chase}
The chase is a family of procedures that repeatedly apply rules to a factbase until a fixpoint is reached. We formally define such procedures before stating their correctness with respect to factbase entailment in Proposition~\ref{proposition:chase-correctness}.

\begin{definition}[Triggers and derivations]
\label{def:triggers-derivations}
Given a fact set $F$, a \emph{trigger} $t$ on $F$ is a tuple $\Tuple{R, \pi}$ with $R = B \to \exists \vz . H$ a rule and $\pi$ a homomorphism from $B$ to $F$. 
Let $\funsup{t} = \pi(B)$ and $\funout{t} = \pi^R(H)$, where $\pi^R$ is the extension of $\pi$ that maps every variable $z \in \vz$ to the fresh variable $z_t$ that is unique for $z$ and $t$.
A \emph{derivation} from a KB $\K=\Tuple{\setR, \FB}$ is a sequence $\der  = (\emptyset, F_0), (t_1, F_1), \ldots$ such that:
\begin{enumerate}
\item Every $F_i$ in \der is a factbase; moreover, $F_0 = F$. 
\item Every $t_i$ in \der is a trigger $\Tuple{R, \pi}$ on $F_{i-1}$ such that $R \in \setR$, $\funout{t_i} \not \subseteq F_{i-1}$, and $F_i = F_{i-1} \cup \funout{t_{i}}$.
\end{enumerate}
The \emph{result} of $\der$, written \funres{\der}, is 
the union of all the factbases in \der.
Let \funtrig{\der} be the set of all triggers in \der and  $\funlen{\der} = |\funtrig{\der}|$ be the length of \der.
\end{definition}

Different chase variants build specific derivations according to different criteria of trigger applicability. Below, the letters \Ob, \SO, \R, and  \E respectively refer to so-called oblivious, semi-oblivious, restricted, and equivalent\footnote{The equivalent chase behaves as the better-known core chase regarding termination: it halts exactly when the KB has a finite universal model.
The difference lies in the fact that the core chase computes a minimal universal model (i.e., a core).
The equivalent chase has the advantage of being monotonic $(\forall i, F_i \subseteq F_{i+1})$.}
 variants.

\begin{definition}[Applicability]
\label{definition:applicability}
A trigger $t = \Tuple{\Rule, \Hom}$ on a factbase $\FB$
 is \FirstItem \emph{\Ob-applicable} on \FB if  $\funout{t} \not \subseteq F$, 
 \SecondItem \emph{\SO-applicable} on \FB if $\funout{t'} \not \subseteq F$ for every trigger $t' = (R, \pi')$ with $\pi(x) = \pi'(x)$ for all $x \in \funfr{R}$, \ThirdItem \emph{\R-applicable} on \FB if there is no retraction from $F \cup \funout{t}$ to \FB, and \FourthItem \emph{\E-applicable} on \FB if there is no homomorphism from $F \cup \funout{t}$ to \FB.
\end{definition}

\begin{example}
\label{example:preliminaries}
Consider the KB $\K = \Tuple{\setR, F}$ with $\setR = \{R = P(x,y) \rightarrow \exists z . P(y,z) \wedge P(z,y)\}$ and $F = \{ P(a,b) \}$ with $a$ and $b$ some constants.
The trigger $t_1 = (R, \pi_1)$ with $\pi_1 = \{ x \mapsto a, y \mapsto b \}$ is  \X-applicable on $F_0 = F$ (for any  \X), and
$\funout{t_{1}} = \{ P(b, z_{t_1}), P(z_{t_1}, b)\}$. There are two new triggers on $F_1 = F \cup \funout{t_{1}}$, both \Ob- and \SO-applicable, but neither  \R- nor \E-applicable. For instance, consider $t_2 = (R, \pi_2)$ with $\pi_2 = \{ x \mapsto b, y \mapsto z_{t_1} \}$ and $\funout{t_{2}} = \{ P( z_{t_1}, z_{t_2}), P(z_{t_2},  z_{t_1})\}$: there is a retraction from $F_1 \cup \funout{t_{2}}$ to $F_1$, which maps $z_{t_2}$ to $b$. 
\end{example}

\begin{definition}[$(\DF{})\X$-Chase]
For an $\X \in \{\Ob, \SO, \R, \E\}$, an \emph{$\X$-derivation from a KB $\K = \Tuple{\setR, \FB}$} is 
a derivation $\der$ such that every trigger $t_{i} \in \funtrig{\der}$ is  \X-applicable on $F_i$;
$\der$ is a \emph{\DF{\X}-derivation}  if it gives priority to Datalog rules: for any $t_i = \Tuple{R, \pi}  \in \funtrig{\der}$, if $R$ is a non-Datalog rule, then $F_{i-1}$ satisfies every Datalog rule in \setR. 
A (\DF){\X}-derivation $\der$ is \emph{fair} if for every $F_i$ occurring in $\der$ and trigger $t$ $\X$-applicable on $F_i$, 
there is some $j > i$ such that $t$ is not \X-applicable on $F_{j}$. 
A (\DF){\X}-derivation is \emph{terminating} if it is fair and finite.
\end{definition}

The result of any fair \X-derivation is a universal model of the KB, for $\X \in \{\Ob, \SO, \R\}$, and has a retraction to a universal model for \X = \E. Therefore, we obtain:

\begin{proposition}
\label{proposition:chase-correctness}
Consider a BCQ $Q$, a KB \K, and some fair \X-derivation \der from \K where $\X \in \{\Y, \DF{\Y}\}$ and $\Y \in \{\Ob, \SO, \R, \E\}$.
Then, $\K \models Q$ iff $\funres{\der} \models Q$.
\end{proposition}

\paragraph{Decidable Classes of Rule Sets}

We now define classes of rule sets that ensure the decidability of BCQ entailment, based either on chase termination or on query rewritability. 

\begin{definition}[Chase-Terminating Sets]
\label{definition:chase-terminating-sets}
For a $\Y \in \{\Ob, \SO, \R, \E\}$ and an $\X \in \{\Y, \DF{\Y}\}$, let \chaseterm{\X}{\forall} (resp. \chaseterm{\X}{\exists}) be the set of all rule sets \setR such that every (resp. some) fair \X-derivation 
 from every KB \Tuple{\setR, \FB} is finite.
\end{definition}

When $\setR \in \chaseterm{\X}{\forall}$ (resp. $\chaseterm{\X}{\exists}$), \setR ensures the \emph{termination} (resp. \emph{sometimes termination}) of the ${\X}$-chase.

\begin{example}
Consider  the KB $\K = \Tuple{\setR, F}$ from Example~\ref{example:preliminaries}.
All fair \Ob- or \SO-derivations from \K are infinite.  
The only one fair \R-derivation (resp. \E-derivation) from $\K$ is $\der  = (\emptyset, F_0), (t_1, F_1)$.
Any fair \R-derivation from a KB with {\setR} is finite and hence, $\setR \in \chaseterm{\R}{\forall}$ (and $\setR \in \chaseterm{\E}{\forall}$). 
\end{example}

\begin{definition}[FO-rewritability]
A rule set $\setR$ is \emph{FO-rewritable}  if for any BCQ $Q$, there is a (finite) BCQ set $\{Q_1, \ldots, Q_n\}$ such that, for every factbase $F$, $\Tuple{\setR, \FB} \models Q$ iff $\FB \models Q_i$ for some $1 \leq i \leq n$. 
\end{definition}

In our proofs, we rely on a property equivalent to FO-rewritability: the bounded derivation depth property, which has the advantage of being based on (a breadth-first version of) the chase \cite{DBLP:conf/pods/CaliGL09}.
See \cite{DBLP:journals/ai/GottlobKKPSZ14} about the equivalence between both properties.

\begin{definition}[BDDP]
For a rule \Rule and a factbase \FB, let $\Rule(\FB) \supseteq \FB$ be the minimal factbase that includes \funout{t} for every trigger $t$ with $R$.
For a rule set $\setR$, let $\setR(F) = \bigcup_{R \in \setR} R(F)$.
For a KB $\KB = \Tuple{\setR, \FB}$, let $\funCh{0}{\KB} = F$ and $\funCh{i}{\KB} = \setR(\funCh{i-1}{\KB})$ for every $i \geq 1$.

A rule set \setR has the \emph{bounded derivation depth property (BDDP)} if, for any BCQ \BCQ, there is some $k \geq 0$ such that, for every factbase $\FB$, $\Tuple{\setR, \FB} \models \BCQ$ iff $\funCh{k}{\Tuple{\setR, \FB}} \models \BCQ$.
\end{definition}

\paragraph{Normalisation Procedures}\label{sub:normProc}

Finally, we formally define normalisation procedures and their impact on the above properties. 
A \emph{normalisation procedure} is a function $f$ that maps rule sets to rule sets (complying with a certain shape) such that 
for any rule set $\setR$, $\Tuple{\setR, \FB}\models Q$  iff $\Tuple{f(\setR), \FB}\models Q$ for any factbase $\FB$ and BCQ $Q$ on $\EI{\Preds}{\setR}$. 

\begin{definition}\label{def:chaseTermination}
Consider some $\mathbb{X}\in\ens{\Ob, \SO, \R, \DF{\R}, \E}$.
Then, a normalisation procedure $f$:
\begin{itemize}
\item \emph{Preserves termination} of the \X-chase if $f(\chaseterm{\X}{\forall}) \subseteq \chaseterm{\X}{\forall}$; it preserves \emph{sometimes-termination} of the \X-chase if $f(\chaseterm{\X}{\exists}) \subseteq \chaseterm{\X}{\exists}$.
\item \emph{Preserves non-termination} 
of the \X-chase if $f(\overline{\chaseterm{\X}{\forall}}) \subseteq \overline{\chaseterm{\X}{\forall}}$.  
Otherwise, $f$ \emph{may gain termination.}
\item \emph{Preserves rewritability} if it maps FO-rewritable rule sets to FO-rewritable rule sets.
\end{itemize}

\end{definition} 

\section{Generality of Chase-Terminating Rule Sets}
\label{section:chase-terminating-sets}

One of our goals is to study normalisation procedures that preserve membership over the sets of chase-terminating rule sets from Definition~\ref{definition:chase-terminating-sets}.
To be systematic, we clarify the equality and strict-subset relations between these sets in Theorems~\ref{theorem:chase-terminating-sets-equality} and \ref{theorem:chase-terminating-sets-strict-subsets}, respectively.
\citeauthor{anatomychase} already proved most of the claims in these theorems (see Theorem~4.5, Propositions~4.6 and 4.7, and Corollary~4.8 in \cite{anatomychase}); we reprove some of them again to be self-contained.
However, note that all results regarding Datalog-first chase variants are our own contribution.

\begin{theorem}
\label{theorem:chase-terminating-sets-equality}
For every $\X \in \{\Ob, \SO, \E\}$, we have that
$$\chaseterm{\X}{\forall} = \chaseterm{\X}{\exists} = \chaseterm{\DF{\X}}{\forall} = \chaseterm{\DF{\X}}{\exists}$$
\end{theorem}
\begin{proof}[Sketch]

To show that the theorem holds if $\X  = \Ob$ (resp. $\X  = \SO$), it suffices to prove that all fair \X-derivations from an input KB \KB produce the same result (resp. same result up to isomorphism); see forthcoming Lemma~\ref{proposition:oblivious-same-result}.
 
 All fair \E-derivations from an input KB \KB are finite iff \KB admits a finite universal model \cite{Rocher2016}. 
Hence, the theorem holds if $\X = \E$.
\end{proof}

The equalities in Theorem~\ref{theorem:chase-terminating-sets-equality} simplify our work: for instance, if a function preserves termination of the oblivious chase, then we know that it also preserves sometimes-termination of this variant.
Alas, the remaining sets of chase-terminating rule sets are not equal:

\newcommand{\NSubset}{\ensuremath{\hspace{-0.18em}\subset\hspace{-0.18em}}\xspace}

\begin{theorem}
\label{theorem:chase-terminating-sets-strict-subsets}
The following hold:
\begin{align*}
\chaseterm{\Ob}{\forall} \subset~& \chaseterm{\SO}{\forall} \subset \chaseterm{\R}{\forall} \subset \chaseterm{\DF{\R}}{\forall} \\
\subset~& \chaseterm{\DF{\R}}{\exists} \subset \chaseterm{\R}{\exists} \subset \chaseterm{\E}{\forall}
\end{align*}
\end{theorem}
\begin{proof}[Sketch]
The subset inclusions follow by definition; we present some rule sets to show that these are strict:
\begin{align*}
\{&P(x, y) \to \exists z . P(x, z)\} \in \chaseterm{\SO}{\forall} \setminus \chaseterm{\Ob}{\forall} \\[0.5ex]
\{&P(x, y) \to \exists z .  P(y, z) \wedge P(z, y)\} \in \chaseterm{\R}{\forall} \setminus \chaseterm{\SO}{\forall}\\[0.5ex]
\{&P(x, y) \to \exists z . P(y, z), \\[-0.5ex]
&P(x, y) \to P(y, x)\} \in \chaseterm{\DF{\R}}{\forall} \setminus \chaseterm{\R}{\forall} \\[0.5ex]
\{&P(x, y) \to \exists z . P(y, z) \wedge P(z, y), \\[-0.5ex]
&P(x, y) \to \exists z . P(y, z)\} \in \chaseterm{\DF{\R}}{\exists} \setminus \chaseterm{\DF{\R}}{\forall} \\[0.5ex]
\{&P(x, y) \to \exists z . P(y, z), \\[-0.5ex]
&P(x, y) \wedge P(y, z) \to P(y, x)\} \in \chaseterm{\E}{\forall} \setminus \chaseterm{\R}{\exists}
\end{align*}
Moreover, the rule set $\setR = \{(\ref{rule:r-loop}\text{--}\ref{rule:s-succ})\}$ is in $\chaseterm{\R}{\exists} \setminus \chaseterm{\DF{\R}}{\exists}$:
\begin{align}
A(x) &\to R(x, x) \label{rule:r-loop} \\
R(x, y) \wedge S(y, z) &\to S(x, x) \label{rule:s-loop} \\
A(x) \wedge S(x, y) &\to A(y) \label{rule:a-prop} \\
A(x) &\to \exists z . R(x, z) \label{rule:r-succ} \\
R(x, y) &\to \exists z . S(y, z) \label{rule:s-succ}
\end{align}


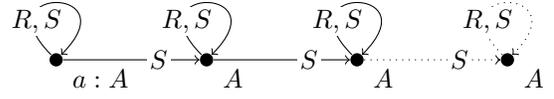
\begin{figure}
\tikzset{every loop/.style={min distance=60,in=130,out=50,looseness=20}}

\begin{center}
\begin{tikzpicture}[roundnode/.style={circle, fill=black, inner sep=0pt, minimum size=1mm},
labelnode/.style={fill=white,inner sep=1pt}]

\newcommand{\xsep}{2}
\newcommand{\xa}{0*\xsep}
\newcommand{\xb}{1*\xsep}
\newcommand{\xc}{2*\xsep}
\newcommand{\xd}{3*\xsep}
\newcommand{\xe}{4*\xsep}

\newcommand{\ya}{0}

\node[fill,circle,inner sep=0pt,minimum size=5pt, label=-3:{$a : A$}] at (\xa, \ya) (1) {};
\node[fill,circle,inner sep=0pt,minimum size=5pt, label=-3:{$A$}] at (\xb, \ya) (2) {};
\node[fill,circle,inner sep=0pt,minimum size=5pt, label=-3:{$A$}] at (\xc, \ya) (3) {};
\node[fill,circle,inner sep=0pt,minimum size=5pt, label=-3:{$A$}] at (\xd, \ya) (4) {};
\node[circle,inner sep=0pt,minimum size=5pt, label=-3:{}] at (\xe, \ya) (5) {};

\path[->] (1) edge [loop above, out=135, in=45, looseness=25] node[labelnode,pos=0.1] {$R, S$} (1);
\path[->] (2) edge [loop above, out=135, in=45, looseness=25] node[labelnode,pos=0.1] {$R, S$} (2);
\path[->] (3) edge [loop above, out=135, in=45, looseness=25] node[labelnode,pos=0.1] {$R, S$} (3);
\path[->] (4) edge [dotted, loop above, out=135, in=45, looseness=25] node[labelnode,pos=0.1] {$R, S$} (4);

\path[->] (1) edge node[labelnode,pos=0.7] {$S$} (2);
\path[->] (2) edge node[labelnode,pos=0.7] {$S$} (3);
\path[->] (3) edge[dotted] node[labelnode,pos=0.7] {$S$} (4);

\end{tikzpicture}
\end{center}
\caption{The only result of the \DF{\R}-chase from the KB $\K = \langle \setR, \{A(a)\} \rangle$ introduced in the proof of Theorem~\ref{theorem:chase-terminating-sets-strict-subsets}}
\label{figure:infinite-dfr-chase}
\end{figure}

To show that $\setR \notin \chaseterm{\DF{\R}}{\exists}$ we prove that the KB $\K = \langle \setR, \{A(a)\} \rangle$ does not admit terminating \DF{\R}-derivations. Specifically, 
all fair \DF{\R}-derivations from \K yield the same result, which is depicted in Figure~\ref{figure:infinite-dfr-chase}. 
Rule \eqref{rule:r-loop} is applied first, then the following pattern is repeated: apply rule \eqref{rule:s-succ} followed by Datalog rules \eqref{rule:s-loop}, \eqref{rule:a-prop} and \eqref{rule:r-loop}. Rule \eqref{rule:r-succ}  is never applicable since priority is given to rule \eqref{rule:r-loop}. 
To show that $\setR \in \chaseterm{\R}{\exists}$ we verify that every KB of the form $\langle \setR, F \rangle$ admits a terminating $\R$-derivation.
We can produce such a derivation by exhaustively applying the rules in \setR in the following order:\footnote{To understand why this strategy results in a terminating \R-derivation, we suggest 
to first try it on $F = \{A(a)\}$,
which yields $\{A(a), R(a,z_1), S(z_1,z_2), S(a,a), R(a,a)\}$. }
first, apply rules \eqref{rule:a-prop}, \eqref{rule:r-succ} and \eqref{rule:s-succ}; then, apply \eqref{rule:s-loop}; finally, apply \eqref{rule:r-loop}.
\end{proof}

Our main achievement is showing that $\chaseterm{\R}{\exists} \setminus \chaseterm{\DF{\R}}{\exists}$ is non-empty; thus proving that Datalog-first strategies are not necessarily the most terminating for the restricted chase.


\section{Single-Piece Decomposition}\label{sec:spt}

The single-piece decomposition (piece-decomposition in short) is a 
procedure that splits a rule $\Rule = B \to \exists z . H$ into several rules $\Rule_1, \ldots, \Rule_n$ that have the same body as \Rule, and whose head is a subset of $H$ that (directly or indirectly) shares some existential variable in $H$.

\begin{definition}\label{def:sptrans}
The \emph{piece graph} of a rule $R = B \to \exists \vz .H$ is the graph whose vertices are the atoms in $H$, and with an edge between $a$ and $a'$ if $\vz \cap \EI{\Vars}{a} \cap \EI{\Vars}{a'}$ is non-empty.
A \emph{(rule) piece} of $R$ is the conjunction of atoms corresponding to a (connected) component of its piece graph.

The \emph{piece-decomposition} of a rule $R = B \to \exists \vec{z} . H$ is the rule set $\spt(R) = \ens{B \to \exists \vec{v} . H'\ |\ H' \text{ is a piece of } R}$.
For a rule set \setR, let $\spt(\setR) = \bigcup_{R \in \setR} \spt(R)$.
\end{definition}

\begin{example}
Consider the rule $\eqref{rule:single-piece-input}$ and its single-piece decomposition $\spt(\eqref{rule:single-piece-input}) = \{(\ref{rule:single-piece-output-1}\text{--}\ref{rule:single-piece-output-3})\}$:
\begin{align}
R(x, y) &\to \exists z, u . P(x, z) \wedge A(z) \wedge A(u) \wedge P(x, y) \label{rule:single-piece-input} \\
R(x, y) &\to \exists z. P(x, z) \wedge A(z) \label{rule:single-piece-output-1} \\
R(x, y) &\to \exists u.  A(u) \label{rule:single-piece-output-2} \\
R(x, y) &\to P(x, y) \label{rule:single-piece-output-3}
\end{align}
\end{example}

Piece-decomposition is indeed a normalisation procedure, since it preserves logical equivalence: 

\begin{proposition}\label{prop:spteq}
A rule set \setR is equivalent to the set \spR.
\end{proposition}

The following lemma is later applied to show that the piece-decomposition preserves termination of the oblivious and semi-oblivious chase in Theorem~\ref{thm:sptOSO}:
\begin{lemma}
\label{proposition:oblivious-same-result}
Consider some fair \X-derivations \der and $\der'$ from a KB \KB.
If $\X = \Ob$, then $\funres{\der} = \funres{\der'}$.
If $\X = \SO$, then \funres{\der} is isomorphic to \funres{\der'}.
\end{lemma}

\begin{definition}\label{def:ch}
Given some $\X \in \ens{\Ob, \SO}$ and a KB \K, let \funCh{\X}{\K} be some (arbitrarily chosen) atom set that is isomorphic to the result of all fair \X-derivations from \K. 
\end{definition}


\begin{theorem}\label{thm:sptOSO}
The piece-decomposition preserves the termination of the \Ob-chase and \SO-chase.
\end{theorem}

\begin{proof}[Sketch]
Consider some $\X\in\ens{\Ob, \SO}$ and some \X-derivation $\der$ from a KB $\KB = \Tuple{\setR, F}$.
We can show via induction on \der that there is an injective homomorphism from $\funCh{\X}{\langle \spR, F \rangle}$ to $\funCh{\X}{\KB}$.
Therefore, finiteness of \funCh{\X}{\KB} implies finiteness of \funCh{\X}{\langle \spR, F \rangle}.
\end{proof}

The piece-decomposition does not preserve the termination of any restricted chase variant.
The reason is that it allows for intertwining the application of split rules that come from different original rules, resulting in new application strategies that may lead to non-termination. 

\newcommand{\NTo}{\ensuremath{\hspace{-0.15em}\to\hspace{-0.15em}}\xspace}
\newcommand{\NWedge}{\ensuremath{\hspace{-0.15em}\wedge\hspace{-0.15em}}\xspace}

\begin{theorem}\label{thm:sptR}
The piece-decomposition does not preserve termination of the \R- or the \DF{\R}-chase.
\end{theorem}

\begin{proof}[Sketch]
Consider the rule set $\setR = \{\eqref{rule:p-loop-a}, \eqref{rule:p-successor}\}$ and its piece-decomposition $\spR = \{(\ref{rule:p-successor}\text{--}\ref{rule:p-range-a})\}$:
\vspace*{-\baselineskip}
\begin{center}
\begin{minipage}{0.55\linewidth}
\begin{align}
P(x, y) &\NTo P(y, y) \NWedge A(y) \label{rule:p-loop-a} \\
A(x) &\NTo \exists z . P(x, z) \label{rule:p-successor}
\end{align}
\end{minipage}~
\begin{minipage}{0.43\linewidth}
\begin{align}
P(x, y) &\NTo P(y, y) \label{rule:p-loop} \\
P(x, y) &\NTo A(y) \label{rule:p-range-a}
\end{align}
\end{minipage}
\end{center}

The set $\setR$ is in \chaseterm{\R}{\forall} because triggers with \eqref{rule:p-successor} are not \R-applicable to the output of triggers with \eqref{rule:p-loop-a}.
The set \spR is not in \chaseterm{\R}{\forall} because the KB $\Tuple{\spR, \{A(a)\}}$ admits the following non-terminating \R-derivation $(\emptyset, F_0), (t_1, F_1), \ldots$:\vspace*{-\baselineskip}
\begin{center}
\begin{minipage}{0.48\linewidth}
\begin{align*}
\FB_0 = \{&A(a)\}, \\
\FB_1 = \{&P(a, z_1)\} \cup F_0, \\
\FB_2 = \{&A(z_1)\} \cup F_1, \\
\FB_3 = \{&P(z_1, z_2)\} \cup F_2,
\end{align*}
\end{minipage}
\begin{minipage}{0.48\linewidth}
\begin{align*}
\FB_4 = \{&P(z_1, z_1)\} \cup F_3, \\
\FB_5 = \{&A(z_2)\} \cup F_4, \\
\FB_6 = \{&P(z_2, z_3)\} \cup F_5, \\
\ldots
\end{align*}
\end{minipage}
\end{center}
This derivation is built by first applying rule \eqref{rule:p-successor} (leading to $\FB_1$), then indefinitely repeating the sequence of rule applications  \eqref{rule:p-range-a}, \eqref{rule:p-successor}, and \eqref{rule:p-loop}. 
 In contrast, the only fair \R-derivation with $\setR$ would apply \eqref{rule:p-successor} then  \eqref{rule:p-loop-a}, leading to 
$\{A(a), P(a,z_1), P(z_1,z_1), A(z_1)\}$.

To get a similar behavior with the \DF{\R} chase, we introduce ``dummy'' existential variables in rules \eqref{rule:p-loop-a} and \eqref{rule:p-successor}, so that their piece-decomposition has no Datalog rules:
\begin{align}
P(x, y, v) &\to \exists u, w . P(y, y, u) \wedge A(y, w) \label{rule:dummy-ext-1}\\
A(x, v) &\to \exists z, u . P(x, z, u) \label{rule:dummy-ext-2} 
\end{align}
Applying analogous arguments we can show that $\setR' = \{\eqref{rule:dummy-ext-1}, \eqref{rule:dummy-ext-2}\}$ is in \chaseterm{\DF{\R}}{\forall} and that $\spt(\setR')$ is not.
\end{proof}

Initially, we believed that the piece-decomposition would preserve sometimes-termination of the \R-chase.
Our intuition was that, given a terminating \R-derivation from a KB $\KB = \Tuple{\setR, \FB}$, we could replicate this derivation from \Tuple{\spt(\setR), \FB} by applying the split rules in $\spt(\setR)$ piece by piece.
Surprisingly, this is not always possible:

\begin{theorem}\label{thm:sptExR}
The piece-decomposition does not preserve the sometimes-termination of the \R-chase.
\end{theorem}
\begin{proof}[Sketch]
The following set $\setR = \{(\ref{rule:u-init-sp}\text{--}\ref{rule:generate-new-a})\}$ is in \chaseterm{\R}{\exists} and its piece-decomposition $\spR = \{(\ref{rule:u-to-r}\text{--}\ref{rule:u-init-sp-2})\}$ is not. The set $\setR$ is adapted from $\{(\ref{rule:r-loop}\text{--}\ref{rule:s-succ})\}$ (proof of Th. \ref{theorem:chase-terminating-sets-strict-subsets}).
Note that \eqref{rule:u-init-sp} is split into two equivalent rules \eqref{rule:u-init-sp-1} and \eqref{rule:u-init-sp-2}. To show that $\spR \not \in \chaseterm{\R}{\exists}$, we start again from $\{A(a)\}$. Again, some $R$-atom is created and leads to apply other rules. With $\setR$, applying \eqref{rule:u-init-sp} then \eqref{rule:u-to-r} creates an atom of form $R(y_1,z_1)$, while with $\spR$, applying \eqref{rule:u-init-sp-1} then \eqref{rule:u-to-r} creates an atom $R(y_1,y_1)$. This loop leads to non-termination.  
\begin{align}
A(x) \to \exists y, z . U(x, y) \NWedge H(&y, x) \NWedge U(x, z) \NWedge H(z, x) \label{rule:u-init-sp} \\
U(x, y) \wedge U(x, z) &\to R(y, z) \label{rule:u-to-r} \\
U(x, z) \wedge R(y, z) &\to \exists v . R(z, v) \label{rule:u-extends-r} \\
R(x, y) \wedge R(y, z) &\to \exists v . S(z, v) \label{rule:r-to-s-sp} \\
R(x, y) \wedge S(y, z) &\to S(x, x) \label{rule:s-loop-sp-1} \\
 A(x) \wedge U(x, y) \wedge S(y, z) &\to \exists v . H(z, v) \wedge A(v) \label{rule:generate-new-a} \\
A(x) \to \exists y . U(x, y) &\NWedge H(y, x) \label{rule:u-init-sp-1} \\
A(x) \to \exists z . U(x, z) &\NWedge H(z, x) \label{rule:u-init-sp-2}
\end{align}

\end{proof}

\begin{theorem}\label{thm:sptExDfR}
The piece-decomposition does not preserve the sometimes-termination of the \DF{\R}-chase.
\end{theorem}

\begin{proof}[Sketch]
The set $\setR = \{(\ref{rule:r-loop-sp}\text{--}\ref{rule:s-succ-sp})\}$ is in \chaseterm{\DF{\R}}{\exists} while $\spR = \{(\ref{rule:s-loop-sp}\text{--}\ref{rule:r-loop-sp-2})\}$ is not. Note that the only difference with $\{(\ref{rule:r-loop}\text{--}\ref{rule:s-succ})\}$ is the atom $H(x,y)$ in the first rule, making it non-Datalog, which prevents its early application.
\vspace*{-\baselineskip}
\begin{center}
\begingroup
\setlength{\abovedisplayskip}{0pt}
\setlength{\belowdisplayskip}{0.2pt}
\begin{align}
A(x) &\to \exists y . R(x, x) \wedge H(x, y) \label{rule:r-loop-sp} \\
R(x, y) \wedge S(y, z) &\to S(x, x) \label{rule:s-loop-sp} \\
A(x) \wedge S(x, y) &\to A(y) \label{rule:a-prop-sp}
\end{align}
\begin{minipage}{0.49\linewidth}
\begin{align}
A(x) &\NTo \exists y . R(x, y) \label{rule:r-succ-sp} \\
R(x, y) &\NTo \exists z . S(y, z) \label{rule:s-succ-sp}
\end{align}
\end{minipage}\quad
\begin{minipage}{0.46\linewidth}
\begin{align}
A(x) &\NTo R(x, x) \label{rule:r-loop-sp-1} \\
A(x) &\NTo \exists y . H(x, y) \label{rule:r-loop-sp-2}
\end{align}
\end{minipage}
\endgroup
\end{center}

\end{proof}


Since piece-decomposition preserves logical equivalence (Proposition~\ref{prop:spteq}), one directly obtains that it preserves termination of the equivalent chase. Indeed, $\Tuple{\setR, \FB}$ admits a finite universal model iff \Tuple{\spR, \FB} admits one:

\begin{theorem}\label{thm:sptE}
The piece-decomposition preserves the termination of the $\E$-chase.
\end{theorem}
%

The piece-decomposition may gain termination:

\begin{theorem}\label{thm:sp-improves}
The piece-decomposition may gain termination (and sometimes-termination) 
of the \SO-, the \R-, and the \DF{\R}-chase but not of the \Ob- and \E-chase.
\end{theorem}
\begin{proof}[Sketch]
Consider the set $\setR = \{P(x, y) \to \exists z . P(x, z) \wedge R(x, y)\}$, which is not in \chaseterm{\SO}{\forall}, \chaseterm{\R}{\exists}, \chaseterm{\R}{\forall}, \chaseterm{\DF{\R}}{\exists}, or \chaseterm{\DF{\R}}{\forall}.
However, $\spR$ is in all of these sets.

Concerning the \Ob-chase, we show via induction that $\funCh{\Ob}{\Tuple{\setR, \FB}}$ and $\funCh{\Ob}{\Tuple{\spR, \FB}}$ are isomorphic for any rule set \setR and factbase \FB (where \funCh{\Ob}{\cdot} is the function from Definition~\ref{def:ch}, which maps a KB to its only \Ob-chase result). Hence, all \Ob-derivations from \Tuple{\setR, \FB} are terminating iff all \Ob-derivations from \Tuple{\spR, \FB} are terminating.
Concerning the \E-chase, we rely again on Proposition~\ref{prop:spteq}. 
\end{proof}

To show that the piece-decomposition preserves FO-rewritability we show that it preserves the BDDP property.

\begin{theorem}\label{thm:sptBDDP}
A rule set $\setR$ is BDDP iff $\spR$ is BDDP.
\end{theorem}
\begin{proof}[Sketch]
We can prove by induction that, for any KBs $\Tuple{\setR, \FB}$ and any $i \geq 1$, the factbases $\funCh{i}{\Tuple{\setR, \FB}}$ and 
$\funCh{i}{\Tuple{\spR, F}}$ are isomorphic. 
\end{proof}

\section{One-Way Atomic Decomposition}\label{sec:1ad}

Piece-decomposition may not produce atomic-head rules; a useful restriction considered in many contexts. 
The following procedure is classically  used to produce such rules:

\begin{definition}\label{def:1ad}
The \emph{one-way atomic decomposition} of a rule $R = B[\vx, \vy] \to \exists \vz . H[\vx, \vz]$ is the rule set $\oad(R)$ that contains the rule $B \to \exists \vz . X_R(\vx, \vz)$ and,  for each atom $P(\vt) \in H$, the rule $X_R(\vx, \vz) \to P(\vt)$, where $X_R$ is a fresh predicate unique for $R$, of arity $\vert \vx \vert + \vert \vz \vert$.
Given a rule set \setR, let $\oad(\setR) = \bigcup_{R \in \setR} \oad(R)$.
\end{definition}

\begin{example}\label{example-1ad}
Consider the rule $\eqref{rule:one-way-input}$ and its one-way atomic decomposition $\oad(\eqref{rule:one-way-input}) = \{(\ref{rule:one-way-output-1}\text{--}\ref{rule:one-way-output-3})\}$:
\begin{align}
R(x, y) &\to \exists z . P(x, z) \wedge S(x, y, z) \label{rule:one-way-input} \\
R(x, y) &\to \exists z. X_\eqref{rule:one-way-input}(x, y, z) \label{rule:one-way-output-1} \\
X_\eqref{rule:one-way-input}(x, y, z) &\to P(x, z) \label{rule:one-way-output-2} \\
X_\eqref{rule:one-way-input}(x, y, z) &\to S(x, y, z) \label{rule:one-way-output-3}
\end{align}
Note that piece-decomposition would not decompose rule \eqref{rule:one-way-input}, i.e., $\spt(\eqref{rule:one-way-input}) = \{\eqref{rule:one-way-input}\}$
\end{example}

Strictly speaking, $\setR$ and $\oadR$ cannot be logically equivalent because they are built on different sets of predicates; however, it is straighforward to check that $\oadR$ is a conservative extension of $\setR$.
Therefore, one-way atomic decomposition is indeed a normalisation procedure.

The following is a corollary of forthcoming Theorem~\ref{thm:2adOSO}:
\begin{theorem}\label{thm:1adOSO}
The one-way atomic decomposition preserves termination of the \Ob- and \SO-chase.
\end{theorem}

An interesting phenomenon occurs with the one-way atomic decomposition: the notions of \SO-applicability on  \setR and \R-applicability 
on $\oadR$ coincide: 

\begin{lemma}\label{thm:1adReq1adSO}
Consider a KB $\langle \setR, \FB \rangle$ and a finite derivation $\der = (\emptyset, F_0), (t_1, F_1), \ldots, (t_n, F_n)$ from $\Tuple{\oadR, \FB}$.
Then, a trigger $t$ with a rule in $\oadR$ is \R-applicable on $F_n$ iff it is \SO-applicable on $F_n$.
\end{lemma}
\begin{proof}
$(\Rightarrow)$: from Definition~\ref{definition:applicability}. $(\Leftarrow)$: Let $t = (R, \pi)$ with $R \in \oadR$. 
If $R$ is Datalog, the notions of \SO- and \R-applicability coincide for every factbase. Otherwise, $R$ is of the form $B[\vx, \vy] \to \exists \vz . X_{R'}(\vx, \vz)$ where $X_{R'} \not \in \EI{\Preds}{\setR}$ and $R' \in \setR$ is of the form $B[\vx, \vy] \to \exists \vz . H[\vx, \vz]$.
If $t$ is \SO-applicable on $F_n$, then, for every trigger $t' = (R, \pi')$ with $\pi(\vx) = \pi'(\vx)$, it holds that $\funout{t'} \not\subseteq F_n$.
By Definition~\ref{def:1ad}, $R$ is the only rule in \oadR with $X_{R'}$ in its head, hence $\pi''(X_{R'}(\vx, \vz)) \notin F_n$ for every extension $\pi''$ of $\pi$. Hence, $t$ is \R-applicable on $F_n$.
\end{proof}

Intuitively, Lemma~\ref{thm:1adReq1adSO} implies that, after applying the one-way atomic decomposition, \R-applicability becomes as loose and unrestrictive as \SO-applicability. 
Therefore:

\begin{theorem}
The one-way atomic decomposition does not preserve termination nor sometimes-termination of the \R- or the \DF{\R}-chase.
\end{theorem}
\begin{proof}
By Theorem~\ref{theorem:chase-terminating-sets-strict-subsets}, there is a rule set $\setR \notin \chaseterm{\SO}{\forall}$ that is in \chaseterm{\R}{\forall}, \chaseterm{\DF{\R}}{\forall}, \chaseterm{\DF{\R}}{\exists}, and \chaseterm{\R}{\exists}.\footnote{Such rule set in given in the proof of Theorem~\ref{theorem:chase-terminating-sets-strict-subsets}; see also $\{\eqref{rule:p-to-p-chain}\}$ in the proof of Theorem~\ref{thm:1adE}.  }
By Theorem~\ref{thm:1adOSO}, $\oadR$ is not in $\chaseterm{\SO}{\forall}$; that is, there is some KB $\KB$ of the form $\Tuple{\oadR, \FB}$ that does not admit any terminating \SO-derivation.
By Lemma~\ref{thm:1adReq1adSO}, every terminating \R-derivation from \KB is also a terminating \SO-derivation from $\KB$.
Therefore, \KB does not admit any terminating \R-derivation, hence \oadR is not in \chaseterm{\R}{\forall}, \chaseterm{\DF{\R}}{\forall}, \chaseterm{\DF{\R}}{\exists}, or \chaseterm{\R}{\exists}.
\end{proof}

\begin{theorem}\label{thm:1adE}
The one-way atomic decomposition does not preserve the termination of the \E-chase.
\end{theorem}

\begin{proof}
Consider the rule set $\setR = \{\eqref{rule:p-to-p-chain}\}$ 
(see also Example \ref{example:preliminaries})
and its decomposition $\oad(\setR) = \{(\ref{rule:p-to-p-chain-1}\text{--}\ref{rule:p-to-p-chain-3})\}$:
\begin{align}
P(x, y) &\to \exists z\dotq P(y,z)\wedge P(z,y) \label{rule:p-to-p-chain} \\
P(x, y) &\to \exists z .  X_\eqref{rule:p-to-p-chain}(y,z) \label{rule:p-to-p-chain-1} \\
X_\eqref{rule:p-to-p-chain}(y, z) &\to P(y, z) \label{rule:p-to-p-chain-2} \\
X_\eqref{rule:p-to-p-chain}(y, z) &\to P(z, y) \label{rule:p-to-p-chain-3}
\end{align}

The rule set \setR is in \chaseterm{\E}{\forall} since every \E-derivation from a KB $\Tuple{\setR, \FB}$ yields a finite result, which is a subset of 
$\FB \cup \{P(a, z_a), P(z_a, a) \mid a \in \EI{\Terms}{\FB}, z_a \not \in \EI{\Terms}{\FB} \}$. 

The rule set $\oad(\setR)$ is not in $\chaseterm{\E}{\forall}$ since the KB $\KB = \Tuple{\oad(\setR), \{ P(a, b)\} }$ has no terminating \E-derivation.
In fact, all 
 fair \E-derivations from \KB yield the same result:
\begin{align*}
\{&P(a, b), X_\eqref{rule:p-to-p-chain}(b, z_1), P(b, z_1), P(z_1, b)\} \cup~ \\
\{&X_\eqref{rule:p-to-p-chain}(z_i, z_{i+1}), P(z_i, z_{i+1}), P(z_{i+1}, z_i) \mid i \geq 1\}
\end{align*}
\end{proof}

Again, to show that the one-way decomposition preserves FO-rewritability, we show that it preserves BDDP.
\begin{theorem}\label{thm:1adBDDP}
A rule set \setR is BDDP iff \oad(\setR) is BDDP.
\end{theorem}
\begin{proof}[Sketch]
$(\Rightarrow)$: For factbases $\FB$ restricted to the original vocabulary $\Sigma$, we prove that $\funCh{i}{\FB,\setR} =\funCh{2i}{\FB,\oadR}_{\mid \Sigma}$. Dealing with arbitrary factbases is tackled by a weakening of this correspondence.
$(\Leftarrow)$: If $\setR$ is not BDDP, there are $Q$ and $\{F_i\}_{i \in \mathbb{N}}$ such that for all $i$,
 $F_i, \setR \models Q$ and $\funCh{i}{F_i,\setR} \not \models Q$.
 Since for all $F_i$ on $\Sigma$, $\funCh{i}{F_i,\setR} = \funCh{2i}{F_i,1ad(\setR})_{\mid \Sigma}$, 
 it holds that
 $\funCh{2i}{F_i,1ad(\setR}) \not \models Q$, hence $\oadR$ is not BDDP. 
 \end{proof}

\section{Two-Way Atomic Decomposition}\label{sec:2ad}

Despite the fact that it produces a conservative extension of the original rule set, the one-way atomic decomposition does not preserve the existence of a finite universal model; hence, it does not preserve equivalent chase termination.

\begin{example}
    As in the proof of Theorem~\ref{thm:1adE}, consider $\setR = \{\eqref{rule:p-to-p-chain}\}$, its decomposition $\oad(\setR) = \{(\ref{rule:p-to-p-chain-1}\text{--}\ref{rule:p-to-p-chain-3})\}$, and the factbase $\FB=\ens{P(a, b)}$.
Then, $\setU=\ens{P(a, b), P(b, z_1), P(z_1, b)}$ is a finite universal model for $\Tuple{\setR,\FB}$ that cannot be extended (keeping the same domain) into is a universal model of $\Tuple{\oad(\setR),\FB}$. Indeed, the set \[\{P(a, b), P(b, z_1), P(z_1, b), X_\eqref{rule:p-to-p-chain}(b, z_1), X_\eqref{rule:p-to-p-chain}(z_1, b)\}\] is the 
  smallest extension of \setU that is a model for $\Tuple{\oad(\setR),\FB}$, but it is not universal.
\end{example}

Hence, we define a notion similar to that of conservative extension, but whose purpose is to  guarantee the preservation of the equivalent chase termination.

\begin{definition}\label{def:uce}
    Let \setR and $\setR'$ be two rule sets such that $\EI{\Preds}{\setR} \subseteq \EI{\Preds}{\setR'}$. The set $\setR'$ is a \emph{universal-conservative extension} of the set \setR if, for any factbase $\FB$ with $\EI{\Preds}{\FB} \subseteq \EI{\Preds}{\setR}$,
    \begin{enumerate}
        \item The restriction of any universal model of $\langle\setR', F\rangle$ to the predicates in $\EI{\Preds}{\setR}$ is a universal model of $\langle\setR, F\rangle$.
        \item Any universal model \setM of $\langle\setR, F\rangle$ can be extended to a universal model of $\langle\setR', F\rangle$ that has the same domain and agrees with \setM on the interpretation of $\EI{\Preds}{\setR}$. 
    \end{enumerate}
\end{definition}

We now introduce a normalisation procedure that produces universal-conservative extensions: 

\begin{definition}\label{def:2ad}
    The \emph{two-way atomic decomposition} of a rule $R = B[\vx, \vy] \to \exists \vz . H[\vx, \vz]$ is the rule set $\tad(R) = \oad(R)\cup \ens{H[\vx, \vz]\to X_R(\vx, \vz)}$, with $X_R$ the fresh predicate in $\oad(R)$.
    For a rule set \setR, we let $\tad(\setR) = \bigcup_{R \in \setR} \tad(R)$.
\end{definition}
    
\begin{example}
        Consider again the rule $\eqref{rule:one-way-input}$ from Example \ref{example-1ad}. 
 Then its two-way atomic decomposition is 
 $\tad(\eqref{rule:one-way-input}) = \oad(\eqref{rule:one-way-input}) \cup \{ \eqref{rule:two-way-output}\}$:
  \begin{align}
  P(x, z) \wedge S(x, y, z) &\to X_\eqref{rule:one-way-input}(x, y, z) \label{rule:two-way-output}
   \end{align}

\end{example}

We establish that this new decomposition is indeed a normalisation procedure that has the desired property.

\begin{proposition}\label{prop:adce}
    The rule set \tadR is a conservative extension and a universal-conservative extension of \setR.
\end{proposition}

We can now focus our interest again on chase termination. 
Both atomic decompositions behave like the single-piece decomposition (Theorem~\ref{thm:sptOSO}) regarding the oblivious and the semi-oblivious chase:

\begin{theorem}\label{thm:2adOSO} Both atomic decompositions preserve the termination of the \Ob-chase and the \SO-chase.
\end{theorem}

\begin{proof}[Sketch]
    Consider $\X\in\ens{\Ob,\SO}$ and $\KB=\Tuple{\setR,\FB}$ a KB. First note that for these \X-chases, applying a rule cannot prevent the application of another rule. Hence, since $\oadR\subseteq\tadR$, it is sufficient to prove the result for $\tadR$.
    The proof is similar to that of Theorem~\ref{thm:sptOSO}:
    we show by induction on an arbitrary derivation \der from \KB that there is an injective homomorphism from $\funCh{\X}{\langle \tadR, F \rangle}$ restricted to the predicates in $\EntitiesIn{\Preds}{\setR}$ to $\funCh{\X}{\KB}$, which leads to a similar conclusion. 
\end{proof}

The behavior of the restricted chase is again less easily characterized, as we will see in the next results. 

\begin{theorem}\label{thm:2adExR}
    The two-way atomic decomposition preserves sometimes-termination of the \R-chase; it may also gain termination of this chase variant.
\end{theorem}

\begin{proof}[Sketch]
    To prove preservation, consider a KB $\KB=\Tuple{\setR,\FB}$ such that $\setR \in \chaseterm{\R}{\exists}$. Then, there is a terminating \R-derivation \der from \KB. We can then show by induction that if a trigger $t=(R, \pi)$ with $R=B\to H$ is applied at some step, the trigger $t'=(B\to X_R, \pi)$
  is  applicable at the same step, then the triggers $t_i=(X_R\to H_i, \pi^R)$ also are, and that applying $t'$ and all the $t_i$ successively yields the same result (when restricted to the predicates in $\setR$) as applying $t$. This shows that we can replicate a terminating derivation, and thus that the sometimes-termination is preserved.

    We now present an example where we gain termination. Consider the rule set $\setR = \ens{(\ref{rule:2ad-nst-r-1}\text{--}\ref{rule:2ad-nst-r-5})}$:
    \begin{align}
        A(x) \to~ &\exists y,z ~R(x,x,x)\wedge R(x,y,z) \label{rule:2ad-nst-r-1} \\
        R(x,y,z) &\to \exists t. R(x,x,t) \label{rule:2ad-nst-r-2} \\
        R(x,x,y) &\to \exists z. S(x,y,z) \label{rule:2ad-nst-r-3} \\
        R(x,x,y)\wedge S(x,y,z) &\to S(x,x,x) \label{rule:2ad-nst-r-4} \\
        A(x) \wedge S(x,x,y) &\to A(y) \label{rule:2ad-nst-r-5}
    \end{align}
    There is no terminating \R-derivation on the KB $\Tuple{\setR,\ens{A(a)}}$ but there is one on $\Tuple{\tadR,\FB}$ for any \FB.
    \qedhere    
\end{proof}

\begin{theorem}\label{thm:2adR}
    The two-way atomic decomposition does not preserve the termination of the \R-chase.
\end{theorem}

\begin{proof}
    Consider the rule set $\setR = \{\eqref{rule:p-to-p-chain}\}$ introduced in the proof of Theorem~\ref{thm:1adE} and
    $\tad(\setR) = \{(\ref{rule:p-to-p-chain-1}\text{--}\ref{rule:p-to-p-chain-3}),\eqref{rule:p-to-p-chain-back}\}$:
    \begin{align}
        P(x, y) &\to \exists z\dotq P(y,z)\wedge P(z,y)  \tag{\ref{rule:p-to-p-chain}} \\
        P(x, y) &\to \exists z .  X_\eqref{rule:p-to-p-chain}(y,z)  \tag{\ref{rule:p-to-p-chain-1}} \\
        X_\eqref{rule:p-to-p-chain}(y, z) &\to P(y, z)  \tag{\ref{rule:p-to-p-chain-2}} \\
        X_\eqref{rule:p-to-p-chain}(y, z) &\to P(z, y)  \tag{\ref{rule:p-to-p-chain-3}} \\
        P(y,z)\wedge P(z,y) &\to X_\eqref{rule:p-to-p-chain}(y, z) \label{rule:p-to-p-chain-back}
    \end{align}
    The \R-chase yields the same result as the \E-chase  on \setR, so $\setR\in\chaseterm{\R}{\forall}$. We then construct an infinite derivation from $\Tuple{\tad,P(a,b)}$. First, apply \eqref{rule:p-to-p-chain-1}, and \eqref{rule:p-to-p-chain-2}. Then, repeat the following pattern: \eqref{rule:p-to-p-chain-1}, \eqref{rule:p-to-p-chain-2} (on the new variable), then \eqref{rule:p-to-p-chain-3} and \eqref{rule:p-to-p-chain-back} (on the variables of the previous loop). Applying \eqref{rule:p-to-p-chain-1} again before applying \eqref{rule:p-to-p-chain-back} yields an infinite chain. 
\end{proof}

Again, the \R-chase is not well-behaved with respect to atomic decomposition. However, the \DF{\R}-chase behaves exactly as desired regarding the two-way atomic decomposition. In fact, we can show an even stronger result: any \DF{\R}-derivation from a KB \Tuple{\setR, \FB} can be replicated by a \DF{\R}-derivation from \Tuple{\tad(\setR), \FB}, and conversely.
\begin{theorem}\label{thm:2adDfR}
The 2-way atomic decomposition has no impact on the (sometimes-)termination of the \DF{\R}-chase; i.e., $\chaseterm{\DF{\R}}{\forall}=\tad(\chaseterm{\DF{\R}}{\forall})$ and $\chaseterm{\DF{\R}}{\exists}=\tad(\chaseterm{\DF{\R}}{\exists})$.
\end{theorem}
\begin{proof}[Sketch]
    One can prove that any fair \DF{\R}-derivation \der from \Tuple{\setR,\FB} with $\setR\in\chaseterm{\DF{\R}}{\forall}$ can be replicated to yield a fair \DF{\R}-derivation $\der'$ from \Tuple{\tad(\setR),\FB} such that $\der'$ is finite if and only if \der is. The reciprocal is also true.
\end{proof}

The following result follows from Proposition~\ref{prop:adce}\begin{theorem}:\label{thm:2adE}
    The two-way atomic decomposition preserves the termination of the \E-chase.
\end{theorem}

The single-piece decomposition may gain termination for some chase variants (Theorem \ref{thm:sp-improves}); we are interested to know if the same can happen with atomic decompositions.
Unfortunately, there is no way for a non-terminating rule set to gain termination, as stated next:
\begin{proposition}\label{prop:adNonTerm}
    If a chase variant does not terminate on a rule set \setR, it does not terminate on \oadR and \tadR.
\end{proposition}
\begin{proof}[Sketch]
For each trigger with a rule $R$ in the original infinite fair derivation, one can consider the corresponding triggers with $\oad(R)$ or $\tad(R)$, and thus produce an infinite fair derivation.
\end{proof}

Regarding FO-rewritability, the two-way atomic decomposition behaves similarily to the one-way atomic decomposition 
(which can be proven similarly, see Theorem \ref{thm:1adBDDP}).
\begin{theorem}\label{thm:2adBDDP}
    A rule set \setR is BDDP iff \tad(\setR) is BDDP.
\end{theorem}

\section{No Normalisation for the Restricted Chase}\label{sec:noNormR}

Normalisation procedures studied so far do not maintain the status of the termination of the $\R$-chase. This raises the question of the existence of such a procedure. We show here that no computable function can map rule sets to sets of rules having atomic head while preserving termination and non-termination of the $\R$-chase. To do that, we show that with atomic-head rules, the class of rule sets $\chaseterminst{\R}{\FB}{\forall}$ for which every fair $\R$-derivation from $\Tuple{\setR,\FB}$ is finite is a recursively enumerable set. With arbitrary rules, we show it is
hard for $\Pi_2^0$, the second level of the arithmetic hierarchy \cite{ThRecFunEffComp}.  A complete problem for $\Pi_2^0$ is to decide whether a given Turing machine halts on every 
input word; it remains complete when inputs are restricted to words on a unary alphabet.

\begin{proposition}
\label{prop-single-head-re}
For any factbase $\FB$, the subset of $\chaseterminst{\R}{\FB}{\forall}$  composed of sets of atomic-head rules is recognizable.
\end{proposition}

\begin{proof}[Sketch]
 With atomic-head rules, it is known that the existence of an infinite fair restricted derivation is equivalent to the existence of an infinite restricted derivation \cite{DBLP:conf/pods/GogaczMP20}. Using K\"onig's lemma, one can show that the chase terminates iff there exists a $k$ such that any fair $\R$-derivation is of length at most $k$. 
\end{proof}

\begin{figure*}
\tikzset{every loop/.style={min distance=50,in=125,out=185,looseness=20}}

\begin{center}
\scalebox{0.8}{
\begin{tikzpicture}[roundnode/.style={circle, fill=black, inner sep=0pt, minimum size=1mm},
labelnode/.style={fill=white,inner sep=1pt}]

\newcommand{\xsep}{7}
\newcommand{\ysep}{1}
\newcommand{\xa}{0.5*\xsep}
\newcommand{\xb}{1.25*\xsep}
\newcommand{\xc}{2*\xsep}
\newcommand{\xd}{2.75*\xsep}
\newcommand{\xe}{4*\xsep}
\newcommand{\xf}{4.75*\xsep}
\newcommand{\xg}{5.5*\xsep}
\newcommand{\xh}{6.75*\xsep}

\newcommand{\ya}{0}
\newcommand{\yf}{1*\ysep}
\newcommand{\yca}{2*\ysep}
\newcommand{\ycb}{3*\ysep}

\node[fill,circle,inner sep=0pt,minimum size=5pt, label=-3:{$a : \Init$}] at (\xa, \ya) (1) {};
\node[fill,circle,inner sep=0pt,minimum size=5pt, label=-3:{$nf_1:\Real$}] at (\xb, \ya) (2) {};
\node[fill,circle,inner sep=0pt,minimum size=5pt, label=-3:{$nf_2:\Real$}] at (\xc, \ya) (3) {};
\node[fill,circle,inner sep=0pt,minimum size=5pt, label=-3:{$b:\Brake,\{\Content{a}\},\{\HeadState{q}\},\First,\End,(\Real)$}] at (0.5*\xb, -0.6*\ycb) (brake) {};

\node[fill,circle,inner sep=0pt,minimum size=5pt, label=-3:{$f_2$}] at (\xb+0.25*\xsep, \yf) (5) {};
\node[fill,circle,inner sep=0pt,minimum size=5pt, label=-3:{$f_3$}] at (\xc+0.25*\xsep, \yf) (5b) {};

\node[fill,circle,inner sep=0pt,minimum size=5pt, label=-3:{$c_2^2:\Content{\Blank},\End$}] at (\xb+0.25*\xsep, \yca) (7) {};
\node[fill,circle,inner sep=0pt,minimum size=5pt, label=-3:{$c_1^2:\Content{1}$}] at (\xb-0.2*\xsep, \yca) (7b) {};
\node[fill,circle,inner sep=0pt,minimum size=5pt, label=-3:{$c_0^2:\Content{1},\First,\HeadState{q_I}$}] at (\xa+0.15*\xsep, \yca) (7c) {};

\node[fill,circle,inner sep=0pt,minimum size=5pt, label=-3:{$c_3^3:\Content{\Blank},\End$}] at (\xd, \ycb) (9) {};
\node[fill,circle,inner sep=0pt,minimum size=5pt, label=-3:{$c_2^3:\Content{1}$}] at (\xc, \ycb) (10) {};
\node[fill,circle,inner sep=0pt,minimum size=5pt, label=-3:{$c_1^3:\Content{1}$}] at (\xb, \ycb) (11) {};
\node[fill,circle,inner sep=0pt,minimum size=5pt, label=-3:{$c_0^3:\Content{1},\First,\HeadState{q_I}$}] at (\xa, \ycb) (12) {};

\node[fill,circle,inner sep=0pt,minimum size=5pt, label=-3:{$c_0^0:\First,\End,\Content{\Blank},\HeadState{q_I}$}] at (\xd, \ya) (t1) {};
\node[fill,circle,inner sep=0pt,minimum size=5pt, label=-3:{$c_0^1:\First,\Content{1},\HeadState{q_I}$}] at (\xc, -\yf) (t3) {};
\node[fill,circle,inner sep=0pt,minimum size=5pt, label=-3:{$c_1^1:\Content{\Blank},\End$}] at (\xd, -\yf) (t2) {};
 \path[->] (t3) edge[dotted] node[labelnode] {} (t2);

 \path[->] (1) edge node[labelnode] {$\NonFinal$} (2);
 \path[->,dashed] (2) edge node[labelnode] {$\NonFinal$} (3);
 \path[->,dashed] (3) edge node[labelnode] {$\Final$} (5b);
 \path[->,dashed] (2) edge node[labelnode] {$\Final$} (5);
 \path[->,dashed] (5) edge node[labelnode] {$\Done$} (7);
 \path[->,dashed] (7b) edge[dotted] node[labelnode] {} (7);
 \path[->,dashed] (2) edge node[labelnode] {$\Done$} (7b);
 \path[->,dashed] (1) edge node[labelnode] {$\Done$} (7c);
 \path[->] (7c) edge[dotted] node[labelnode] {} (7b);
 \path[->,dashed] (5b) edge node[labelnode] {$\Done$} (9);
 \path[->,dashed] (3) edge node[labelnode] {$\Done$} (10);
 \path[->,dashed] (2) edge node[labelnode] {$\Done$} (11);
 \path[->,dashed] (1) edge node[labelnode] {$\Done$} (12);
 \path[->] (12) edge[dotted] node[labelnode] {} (11);
 \path[->] (11) edge[dotted] node[labelnode] {} (10);
 \path[->] (10) edge[dotted] node[labelnode] {} (9);
 \path[->] (2) edge node[labelnode] {$\Done,\NonFinal$} (brake);
 \path[->,dashed] (3) edge node[labelnode] {$\Done,\NonFinal$} (brake);
%
\path[->] (brake) edge [loop above] node[labelnode,pos=0.6] {$\Done, \NonFinal,\Next,\NextPlus,\Step$} (brake);

\end{tikzpicture}
}
\end{center}

\caption{The effect of $\setR_w$ on $\FB$: Dashed atoms are $\Done,\Final$ or $\NonFinal$ atoms generated by the chase; $\Next$ atoms used by $\setR_M$ are dotted.}
\label{figure:no-nf-restricted-chase}
\end{figure*}

\begin{proposition}
\label{prop-hardness-termination}
There exists a factbase $\FB$ such that $\chaseterminst{\R}{\FB}{\forall}$ is $\Pi_2^0$-hard.\footnote{
Note that this contradicts the first item of Theorem 5.1 in \cite{anatomychase}. However, no proof is given for that statement, which is uncorrectly attributed to \cite{chaserevisited}.}
\end{proposition}

\begin{proof}[Sketch]
 Given a Turing machine (TM) $M$ whose input alphabet is unary, we build a KB $\KB = \Tuple{\setR_w \cup \setR_M, F}$ s.t. every fair $\R$-derivation from $\KB$ is finite iff $M$ halts on every input. 
Regardless the chase variant, simulating a TM with a rule set such that the chase terminates whenever the TM halts is classical; we reuse the rule set $\setR_M$ provided in \cite{DBLP:conf/kr/BourgauxCKRT21}, which we recall in Figure~\ref{rule-tm-simulation} for self-containment. We show that we can assume wlog that all the rules of $\setR_M$ are applied after all the rules of $\setR_w$ (listed in Figure~\ref{rule-tape-creation}). The set $\setR_w$ is used to generate from $\FB$ arbitrarily large input tape representations in a terminating way. To ensure that any fair $\R$-derivation from $\Tuple{\setR_w, F}$ terminates, we reuse the emergency brake technique from \cite{DBLP:conf/icdt/KrotzschMR19}, which allows one to stop the derivation at any desired length. 
 The representation of an input word of length $j$ is a set of atoms of the shape $\{\Next(c_i^j,c_{i+1}^j),\Content{1}(c_i^j) \mid 0 \leq i<j\} \cup \{\Content{\Blank}(c_j^j),\First(c_0^j),\End(c_j^j)\}$. 
  As detailed below, the factbase $\FB$ contains the representation of the input words of length $0$ and $1$ (Item \ref{item-init}), atoms used as seeds to build larger words (Item \ref{item-seed}) and atoms that initialize the emergency brake (Items \ref{item-brake} and \ref{item-brake-tm}): 
  \begin{enumerate}
 \item\label{item-init} $\First(c_0^1), \Content{1}(c_0^1),\Next(c_0^1,c_1^1), \End(c_1^1), \Content{\Blank}(c_1^1)$, $\First(c_0^0), \End(c_0^0), \Content{\Blank}(c_0^0)$
 \item\label{item-seed} $\Init(a), \NonFinal(a,nf_1), \Real(nf_1), \NonFinal(nf_1,b), \Done(nf_1,b)$
 \item\label{item-brake} $\Brake(b),\Final(b,b),\NonFinal(b,b),\Done(b,b), \Next(b,b),\Last(b),\First(b)$
 \item\label{item-brake-tm} $\HeadState{s}(b),\Content{l}(b),\End(b),\Step(b,b),\NextPlus(b,b)$
\end{enumerate}
The chase works as follows: after generating a \emph{non-final} ($\NonFinal$) chain with Rule~(\ref{main-rule-chain}), the \emph{brake} ($\Brake$) is made \emph{real} ($\Real$) by Rule~(\ref{main-rule-brake}), which prevents any extension of the non-final chain through restricted rule applications. A \emph{final} ($\Final$) element is added after each non-final element by Rule~(\ref{main-rule-final}), and from each final element a tape is created, by traversing the chain, marking as \emph{done} ($\Done$) processed elements, thanks to Rules~(\ref{main-rule-end})-(\ref{main-rule-first}). Figure~\ref{figure:no-nf-restricted-chase} depicts the result of any $\R$-chase derivation from $\Tuple{\setR_w, F}$ in which $\Real(b)$ has been derived after exactly one application of Rule~(\ref{main-rule-chain}). 
Rule~(\ref{main-rule-head}) sets the initial state on the first cell.
%
%
\end{proof}

\begin{figure}
\begin{align}
&\begin{aligned}\label{main-rule-chain}
\Brake(b) \wedge \NonFinal(z,x) \wedge \Real(x) \to \exists y . &\NonFinal(x,y) \wedge \Real(y) \wedge \\
&\Done(y,b) \wedge \NonFinal(y,b) \\
\end{aligned}\\
&\Brake(b) \to \Real(b) \label{main-rule-brake}\\
&\NonFinal(x, y) \to \exists z . \Final(y, z) \label{main-rule-final}\\
&\Final(x, y) \to \exists z . \Done(y, z) \wedge \End(z) \wedge \Content{\Blank}(z) \label{main-rule-end} \\
&\begin{aligned}
\NonFinal(t,x) \wedge \NonFinal(x,y) \wedge \Done(y,z) \to \exists u . &\Next(u,z) \wedge \Done(x,u)\\
&\wedge \Content{1}(u) \\
\end{aligned}\\
&\begin{aligned}
\NonFinal(t,x) \wedge \Final(x,y) \wedge \Done(y,z) \to \exists u . &\Next(u,z) \wedge \Done(x,u)\\
&\wedge \Content{1}(u) \\
\end{aligned}\\
&\begin{aligned}\label{main-rule-first}
\Init(x) \wedge \NonFinal(x,y) \wedge \Done(y,z) \to \exists u . &\Next(u,z) \wedge \Done(x,u)\\
&\wedge \Content{1}(u) \wedge \First(u) 
\end{aligned}\\
&\First(x) \to \HeadState{q_I}(x) \label{main-rule-head}
\end{align}
\caption{Rules $\setR_w$ to create the initial tapes}
\label{rule-tape-creation}
\end{figure}

\begin{figure}
\begin{align*}
 \Next(x,y) \to \NextPlus(x,y)& \\
 \NextPlus(x,y) \wedge \NextPlus(y,z) \to \NextPlus(x,z)& \\
 \Next(x,y) \wedge \Step(x,z) \wedge \Step(y,w) \to \Next(z,w)& \\
 \End(x) \wedge \Step(x,z) \to \exists v. \Next(z,v) \wedge \Content{\Blank}(v) \wedge \End(v)& \\
 \HeadState{q}(x) \wedge \NextPlus(x,y) \wedge \Content{c}(y) \to \exists z . \Step(y,z) \wedge \Content{c}(z)& \\
  \HeadState{q}(x) \wedge \NextPlus(y,x) \wedge \Content{c}(y) \to \exists z . \Step(y,z) \wedge \Content{c}(z)& \\
 \HeadState{q}(x) \wedge \Content{a}(x) \to \exists z . \Step(x,z) \wedge \Content{b}(z) &\\
  \HeadState{q}(x) \wedge \Content{a}(x) \wedge \Step(x,z) \wedge \Next(z,w) \to \HeadState{r}(w) &\\
 \HeadState{q}(x) \wedge \Content{a}(x) \wedge \Step(x,z) \wedge \Next(w,z) \to \HeadState{r}(w)& 
\end{align*}
\caption{Rules $\setR_M$ for the Turing Machine simulation: the last three rules are instantiated w.r.t.the transition function of $M$.}
\label{rule-tm-simulation}
\end{figure}

As it is known recursively enumerable sets are a strict subsets of $\Pi_2^0$ \cite{ThRecFunEffComp}, the following theorem follows.

\begin{theorem}
\label{thm-no-nf-restricted}
No computable function $f$ exists that maps rule sets to rule sets having atomic-head rules such that $\setR \in \chaseterm{\R}{\forall}$ if and only if $f(\setR) \in \chaseterm{\R}{\forall}$.
\end{theorem}

This applies in particular to normalisation procedures producing rule sets with atomic-head rules.

\section{Conclusion}
\label{section:conclusions}

As shown in this paper, normalisation procedures do have an impact, sometimes unexpected, on chase termination. This is particularly true regarding the restricted chase, which is the most relevant in practice but also the most difficult to control. 
We extend the understanding of its behavior by three results. We show that the Datatog-first strategy is in fact not always the most terminating, which goes against a common belief. We introduce a new atomic-decomposition (two-way), which behaves nicely, in particular regarding the Datalog-first restricted chase, but still has a negative impact on the restricted chase termination. 
This leads to us to show a more fundamental decidability result, which implies that no computable atomic-decomposition exists that exactly preserves the termination of the restricted chase (i.e., termination and non-termination). Note however that our result does not rule out the existence of a computable normalisation procedure  into atomic-head rules that would improve the termination of the restricted chase, although this seems unlikely. 
Future work includes investigating normalisation procedures for first-order logical formulas, to translate these into the existential rule framework.

\section*{Acknowledgements} This work was partly supported by the ANR project CQFD (ANR-18-CE23-0003).

\bibliographystyle{kr}
\bibliography{references}

\begin{tr}
\appendix
\onecolumn
\Large
\section{Proofs of Section~\ref{section:chase-terminating-sets}}

\begin{reptheorem}{theorem:chase-terminating-sets-equality}
For every $\X \in \{\Ob, \SO, \E\}$, we have that $\chaseterm{\X}{\forall} = \chaseterm{\X}{\exists} = \chaseterm{\DF{\X}}{\forall} = \chaseterm{\DF{\X}}{\exists}$.
\end{reptheorem}
\begin{proof}
The theorem follows from (1), (2), and (3).
\begin{enumerate}
\item We have that $\chaseterm{\Ob}{\forall} = \chaseterm{\Ob}{\exists} = \chaseterm{\DF{\Ob}}{\forall} = \chaseterm{\DF{\Ob}}{\exists}$.
\begin{enumerate}
\item[A.] Consider some fair \Ob-derivations $\der = (\emptyset, \FB_0), (t_1, \FB_1), (t_2, \FB_2), \ldots$ and $\der'$ from some KB $\KB$.
We can show via induction that $\FB_i \subseteq \funres{\der'}$ for every $i \geq 1$.
\item[B.] By (1.A), all fair \Ob-derivations from the same KB produce the same result.
Therefore, if a KB \KB admits one fair infinite \Ob-derivation, then all fair \Ob-derivations from \KB are infinite.
\item[C.] Consider a rule set \setR and the following cases:
\begin{itemize}
\item There is a KB of the form $\Tuple{\setR, \FB}$ that admits one infinite fair \Ob-derivation.
Then, all fair \Ob-derivations from the KB \Tuple{\setR, \FB} are infinite by (1.B).
Therefore:
$$\setR \notin \chaseterm{\Ob}{\forall} \cup \chaseterm{\Ob}{\exists} \cup \chaseterm{\DF{\Ob}}{\forall} \cup \chaseterm{\DF{\Ob}}{\exists}$$
\item There is no KB of the form $\Tuple{\setR, \FB}$ that admits an infinite fair \Ob-derivation.
That is, every fair \Ob-derivation from a KB of the form $\Tuple{\setR, \FB}$ is finite.
Therefore:
$$\setR \in \chaseterm{\Ob}{\forall} \cap \chaseterm{\Ob}{\exists} \cap \chaseterm{\DF{\Ob}}{\forall} \cap \chaseterm{\DF{\Ob}}{\exists}$$
\end{itemize}
\item[D.] By (1.C), we have that, for a rule set \setR:
 $$\setR \notin \chaseterm{\Ob}{\forall} \cup \chaseterm{\Ob}{\exists} \cup \chaseterm{\DF{\Ob}}{\forall} \cup \chaseterm{\DF{\Ob}}{\exists} \text{ or }\setR \in \chaseterm{\Ob}{\forall} \cap \chaseterm{\Ob}{\exists} \cap \chaseterm{\DF{\Ob}}{\forall} \cap \chaseterm{\DF{\Ob}}{\exists}$$
\end{enumerate}
\item We have that $\chaseterm{\SO}{\forall} = \chaseterm{\SO}{\exists} = \chaseterm{\DF{\SO}}{\forall} = \chaseterm{\DF{\SO}}{\exists}$.
\begin{enumerate}
\item[A.] Given some fair \SO-derivations $\der = (\emptyset, \FB_0), (t_1, \FB_1), (t_2, \FB_2), \ldots$ and $\der'$ from a KB $\KB$, we can show via induction that, for every $i \geq 1$, there is an injective homomorphism $\Hom_i$ from $\FB_i$ to $\funres{\der'}$.
\item[B.] By (a), all fair \SO-derivations from a KB \KB are infinite if \KB admits one fair infinite \SO-derivation.
\item[C.] The remainder of the proof is analogous to the one of (1).
\end{enumerate}
\item We have that $\chaseterm{\E}{\forall} = \chaseterm{\E}{\exists} = \chaseterm{\DF{\E}}{\forall} = \chaseterm{\DF{\E}}{\exists}$.
\begin{enumerate}
\item[A.] Every fair \E-derivation from a KB \KB is finite iff \KB admits a finite universal model.
\item[B.] By (a), all fair \E-derivations from a KB \KB are infinite if \KB admits one fair infinite \E-derivation.
\item[C.] The remainder of the proof is analogous to the one of (1).
\end{enumerate}
\end{enumerate}
\end{proof}

\begin{reptheorem}{theorem:chase-terminating-sets-strict-subsets}
The following hold: $\chaseterm{\Ob}{\forall} \subset \chaseterm{\SO}{\forall} \subset \chaseterm{\R}{\forall} \subset \chaseterm{\DF{\R}}{\forall} \subset \chaseterm{\DF{\R}}{\exists} \subset \chaseterm{\R}{\exists} \subset \chaseterm{\E}{\forall}$
\end{reptheorem}
\begin{proof}
The theorem follows from (1--5) and Lemma~\ref{lemma:theorem-2-aux}.
\begin{enumerate}
\item We obtain $\chaseterm{\Ob}{\forall} \subset \chaseterm{\SO}{\forall} \subset \chaseterm{\R}{\forall}$ from Theorem~4.5 in \cite{anatomychase}.
\item We obtain $\chaseterm{\R}{\forall} \subseteq \chaseterm{\DF{\R}}{\forall} \subseteq \chaseterm{\DF{\R}}{\exists} \subseteq \chaseterm{\R}{\exists}$ from Definition~\ref{definition:chase-terminating-sets}.
\item We verify that the rule set $\setR = \{P(x, y) \to \exists z . P(y, z), P(x, y) \to P(y, x)\}$ is in $\chaseterm{\DF{\R}}{\forall} \setminus \chaseterm{\R}{\forall}$.
\begin{enumerate}
\item[A.] The rule set $\setR$ is in $\chaseterm{\DF{\R}}{\forall}$.
Note that $\FB \cup \{P(t, u) \mid P(u, t) \in \FB\} = \funres{\der}$ for every factbase \FB and every fair \DF{\R}-derivation \der from \Tuple{\setR, \FB}.
\item[B.] The rule set $\setR$ is not in \chaseterm{\R}{\forall}.
Note that the KB \Tuple{\setR, P(a, b)} admits an infinite fair \R-derivation $\der = (\emptyset, \FB), (t_1, \FB_1), (t_2, \FB_2), \ldots$ such that
\begin{align*}
\FB_1 = \{&P(b, z_1)\} \cup \FB_0,	& \FB_2 = \{&P(b, a)\} \cup F_1,	& \FB_3 = \{&P(z_1, z_2)\} \cup F_1, \\
\FB_4 = \{&P(z_1, b)\} \cup \FB_3,	& \FB_5 = \{&P(z_2, z_3)\} \cup F_4,	& \FB_6 = \{&P(z_2, z_1)\} \cup F_5, \ldots
\end{align*}
\end{enumerate}
\item We verify that the rule set $\setR = \{P(x, y) \to \exists z . P(y, z) \wedge P(z, y), P(x, y) \to \exists z . P(y, z)\}$ is in $\chaseterm{\DF{\R}}{\exists}$ but not in $\chaseterm{\DF{\R}}{\forall}$.
\begin{enumerate}
\item[A.] The rule set $\setR$ is in \chaseterm{\DF{\R}}{\exists}.
Note that, for every factbase $\FB$, there is a finite fair \R-derivation \der from \Tuple{\setR, \FB} and an injective homomorphism from $\funres{\der}$ to $\FB \cup \{P(t, x_t), P(x_t, t) \mid t \in \EI{\Terms}{\FB}\}$ where $x_t$ is a fresh variable unique for $t$.
This derivation can be obtained by exhaustively applying the rule $P(x, y) \to \exists z . P(y, z) \wedge P(z, y)$ on \FB.
\item[B.] The rule set from (5) is not in \chaseterm{\DF{\R}}{\forall}.
Note that the KB \Tuple{\setR, P(a, b)} admits an infinite fair \R-derivation $\der = (\emptyset, \FB), (t_1, \FB_1), (t_2, \FB_2), \ldots$ such that
\begin{align*}
\FB_1 = \{&P(b, z_1)\} \cup \FB_0,	& \FB_2 = \{&P(b, v_1), P(v_1, b)\} \cup F_1, \\
\FB_3 = \{&P(z_1, z_2)\} \cup F_1,	& \FB_4 = \{&P(z_1, v_2), P(v_2, z_1)\} \cup \FB_3, \\
\FB_5 = \{&P(z_2, z_3)\} \cup F_4,	& \FB_6 = \{&P(z_2, v_3), P(v_3, z_2)\} \cup F_5, \ldots
\end{align*}
\end{enumerate}
%
%
\item We obtain $\chaseterm{\R}{\exists} \subset \chaseterm{\E}{\forall}$ from Proposition~4.7 in \cite{anatomychase}.
Note that a rule set $\setR$ is in $\chaseterm{\E}{\forall}$ iff the core chase terminates for \setR iff the core chase as defined in \cite{anatomychase} terminates for every KB with \setR.
\end{enumerate}
\end{proof}

\begin{lemma}
\label{lemma:theorem-2-aux}
There is a rule set in $\chaseterm{\R}{\exists}$ that is not in $\chaseterm{\DF{\R}}{\exists}$.
\end{lemma}
\begin{proof}
We show that the rule set $\setR = \{\text{(\ref{rule:r-loop}--\ref{rule:s-succ})}\}$ is in \chaseterm{\R}{\exists} but not in \chaseterm{\DF{\R}}{\exists}.
\vspace*{-\baselineskip}
\begin{center}
\begin{minipage}[t]{0.45\linewidth}
\begin{align}
A(x) &\to R(x, x) \tag{\ref{rule:r-loop}} \\
R(x, y) \wedge S(y, z) &\to S(x, x) \tag{\ref{rule:s-loop}} \\
A(x) \wedge S(x, y) &\to A(y) \tag{\ref{rule:a-prop}}
\end{align}
\end{minipage}
\begin{minipage}[t]{0.45\linewidth}
\begin{align}
A(x) &\to \exists w ~R(x, w) \tag{\ref{rule:r-succ}}\\
R(x, y) &\to \exists v ~S(y, v) \tag{\ref{rule:s-succ}}
\end{align}
\end{minipage}
\end{center}

To show that $\setR \in \chaseterm{\R}{\exists}$ we prove that every KB of the form $\Tuple{\setR, \FB}$ admits a terminating \R-derivation:
\begin{enumerate}
\item We define a derivation \der{} that is computed by starting with the factbase \FB and then applying the rules in \setR in the following manner:
\begin{enumerate}
\item[A.] Apply rules \eqref{rule:a-prop}, \eqref{rule:r-succ}, and \eqref{rule:s-succ} exhaustively (in any arbitrary order).
\item[B.] Apply rule \eqref{rule:s-loop} exhaustively.
\item[C.] Apply rule \eqref{rule:r-loop} exhaustively.
\end{enumerate}
\item After the application of the rules in Step (1.A), every term in \der{} is of depth $2$ or smaller.
In Steps~(1.B) and (1.C), we only apply Datalog rules and hence, no new terms are introduced.
Hence, \der is finite.
\item We argue that no rule in \setR is \R-applicable to the last element of \der{} and hence, this sequence is fair.
\begin{itemize}
\item After Step~(1.A) rules \eqref{rule:a-prop}, \eqref{rule:r-succ}, and \eqref{rule:s-succ} are satisfied.
\item After Step~(1.B) rules \eqref{rule:s-loop}, \eqref{rule:a-prop}, \eqref{rule:r-succ}, and \eqref{rule:s-succ} satisfied.
Note that the only rule that could become \R-applicable after Step~(1.B) is \eqref{rule:a-prop} since it is the only one that features the predicate $S$ in its body.
However, triggers of the form $(\eqref{rule:s-loop}, [x / t, y / t])$ are never \R-applicable to any factbase.
\item After Step~(1.C) all rules in \setR are satisfied.
The only rules that could become \R-applicable after Step~(1.C) are \eqref{rule:s-loop} and \eqref{rule:s-succ} since they are the only ones with the predicate $R$ in their bodies.
However, $S(t, t) \in \funres{\der}$ for every term $t$ such that $A(t) \in \funres{\der}$ due to the exhaustive application of rules \eqref{rule:s-loop}, \eqref{rule:r-succ}, and \eqref{rule:s-succ} in previous steps.
Hence, neither \eqref{rule:s-loop} nor \eqref{rule:s-succ} are \R-applicable at this point.
\end{itemize}
\item By (1--3), \der is a terminating \R-derivation from the KB \Tuple{\setR, \FB}.
\end{enumerate}

\begin{figure}
\tikzset{every loop/.style={min distance=60,in=130,out=50,looseness=20}}
\begin{center}
\begin{tikzpicture}[roundnode/.style={circle, fill=black, inner sep=0pt, minimum size=1mm},
labelnode/.style={fill=white,inner sep=1pt}]

\newcommand{\xsep}{3.75}
\newcommand{\xa}{1*\xsep}
\newcommand{\xb}{2*\xsep}
\newcommand{\xc}{3*\xsep}
\newcommand{\xd}{4*\xsep}
\newcommand{\xe}{5*\xsep}

\newcommand{\ya}{0}

\node[fill,circle,inner sep=0pt,minimum size=5pt, label=-3:{$a : A$}] at (\xa, \ya) (1) {};
\node[fill,circle,inner sep=0pt,minimum size=5pt, label=-3:{$v_{t_2} : A$}] at (\xb, \ya) (2) {};
\node[fill,circle,inner sep=0pt,minimum size=5pt, label=-3:{$v_{t_6} : A$}] at (\xc, \ya) (3) {};
\node[fill,circle,inner sep=0pt,minimum size=5pt, label=-3:{$v_{t_{10}} : A$}] at (\xd, \ya) (4) {};
\node[circle,inner sep=0pt,minimum size=5pt, label=-3:{}] at (\xe, \ya) (5) {};

\path[->] (1) edge [loop above] node[labelnode,pos=0.1] {$R, S$} (1);
\path[->] (2) edge [loop above] node[labelnode,pos=0.1] {$R, S$} (2);
\path[->] (3) edge [loop above] node[labelnode,pos=0.1] {$R, S$} (3);
\path[->] (4) edge [dashed, loop above] node[labelnode,pos=0.1] {$R, S$} (4);

\path[->] (1) edge node[labelnode,pos=0.7] {$S$} (2);
\path[->] (2) edge node[labelnode,pos=0.7] {$S$} (3);
\path[->] (3) edge[dashed] node[labelnode,pos=0.7] {$S$} (4);
\path[->] (4) edge[dashed] node[labelnode,pos=0.7] {} (5);
\end{tikzpicture}
\end{center}

\caption{\Large The only result of the \DF{\R}-chase of the knowledge base $\K = \langle \setR, \{A(a)\} \rangle$.}
\label{figure:infinite-dfr-chase-appendix}
\end{figure}

To show that $\setR \notin \chaseterm{\DF{\R}}{\exists}$ we prove that the knowledge base $\K = \langle \setR, \{A(a)\} \rangle$ only admits infinite \DF{\R}-derivations.
More precisely, we show that
$$\{\funres{\der} \mid \der \text{ a \DF{\R}-derivation from } \K\}$$
is a singleton set containing the infinite factbase depicted in Figure~\ref{figure:infinite-dfr-chase-appendix}.
Note that, given a \DF{\R}-derivation $\der = (\emptyset, F_0), (t_1, F_1), (t_2, F_2), \ldots$ from $\K$, we can show that
\begin{align*}
F_0 =&~ \{A(a)\}				& F_1 =&~ F_0 \cup \{R(a, a)\} \\
F_2 =&~ F_1 \cup \{S(a, v_{t_2})\}	& F_5 =&~ F_2 \cup \{S(a, a), A(v_{t_2}), R(v_{t_2}, v_{t_2})\} \\
F_i =&~ \{S(v_{t_{i-4}}, v_{i})\} \cup F_{i-1}	& F_{i+3} =&~ \{S(v_{t_{i-4}}, v_{t_{i-4}}, A(v_{t_i}), R(v_{t_i}, v_{t_i})\} \cup F_i
\end{align*}
for every $i \geq 6$ that is an even multiple of $4$.
Note how the only rules applicable to $F_0$ and $F_1$ are \eqref{rule:r-loop} and \eqref{rule:s-succ}, respectively.
Since we only consider Datalog-first derivations, we must apply all of the Datalog rules in \setR (in some possible order) to produce $F_5$.
From there on, we must continue applying rule \eqref{rule:s-succ} and the three Datalog rules in \setR in alternation, thus producing an infinite chase \DF{\R}-derivation that yields the factbase depicted in Figure~\ref{figure:infinite-dfr-chase-appendix}.

\end{proof}
\section{Proofs of Section~\ref{sec:spt}}

\begin{repproposition}{prop:spteq}
A rule set \setR is equivalent to the set \spR.
\end{repproposition}

\begin{proof}
Consider a rule $R = B \to \exists \vec{z} . H$ in \setR, the pieces $H_1, \ldots, H_m$ of $R$, and let $\spt(R) = \{R_1, \ldots, R_m\}$.
Existential variables appearing in a piece are disjoint from those appearing in another; otherwise they would not be in different components in the piece graph of $R$.
Hence, $R$ is equivalent to:
\[\forall \vec{x} \forall \vec{y}\dotq B \to (\exists \vec{z}_1\dotq H_1)\wedge \ldots \wedge (\exists \vec{z}_m\dotq H_m)\]
    
For all first-order formulas $A$, $B$, and $C$, $(A \to B) \wedge (A \to C)$ is equivalent to $A \to B\wedge C$.
Hence $R'_1 \wedge \ldots \wedge R'_m$ is equivalent to $R$.
\end{proof}

\begin{replemma}{proposition:oblivious-same-result}
Consider some \X-derivations \der and $\der'$ from a KB \KB.
If $\X = \Ob$, then $\funres{\der} = \funres{\der'}$.
If $\X = \SO$, then \funres{\der} is isomorphic to \funres{\der'}.
\end{replemma}
\begin{proof}
See the proof of Theorem~\ref{theorem:chase-terminating-sets-equality}.
\end{proof}

\begin{reptheorem}{thm:sptOSO}
The single-piece decomposition preserves the termination of the \Ob-chase and \SO-chase.
\end{reptheorem}

Theorem~\ref{thm:sptOSO} is a corollary of Lemmas~\ref{lemma:sptO} and \ref{lemma:sptSO}.

\begin{lemma}
\label{lemma:sptO}
The single-piece decomposition preserves the termination of the \Ob-chase.
\end{lemma}
\begin{proof}
The lemma follows from (1) and (4):
\begin{enumerate}
\item Consider a rule set $\setR$, a factbase \FB, and a fair \Ob-derivation $\der = (\emptyset, \FB_0), (t_1, \FB_1), \ldots$ from $\Tuple{\spR, \FB}$.
\item We show via induction that, for every $0 \leq i \leq \funlen{\der}$, there is an injective homomorphism $h_i$ such that $h_i(\FB_i) \subseteq \funCh{\Ob}{\Tuple{\setR, \FB}}$.
Regarding the base case, let $h_0$ be the identity \EI{\Terms}{\FB} and note that $\FB_0 = \FB$.
Regarding the induction step, consider some $1 \leq i \leq \funlen{\der}$:
\begin{enumerate}
\item By induction hypothesis, there is an injective homomorphism $h_{i-1}$ with $h_{i-1}(\FB_{i-1}) \subseteq \funCh{\Ob}{\Tuple{\setR, \FB}}$.
\item Consider the rule $R = B \to \exists \vec{z} . H$ and the substitution $\pi$ in the trigger $t_i$.
Moreover, consider some rule $R' = B \to \exists \vec{y} . H' \in \setR$ with $R \in \spt(R')$ and the trigger $t_i' = (R', h_{i-1} \circ \pi)$.
\item For every term $u \in \EI{\Terms}{F_{i-1}}$, let $h_i(t) = h_{i-1}(t)$.
For every term $u \in \EI{\Terms}{F_i} \setminus \EI{\Terms}{F_{i-1}}$, which is of the form $z_{t_i}$ for some variable $z \in \vec{z}$, let $h_i(z_{t_i}) = z_{t_i'}$.
\item Since $t_i$ is \Ob-applicable on $\FB_{i-1}$, we have that $\funsup{t_i} \subseteq \FB_{i-1}$.
By (a), $\funsup{t_i'} \subseteq \funCh{\Ob}{\Tuple{\setR, \FB}}$ and hence, $\funout{t_i'} \subseteq \funCh{\Ob}{\Tuple{\setR, \FB}}$.
Therefore, $h_i(F_i) \subseteq \funCh{\Ob}{\Tuple{\setR, \FB}}$ .
\item By (a) and (c), the function $h_i$ is injective.
\end{enumerate}
\item By (2), the number of terms in any factbase $\FB_i$ is less than the number of terms in $\EI{\Terms}{\funCh{\Ob}{\Tuple{\setR, \FB}}}$.
Hence, finiteness \funCh{\Ob}{\Tuple{\setR, \FB}} implies finiteness of \der.
\item By (1) and (3), we have that $\setR \in \chaseterm{\Ob}{\forall}$ implies $\spt(\setR) \in \chaseterm{\Ob}{\forall}$.
\end{enumerate}
\end{proof}

\begin{lemma}
\label{lemma:sptSO}
The single-piece decomposition preserves the termination of the \SO-chase.
\end{lemma}
\begin{proof}
The lemma follows from (1) and (4):
\begin{enumerate}
\item Consider a rule set $\setR$, a factbase \FB, and a fair \SO-derivation $\der = (\emptyset, \FB_0), (t_1, \FB_1), \ldots$ from $\Tuple{\spR, \FB}$.
\item We show via induction that, for every $0 \leq i \leq \funlen{\der}$, there is an injective homomorphism $h_i$ such that $h_i(\FB_i) \subseteq \funCh{\SO}{\Tuple{\setR, \FB}}$.
Regarding the base case, let $h_0$ be the identity \EI{\Terms}{\FB} and note that $\FB_0 = \FB$.
Regarding the induction step, consider some $1 \leq i \leq \funlen{\der}$:
\begin{enumerate}
\item By induction hypothesis, there is an injective homomorphism $h_{i-1}$ with $h_{i-1}(\FB_{i-1}) \subseteq \funCh{\SO}{\Tuple{\setR, \FB}}$.
\item Consider the rule $R = B \to \exists \vec{z} . H$ and the substitution $\pi$ in the trigger $t_i$.
Moreover, consider some rule $R' = B \to \exists \vec{y} . H' \in \setR$ with $R \in \spt(R')$ and the trigger $t_i' = (R', h_{i-1} \circ \pi)$.
\item For every term $u \in \EI{\Terms}{F_{i-1}}$, let $h_i(t) = h_{i-1}(t)$.
\item Since $t_i$ is \SO-applicable on $\FB_{i-1}$, we have that $\funsup{t_i} \subseteq \FB_{i-1}$.
By (a), $\funsup{t_i'} \subseteq \funCh{\SO}{\Tuple{\setR, \FB}}$ and hence, $\funout{t_i''} \subseteq \funCh{\SO}{\Tuple{\setR, \FB}}$ for some trigger $t_i'' = (R,  \pi'')$ such that $\pi''$ and $h_{i-1} \circ \pi$ coincide for all the universal variables in $R$ and $\funsup{t_i''} \subseteq \funCh{\SO}{\Tuple{\setR, \FB}}$.
\item For all $u \in \EI{\Terms}{F_i} \setminus \EI{\Terms}{F_{i-1}}$, which is of the form $z_{t_i}$ for a variable $z \in \vec{z}$, let $h_i(z_{t_i}) = z_{t_i''}$.
\item By (a), (c), (d), and (e), $h_i(F_i) \subseteq \funCh{\SO}{\Tuple{\setR, \FB}}$ .
\item By (a) and (c), the function $h_i$ is injective.
\end{enumerate}
\item By (2), the number of terms in any factbase $\FB_i$ is less than the number of terms in $\EI{\Terms}{\funCh{\Ob}{\Tuple{\setR, \FB}}}$.
Hence, finiteness \funCh{\Ob}{\Tuple{\setR, \FB}} implies finiteness of \der.
\item By (1) and (3), we have that $\setR \in \chaseterm{\SO}{\forall}$ implies $\spt(\setR) \in \chaseterm{\Ob}{\forall}$.
\end{enumerate}
\end{proof}

\begin{reptheorem}{thm:sptR}
The single-piece decomposition does not preserve termination of the \R- or the \DF{\R}-chase.
\end{reptheorem}

\begin{proof}
Consider the rule set $\setR = \{\eqref{rule:p-loop-a}, \eqref{rule:p-successor}\}$ and its single-piece decomposition $\spR = \{(\ref{rule:p-successor}\text{--}\ref{rule:p-range-a})\}$:
\vspace*{-\baselineskip}
\begin{center}
\begin{minipage}{0.42\linewidth}
\begin{align}
P(x, y) &\to P(y, y) \wedge A(y) \tag{\ref{rule:p-loop-a}} \\
A(x) &\to \exists z . P(x, z) \tag{\ref{rule:p-successor}}
\end{align}
\end{minipage}~
\begin{minipage}{0.42\linewidth}
\begin{align}
P(x, y) &\to P(y, y) \tag{\ref{rule:p-loop}} \\
P(x, y) &\to A(y) \tag{\ref{rule:p-range-a}}
\end{align}
\end{minipage}
\end{center}
We verify that $\setR \in \chaseterm{\R}{\forall}$ and $\spR \notin \chaseterm{\R}{\forall}$:
\begin{itemize}
\item The set $\setR$ is in \chaseterm{\R}{\forall} because every \R-derivation \der from a KB of the form $\Tuple{\setR, \FB}$ is finite.
More precisely, we have that $\funres{\der} \subseteq \{P(t, z_t), P(z_t, z_t), A(z_t), P(t, t), A(t) \mid t \in \EI{\Terms}{\FB}\} \cup \FB$.
This is because no trigger of the form $(\eqref{rule:p-successor}, \pi)$ with $\pi(x) \in \Variables$ is a \R-applicable to a factbase in \der.
\item The set \spR is not in \chaseterm{\R}{\forall} because the KB $\Tuple{\spR, \{A(a)\}}$ admits a infinite fair \R-derivation $(\emptyset, F_0), (t_1, F_1), \ldots$ such that:
\begin{align*}
\FB_0 = \{&A(a)\},				& \FB_1 = \{&P(a, z_1)\} \cup F_0,	& \FB_2 = \{&A(z_1)\} \cup F_1, \\
\FB_3 = \{&P(z_1, z_2)\} \cup F_2,	& \FB_4 = \{&P(z_1, z_1)\} \cup F_3,	& \FB_5 = \{&A(z_2)\} \cup F_4, \\
\FB_6 = \{&P(z_2, z_3)\} \cup F_5,	& \FB_7 = \{&P(z_2, z_2)\} \cup F_6,	& \FB_8 = \{&A(z_3)\} \cup F_7,\ldots
\end{align*}
The triggers $t_1$, $t_3$, and $t_6$ feature rule \eqref{rule:p-successor}; triggers $t_2$, $t_5$, and $t_8$ feature rule \eqref{rule:p-range-a}; and triggers $t_4$ and $t_7$ feature rule \eqref{rule:p-loop}.
\end{itemize}

Consider the rule set $\setR = \{\eqref{rule:dummy-ext-1}, \eqref{rule:dummy-ext-2}\}$ and its single-piece decomposition $\spR = \{(\ref{rule:dummy-ext-2}\text{--}\ref{rule:dummy-ext-1-2})\}$:
\vspace*{-\baselineskip}
\begin{center}
\begin{minipage}{0.46\linewidth}
\begin{align}
P(x, y, v) &\to \exists u, w . P(y, y, u) \wedge A(y, w) \tag{\ref{rule:dummy-ext-1}}\\
A(x, u) &\to \exists z, v . P(x, z, v) \tag{\ref{rule:dummy-ext-2}} 
\end{align}
\end{minipage}~
\begin{minipage}{0.38\linewidth}
\begin{align}
P(x, y, v) &\to \exists u . P(y, y, u) \label{rule:dummy-ext-1-1} \\
P(x, y, v) &\to \exists w . A(y, w) \label{rule:dummy-ext-1-2}
\end{align}
\end{minipage}
\end{center}
We verify that $\setR \in \chaseterm{\R}{\forall}$ and $\spR \notin \chaseterm{\R}{\forall}$:
\begin{itemize}
\item The set $\setR$ is in \chaseterm{\R}{\forall} because every \DF{\R}-derivation \der from a KB of the form $\Tuple{\setR, \FB}$ is finite.
More precisely, we have that $\funres{\der}$ is a subset of:
$$\{P(a, a, u_a), A(a, w_a), P(a, z_a, v_a), P(z_a, z_a, u_a'), P(z_a, w_a') \mid a \in \EI{\Terms}{\FB}\} \cup \FB$$
This is because no trigger of the form $(\eqref{rule:dummy-ext-2}, \pi)$ with $\pi(x) \in \Variables$ is a \R-applicable to a factbase in \der.
\item The set \spR is not in \chaseterm{\R}{\forall} because the KB $\Tuple{\spR, \{A(a, b)\}}$ admits a infinite fair \R-derivation $(\emptyset, F_0), (t_1, F_1), \ldots$ such that:
\begin{align*}
\FB_0 = \{&A(a, b)\},					& \FB_1 = \{&P(a, z_1, v_1)\} \cup F_0,	& \FB_2 = \{&A(z_1, w_1)\} \cup F_1, \\
\FB_3 = \{&P(z_1, z_2, v_2)\} \cup F_2,	& \FB_4 = \{&P(z_1, z_1, u_1)\} \cup F_3,	& \FB_5 = \{&A(z_2, w_2)\} \cup F_4, \\
\FB_6 = \{&P(z_2, z_3, v_3)\} \cup F_5,	& \FB_7 = \{&P(z_2, z_2, u_2)\} \cup F_6,	& \FB_8 = \{&A(z_3, w_3)\} \cup F_7, \ldots
\end{align*}
The triggers $t_1$, $t_3$, and $t_6$ feature rule \eqref{rule:dummy-ext-2}; triggers $t_2$, $t_5$, and $t_8$ feature rule \eqref{rule:dummy-ext-1-2}; and triggers $t_4$ and $t_7$ feature rule \eqref{rule:dummy-ext-1-1}.
\end{itemize}
\end{proof}

\begin{reptheorem}{thm:sptExR}
The piece decomposition does not preserve the sometimes-termination of the \R-chase.
\end{reptheorem}
\begin{proof}
The rule set $\setR = \{(\ref{rule:u-init-sp}\text{--}\ref{rule:generate-new-a})\}$ is in \chaseterm{\R}{\exists} and its piece decomposition $\spR = \{(\ref{rule:u-to-r}\text{--}\ref{rule:u-init-sp-2})\}$ is not.
\begin{align}
A(x) \to \exists y, z . U(x, y) \wedge H(&y, x) \wedge U(x, z) \wedge H(z, x) \tag{\ref{rule:u-init-sp}} \\
U(x, y) \wedge U(x, z) &\to R(y, z) \tag{\ref{rule:u-to-r}} \\
U(x, z) \wedge R(y, z) &\to \exists u . R(z, u) \tag{\ref{rule:u-extends-r}} \\
R(x, y) \wedge R(y, z) &\to \exists v . S(z, v) \tag{\ref{rule:r-to-s-sp}} \\
R(x, y) \wedge S(y, z) &\to S(x, x) \tag{\ref{rule:s-loop-sp-1}} \\
 A(x) \wedge U(x, y) \wedge S(y, z) &\to \exists w . H(z, w) \wedge A(w) \tag{\ref{rule:generate-new-a}} \\
A(x) \to \exists y . U(x, y) &\wedge H(y, x) \tag{\ref{rule:u-init-sp-1}} \\
A(x) \to \exists z . U(x, z) &\wedge H(z, x) \tag{\ref{rule:u-init-sp-2}}
\end{align}
This rule set is an adaptation of $\{(\ref{rule:r-loop}\text{--}\ref{rule:s-succ})\}$, which is used to show that $\chaseterm{\DF{\R}}{\exists} \not \NSubset \chaseterm{\R}{\exists}$.

\begin{figure*}
\tikzset{every loop/.style={min distance=60,in=130,out=50,looseness=20}}

\begin{center}
\begin{tikzpicture}[roundnode/.style={circle, fill=black, inner sep=0pt, minimum size=1mm},
labelnode/.style={fill=white,inner sep=1pt}]

\newcommand{\xsep}{2.2}
\newcommand{\xa}{0*\xsep}
\newcommand{\xb}{1.25*\xsep}
\newcommand{\xc}{2*\xsep}
\newcommand{\xd}{2.75*\xsep}
\newcommand{\xe}{4*\xsep}
\newcommand{\xf}{4.75*\xsep}
\newcommand{\xg}{5.5*\xsep}
\newcommand{\xh}{6.75*\xsep}

\newcommand{\ya}{0}

\node[fill,circle,inner sep=0pt,minimum size=5pt, label=-3:{$a : A$}] at (\xa, \ya) (1) {};
\node[fill,circle,inner sep=0pt,minimum size=5pt, label=-3:{}] at (\xb, \ya) (2) {};
\node[fill,circle,inner sep=0pt,minimum size=5pt, label=-3:{}] at (\xc, \ya) (3) {};
\node[fill,circle,inner sep=0pt,minimum size=5pt, label=-3:{$A$}] at (\xd, \ya) (4) {};
\node[fill,circle,inner sep=0pt,minimum size=5pt, label=-3:{}] at (\xe, \ya) (5) {};
\node[fill,circle,inner sep=0pt,minimum size=5pt, label=-3:{}] at (\xf, \ya) (6) {};
\node[fill,circle,inner sep=0pt,minimum size=5pt, label=-3:{$A$}] at (\xg, \ya) (7) {};
\node[fill,circle,inner sep=0pt,minimum size=5pt, label=-3:{}] at (\xh, \ya) (8) {};

\path[->] (1) edge node[labelnode,pos=0.7] {$U$} (2);
\path[->] (2) edge[bend right=25] node[labelnode] {$H$} (1);
\path[->] (2) edge node[labelnode,pos=0.7] {$S$} (3);
\path[->] (3) edge node[labelnode,pos=0.7] {$H$} (4);
\path[->] (4) edge node[labelnode,pos=0.7] {$U$} (5);
\path[->] (5) edge[bend right=25] node[labelnode] {$H$} (4);
\path[->] (5) edge node[labelnode,pos=0.7] {$S$} (6);
\path[->] (6) edge node[labelnode,pos=0.7] {$H$} (7);
\path[->] (7) edge[dotted] node[labelnode,pos=0.7] {$U$} (8);
\path[->] (8) edge[bend right=20, dotted] node[labelnode] {$H$} (7);

\path[->] (2) edge [loop above, out=135, in=45, looseness=25] node[labelnode,pos=0.9] {$R, S$} (2);
\path[->] (5) edge [loop above, out=135, in=45, looseness=25] node[labelnode,pos=0.9] {$R, S$} (5);
\path[->] (8) edge [loop above, dotted, out=135, in=45, looseness=25] node[labelnode,pos=0.9] {$R, S$} (8);

\end{tikzpicture}
\end{center}

\caption{\Large The only result of the \R-chase from the KB $\K$ introduced in the proof of Theorem~\ref{thm:sptExR}}
\label{figure:single-piece-infinite-restricted}
\end{figure*}
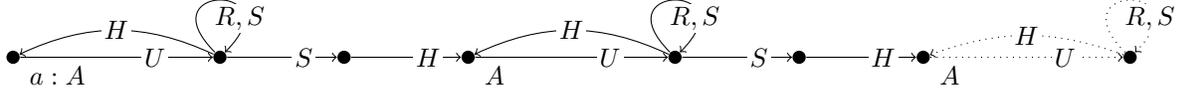

To show that $\spt(\setR) \notin \chaseterm{\R}{\exists}$ one can verify that the KB $\KB = \Tuple{\setR, \{A(a)\}}$ does not admit terminating \R-derivations.
More precisely, one can easily verify that that all fair \R-derivations from \KB yield the same infinite result, which is depicted Figure~\ref{figure:single-piece-infinite-restricted}.

To show that $\setR \in \chaseterm{\R}{\exists}$ we verify that every KB $\KB$ of the form $\Tuple{\setR, \FB}$ admits a terminating \R-derivation:
\begin{enumerate}
\item Consider an \R-derivation \der from \KB that is obtained from the following rule application strategy:\footnote{To better understand this rule application strategy, apply it to the KB \Tuple{\setR, \{A(a)\}} and see what facts it produces.}
\begin{enumerate}
\item Apply all rules except for \eqref{rule:u-init-sp} exhaustively.
\item Apply \eqref{rule:u-init-sp} exhaustively.
\item Consider some (arbitrarily chosen) strict partial order $\prec$ over the set of all terms \Terms.
Then, apply all triggers of the form $\Tuple{\eqref{rule:u-to-r}, \pi}$ with $\pi(y) \prec \pi(z)$.
\item Apply \eqref{rule:u-extends-r} and \eqref{rule:r-to-s-sp} exhaustively.
\item Apply \eqref{rule:u-to-r} and \eqref{rule:s-loop-sp-1} exhaustively.
\end{enumerate}
\item To show that \der is finite, we show that there is an injective homomorphism from $\funres{\der}$ to (finite) factbase.
\begin{itemize}
\item There is an injective homomorphism from the factbase that results after applying Step~(1.a) to:
$$\FB' = \{S(t, t), R(t, t), R(t, u_t), S(u_t, v_t), S(t, v_t'), H(t, w_t), A(w_t), R(t, s) \mid t, s \in \EI{\Terms}{\FB}\} \cup \FB$$
In the above, $u_t$, $v_t$, $v_t'$, and $w_t$ are fresh variables unique for the term $t$.
\item Moreover, there is an injective homomorphism from \funres{\der} to
\begin{align*}
\{&U(t, y_t), H(y_t, t), U(t, z_t), H(z_t, t), R(y_t, z_t), R(z_t, u_t), \\
&S(u_t, v_t), R(z_t, y_t), S(z_t, z_t), S(y_t, y_t) \mid t \in \EI{\Terms}{\FB'}\} \cup \FB'
\end{align*}
In the above, $y_t$, $z_t$, and $u_t$ are fresh variables unique for the term $t$.
\end{itemize}
\item By (2), there is a factbase $\FB_x$ in \der for every $x \in \{a, b, c, d, e\}$ that results from the sequential application Steps~(1.a) to (1.x) to \FB as described in (1).
\item To show that \der is fair we check that no trigger with a rule from \setR is \R-applicable to the last element of \der.
That is, we verify that no such trigger is applicable to $\FB_e$.
\begin{itemize}
\item No trigger with \eqref{rule:u-init-sp} is \R-applicable to $\FB_b$, $\FB_c$, $\FB_d$, or $\FB_e$ since \FirstItem $\FB_b$ is constructed by exhaustively applying \eqref{rule:u-init-sp} and \SecondItem no new facts over the predicate $A$  are introduced in either of these sets.
That is, for every $t \in \EI{\Terms}{\FB_e}$ and every $x \in \{b, c, d, e\}$, we have that $A(t) \in \FB_x$ implies $A(t) \in \FB_a$.
That is, for every $x \in \{b, c, d, e\}$, the set of all facts in $\FB_x$ defined over the predicate $A$ is a subset of $\FB_a$.
\item No trigger with either \eqref{rule:u-to-r} or \eqref{rule:s-loop-sp-1} is \R-applicable to $\FB_e$ since this factbase is constructed by exhaustively applying these two rules.
\item No trigger with either \eqref{rule:u-extends-r} or \eqref{rule:r-to-s-sp} is \R-applicable to $\FB_d$ because this factbase is constructed by exhaustively applying these rules.
Moreover, no trigger with either of these rules is \R-applicable to $\FB_e$ because $R(u, t), R(u, u) \in \FB_e$ for every fact of the form $R(t, u) \in \FB_e \setminus \FB_d$.
\item No trigger with \eqref{rule:generate-new-a} is \R-applicable to $\FB_b$, $\FB_c$, $\FB_d$, and $\FB_e$ since \FirstItem $\FB_b$ is constructed by exhaustively applying this rule and, \SecondItem if $U(t, u), S(u, v) \in \FB_x \setminus \FB_b$ for some $t, u, v \in \Terms$ and $x \in \{c, d, e\}$, then $u = v$ and $H(u, t), A(t) \in \FB_x$.
\end{itemize}
\item By (1), (2), and (4); the sequence \der is a terminating \R-derivation from \Tuple{\setR, \FB}.
\end{enumerate}
\end{proof}

\begin{reptheorem}{thm:sptExDfR}
The single-piece decomposition does not preserve sometimes-termination of the \DF{\R}-chase.
\end{reptheorem}

\begin{proof}
The rule set $\setR = \{(\ref{rule:r-loop-sp}\text{--}\ref{rule:s-succ-sp})\}$ is in \chaseterm{\DF{\R}}{\exists} and its piece decomposition $\spR = \{(\ref{rule:s-loop-sp}\text{--}\ref{rule:r-loop-sp-2})\}$ is not.
\begin{align}
A(x) &\to \exists y . R(x, x) \wedge H(x, y) \tag{\ref{rule:r-loop-sp}} \\
R(x, y) \wedge S(y, z) &\to S(x, x) \tag{\ref{rule:s-loop-sp}} \\
A(x) \wedge S(x, y) &\to A(y) \tag{\ref{rule:a-prop-sp}} \\
A(x) &\to \exists z . R(x, z) \tag{\ref{rule:r-succ-sp}} \\
R(x, y) &\to \exists w . S(y, w) \tag{\ref{rule:s-succ-sp}} \\
A(x) &\to R(x, x) \tag{\ref{rule:r-loop-sp-1}} \\
A(x) &\to \exists y . H(x, y) \tag{\ref{rule:r-loop-sp-2}}
\end{align}
Note that the only difference with respect to $\{(\ref{rule:r-loop}\text{--}\ref{rule:s-succ})\}$, which is used to show that $\chaseterm{\DF{\R}}{\exists} \not \NSubset \chaseterm{\R}{\exists}$, is the atom $H(x, y)$ in the first rule that makes this rule non Datalog, preventing its early application.

\begin{figure}
\begin{center}
\begin{tikzpicture}[roundnode/.style={circle, fill=black, inner sep=0pt, minimum size=1mm},
labelnode/.style={fill=white,inner sep=1pt}]

\newcommand{\xsep}{4}
\newcommand{\xa}{0*\xsep}
\newcommand{\xap}{0.35*\xsep}
\newcommand{\xb}{1*\xsep}
\newcommand{\xbp}{1.35*\xsep}
\newcommand{\xc}{2*\xsep}
\newcommand{\xcp}{2.35*\xsep}
\newcommand{\xd}{3*\xsep}
\newcommand{\xdp}{3.35*\xsep}
\newcommand{\xe}{4*\xsep}

\newcommand{\ya}{0}
\newcommand{\yb}{0.85}

\node[fill,circle,inner sep=0pt,minimum size=5pt, label=-3:{$a : A$}] at (\xa, \ya) (1) {};
\node[fill,circle,inner sep=0pt,minimum size=5pt, label=-3:{}] at (\xap, \yb) (1p) {};
\node[fill,circle,inner sep=0pt,minimum size=5pt, label=-3:{$A$}] at (\xb, \ya) (2) {};
\node[fill,circle,inner sep=0pt,minimum size=5pt, label=-3:{}] at (\xbp, \yb) (2p) {};
\node[fill,circle,inner sep=0pt,minimum size=5pt, label=-3:{$A$}] at (\xc, \ya) (3) {};
\node[fill,circle,inner sep=0pt,minimum size=5pt, label=-3:{}] at (\xcp, \yb) (3p) {};
\node[fill,circle,inner sep=0pt,minimum size=5pt, label=-3:{$A$}] at (\xd, \ya) (4) {};
\node[fill,circle,inner sep=0pt,minimum size=5pt, label=-3:{}] at (\xdp, \yb) (4p) {};
\node[circle,inner sep=0pt,minimum size=5pt, label=-3:{}] at (\xe, \ya) (5) {};

\path[->] (1) edge [loop above, in=90, out=155, looseness=40] node[labelnode,pos=0.2] {$R, S$} (1);
\path[->] (2) edge [loop above, in=90, out=155, looseness=40] node[labelnode,pos=0.2] {$R, S$} (2);
\path[->] (3) edge [loop above, in=90, out=155, looseness=40] node[labelnode,pos=0.2] {$R, S$} (3);
\path[->] (4) edge [dotted, loop above, in=90, out=155, looseness=40] node[labelnode,pos=0.2] {$R, S$} (4);

\path[->] (1) edge node[labelnode] {$H$} (1p);
\path[->] (2) edge node[labelnode] {$H$} (2p);
\path[->] (3) edge node[labelnode] {$H$} (3p);
\path[->] (4) edge[dotted] node[labelnode] {$H$} (4p);

\path[->] (1) edge node[labelnode,pos=0.7] {$S$} (2);
\path[->] (2) edge node[labelnode,pos=0.7] {$S$} (3);
\path[->] (3) edge[dotted] node[labelnode,pos=0.7] {$S$} (4);

\end{tikzpicture}
\end{center}

\caption{\Large The only result of the \DF{\R}-chase from the KB $\K$ introduced in the proof of Theorem~\ref{thm:sptExDfR}}
\label{figure:single-piece-infinite-df-restricted}
\end{figure}

To show that $\spt(\setR) \notin \chaseterm{\R}{\exists}$ one can verify that the KB $\KB = \Tuple{\setR, \{A(a)\}}$ does not admit terminating \R-derivations.
More precisely, one can easily verify that that all fair \R-derivations from \KB yield the same infinite result, which is depicted Figure~\ref{figure:single-piece-infinite-df-restricted}.

To show that $\setR \in \chaseterm{\R}{\exists}$ we verify that every KB $\KB$ of the form $\Tuple{\setR, \FB}$ admits a finite and fair \R-derivation.
Such a derivation can be obtained by applying the following rule application strategy: ; then 

To show that $\setR \in \chaseterm{\R}{\exists}$ we verify that every KB $\KB$ of the form $\Tuple{\setR, \FB}$ admits a terminating \R-derivation:
\begin{enumerate}
\item Consider an \R-derivation \der from \KB that is obtained from the following rule application strategy on \FB:
\begin{enumerate}
\item Apply (\ref{rule:s-loop-sp}\text{--}\ref{rule:s-succ-sp}) exhaustively (in any order that prioritises Datalog rules).
\item Apply \eqref{rule:r-loop-sp} exhaustively
\end{enumerate}
\item The sequence \der is finite, since there is an injective homomorphism from $\funres{\der}$ to:
$$\{S(t, t), R(t, t), A(t), R(t, z_t), S(z_t, w_t), S(t, w_t'), H(t, y_t) \mid t \in \EI{\Terms}{\FB}\}$$
\item By (2), there is a factbase $\FB_a$ that is obtained from applying Step~(1.a) to the factbase \FB.
\item To show that \der is fair we check that no trigger with a rule from \setR is \R-applicable to the last element of \der.
That is, we verify that no such trigger is applicable to $\funres{\der}$; see (a), (d), (e), and (f).
\begin{enumerate}
\item No trigger with \eqref{rule:r-loop-sp} is applicable to $\funres{\der}$ because this factbase is obtained from the exhaustive application of this rule.
\item If $R(t, u) \in \funres{\der} \setminus \FB_a$ for some $t, u \in \Terms$, then $t = u$ and $A(t) \in \FB_a$ since $R(t, u)$ is produced by the application of \eqref{rule:r-loop-sp}.
\item If $A(t) \in \FB_a$ for some $t \in \Terms$, then $S(t, t) \in \FB_a$ since $\FB_a$ is obtained from the exhaustive application of \eqref{rule:r-succ-sp}, \eqref{rule:s-succ-sp}, and \eqref{rule:s-loop-sp}.
\item No trigger with \eqref{rule:s-loop-sp} is applicable to $\FB_a$ because this factbase is obtained from the exhaustive application of this rule.
No trigger with this rule is applicable to \funres{\der} because, for every $R(t, u) \in \funres{\der} \setminus \FB_a$,  we have that $S(t, t) \in \FB_a$ by (2.b) and (2.c).
\item No trigger with \eqref{rule:a-prop-sp} or \eqref{rule:r-succ-sp} is applicable to $\FB_a$ because this factbase is obtained from the exhaustive application of these rules.
Moreover, no trigger with these rules is applicable to \funres{\der} because the set of all facts in \funres{\der} defined over the predicate $A$ or $S$ is a subset of $\FB_a$.
\item No trigger with \eqref{rule:s-succ-sp} is applicable to $\FB_a$ because this factbase is obtained from the exhaustive application of this rule.
No trigger with this rule is applicable to \funres{\der} because, for every $R(t, u) \in \funres{\der} \setminus \FB_a$,  we have that $t = u$ and $S(t, t) \in \FB_a$ by (2.b) and (2.c).

\end{enumerate}
\item By (1), (2), and (4); the sequence \der is a terminating \R-derivation from \Tuple{\setR, \FB}.
\end{enumerate}

\end{proof}


\begin{reptheorem}{thm:sptE}
The single-piece decomposition preserves the termination of the $\E$-chase.
\end{reptheorem}

\begin{proof}
Given some KBs $\K = \Tuple{\setR, \FB}$ and $\K' = \Tuple{\setR', \FB}$ where $\setR$ and $\setR'$ are equivalent rule sets, a factbase $M$ is a universal model for $\K$ iff it is a universal model for $\K'$.
Therefore, a KB $\Tuple{\setR, \FB}$ admits a finite universal model iff \Tuple{\spR, \FB} also admits one by Proposition~\ref{prop:spteq}.
Since a rule set \setR is in \chaseterm{\E}{\forall} iff every KB of the form $\Tuple{\setR, \FB}$ admit a finite universal model, the theorem holds.
\end{proof}

The single-piece decomposition does improve the termination of some chase variants:

\begin{reptheorem}{thm:sp-improves}
The single-piece decomposition improves the termination and sometimes-termination of the \SO-, the \R-, and the \DF{\R}-chase; it does not improve the termination of the \Ob- and \E-chase.
\end{reptheorem}
\begin{proof}
The theorem follows from (1), (2), and (3):
\begin{enumerate}
\item We show that the rule set $\setR = \{P(x, y) \to \exists z . P(x, z) \wedge R(x, y)\}$ is not in \chaseterm{\SO}{\forall}, \chaseterm{\R}{\exists}, \chaseterm{\R}{\forall}, \chaseterm{\DF{\R}}{\exists}, or \chaseterm{\DF{\R}}{\forall} and that $\spR$ is in all of these sets.
\begin{itemize}
\item To show that \setR is not in $\chaseterm{\X}{\forall} \cup \chaseterm{\X}{\forall}$ for any $\X \in \{\Ob, \R, \DF{\R}\}$, it suffices to observe that the KB $\Tuple{\setR, P(a, b)}$ admits exactly one \X-derivation, which is infinite.
\item The rule set \spR is in $\chaseterm{\X}{\forall} \cup \chaseterm{\X}{\forall}$ for any $\X \in \{\Ob, \R, \DF{\R}\}$ because all \X-derivations from a KB of the form $\Tuple{\spR, \FB}$ are finite.
More precisely, given some \X-derivation \der from \Tuple{\spR, \FB}, there is an injective homomorphism from $\funres{\der}$ to $\{R(t, u), P(t, z_t), R(t, z_t) \mid t, u \in \EI{\Terms}{\FB}\} \cup \FB$.
\end{itemize}
\item To show that the \E-chase does not improve termination, we can use an argument analogous to the proof of Theorem~\ref{thm:sptE}.
\item To show that the \Ob-chase does not improve terminating, we can show via induction that, for every KB $\Tuple{\setR, \FB}$, the sets $\funCh{\Ob}{\Tuple{\setR, \FB}}$ and $\funCh{\Ob}{\Tuple{\spt(\setR), \FB}}$ are isomorphic (where \funCh{\Ob}{\cdot} is the function from Definition~\ref{def:ch}, which maps a KB to its only \Ob-chase result).
We can do so with an argument analogous to the proof of Lemma~\ref{lemma:sptO}.
\end{enumerate}
\end{proof}


\section{Proofs of Section~\ref{sec:1ad}}

Next, we denote by $\Sigma$ the set of \emph{original predicates}, i.e., those occuring in $\setR$, or in $\KB = \Tuple{\setR, \FB}$ when a KB is considered. Furthermore, given a rule set $\setR$, we denote by $\setR_X$ the subset of $\oadR$ containing all the rules that introduce a fresh predicate, i.e., 
$$\setR_X=\ens{B\to X_R[\vec{y}]\mid R=B\to H\in\setR, X_R \not \in \Sigma}$$

\subsection{Proof of Theorem~\ref{thm:1adBDDP}}

\begin{proposition}
\label{prop:RTo1ad}
For a factbase $\FB$, a rule set $\setR$, and some $i \geq 1$; we have $\funCh{i}{\setR, \FB} \subseteq \funCh{2i}{\oadR, \FB}$.
\end{proposition}

\begin{proof}
We show the result by induction on $i \geq 1$.
Regarding the base case, note that $\funCh{0}{\setR, \FB} = \funCh{0}{\oadR, \FB} = \FB$.
Regarding the induction step, we assume that the proposition holds for some $i \geq 1$ and show that it holds for $i+1$.
\begin{enumerate}
\item By induction hypothesis, we have that $\funCh{i}{\setR, \FB} \subseteq \funCh{2i}{\oadR, \FB}$.
\item By definition, $\funCh{i+1}{\setR, \FB}$ is the minimal set that includes $\funCh{i}{\FB,\setR}$ and \funout{t} for every trigger $t = (R, \pi)$ such that $R \in \setR$ and $\funsup{t} \subseteq \funCh{i}{\setR, \FB}$.
\item Consider a trigger $t = (R, \pi)$ where $R = B[\vec{x}, \vec{z}] \to \exists \vec{y} . (H_1 \wedge \ldots \wedge H_n)[\vec{x}, \vec{y}]$ is a rule in $\setR$ and $\funsup{t} \subseteq \funCh{i}{\setR, \FB}$.
Moreover, consider the rule set $\oad(R) = \{B \to \exists \vec{y} . X_R(\vec{x}, \vec{y}), X_R(\vec{x}, \vec{y}) \to H_1, \ldots, X_R(\vec{x}, \vec{y}) \to H_n\}$, which is a subset of \oadR.
\item $\funout{B \to \exists \vec{y} . X_R(\vec{x}, \vec{y}), \pi} \subseteq \funCh{2i + 1}{\oadR, \FB}$ by (1) and (3).
\item $\funout{\Rule, \pi} \subseteq \funCh{2(i + 1)}{\oadR, \FB}$ for every $R \in \oad(R)$ by (1), (3), and (4).
\item The induction step holds by (2), (3), and (5).
\end{enumerate}
\end{proof}

\begin{proposition}
\label{prop:1adToR}
For a factbase $\FB$ and a rule set $\setR$, for some $i \geq 1$; we have that
$$\funCh{2i}{\FB,\oadR} \subseteq \funCh{1}{\funCh{i}{\funCh{1}{\FB,\setR_X}_{\mid\Sigma},\setR},\setR_X} \cup (\FB \setminus \FB_{\mid \Sigma})$$
\end{proposition}

\begin{proof}
We show the result by induction on $i$. 

For $i=0$, $\funCh{0}{\FB,\oadR} = \FB \subseteq \funCh{1}{\funCh{1}{\FB,\setR_X}_{\mid\Sigma},\setR_X}) \cup (\FB \setminus \FB_{\mid \Sigma})$, as $\FB_{\mid \Sigma}$ is included in the first part of the union.

Let us assume the result to be true for $i$, and let us show it for $i+1$. Let us notice that an atom of $\funCh{2i+2}{\FB,\oadR} \setminus \funCh{2i}{\FB,\oadR}$ can be generated in the following way:
\begin{enumerate}
 \item\label{item-rx-2i} by applying a rule of $\setR_X$ on $\funCh{2i}{\FB,\oadR}$;
 \item\label{item-rset-2i} by applying a rule of $\oadR \setminus \setR_X$ on $\funCh{2i}{\FB,\oadR}$;
 \item\label{item-rx-2i+1} by applying a rule of $\setR_X$ on $\funCh{2i+1}{\FB,\oadR}$;
 \item\label{item-rset-2i+1} by applying a rule of $\oadR \setminus \setR_X$ on $\funCh{2i+1}{\FB,\oadR}$.
\end{enumerate}

Any atom generated by Case\ \ref{item-rx-2i}. belongs to 
$\funCh{1}{\funCh{i+1}{\funCh{1}{\FB,\setR_X}_{\mid\Sigma},\setR},\setR_X} \cup (\FB \setminus \FB_{\mid \Sigma})$. Indeed, the corresponding rule should match its body to atoms having a predicate in $\Sigma$, hence belonging to $\funCh{i}{\funCh{1}{\FB,\setR_X}_{\mid \Sigma},\setR}$, hence the rule is applied when computing $\mathsf{chase}_1$ w.r.t. $\setR_X$ of that set.

Atoms generated by Case\ \ref{item-rset-2i} belong to $\funCh{i+1}{\funCh{1}{\FB,\setR_X}_{\mid\Sigma},\setR}$. Indeed, let $(R,\pi)$ generating such an atom. $(R,\pi)$ must be of the shape $X_{R'}(\vec{y}) \rightarrow \alpha$. Let us consider $\pi'$ such that $(R_a',\pi')$ has been applied to generate $\pi(X_{R'}(\vec{y}))$. $(R',\pi')$ is also applicable on $\funCh{i}{\funCh{1}{\FB,\setR_X}_{\mid\Sigma},\setR}$, and the result of that rule application belongs to $\funCh{i+1}{\funCh{1}{\FB,\setR_X}_{\mid\Sigma},\setR}$.

Case\ \ref{item-rset-2i+1} can be treated in the same way, while Case\ \ref{item-rx-2i+1} is treated by noticing that such atoms must have been generated by an application of a rule of $\setR_X$, mapping its body to atoms that belong (by previous cases) to $\funCh{i+1}{\funCh{1}{\FB,\setR_X}_{\mid\Sigma},\setR}$. Hence the generated atoms belong to $\funCh{1}{\funCh{i+1}{\funCh{1}{\FB,\setR_X}_{\mid\Sigma},\setR},\setR_X}$
\end{proof}

\begin{proposition}\label{prop:1adBDDPRTo1ad}
 For any rule set $\setR$ having the bounded derivation depth property, $\oadR$ has the bounded derivation depth property.
\end{proposition}

\begin{proof}
 Let $q$ be such that $\FB, \oadR \models q$. There must exists $i$ such that 
  \[\funCh{2i}{\FB,\oadR} \models q.\]
  
By Proposition\ \ref{prop:1adToR}, it holds that 

\[\funCh{1}{\funCh{i}{\funCh{1}{\FB,\oadR}_{\mid\Sigma},\setR},\setR_X} \cup (\FB \setminus \FB_{\mid \Sigma}) \models q\]

Let $\pi$ be a match of $q$ witnessing this entailment. Let $\hat{q}$ be the subset of $q$ containing only the atoms that are not mapped to $\FB \setminus \FB_{\mid\Sigma}$, having as answer variables the terms appearing both in $\hat{q}$ and $q \setminus \hat{q}$. $\pi$ is a match of $\hat{q}$ in $\funCh{1}{\funCh{i}{\funCh{1}{\FB,\oadR}_{\mid\Sigma},\setR},\setR_X}$. Let $\sigma_{\hat{q}}$ be equal to $\pi$ on answer variables of $\hat{q}$ and the identity on other variables. By definition of rewritings, and as $\setR_X$ is acyclic, there exists a finite set of conjunctive queries $\hat{\mathfrak{Q}}$ and $q'\in \hat{\mathfrak{Q}}$ s.t.:

\[\funCh{i}{\funCh{1}{\FB,\oadR}_{\mid\Sigma},\setR} \models \sigma_{\hat{q}}(q').\]

As $\funCh{1}{\FB,\oadR}_{\mid\Sigma}$ is on $\Sigma$, $\funCh{i}{\funCh{1}{\FB,\oadR}_{\mid\Sigma},\setR}$ must also be on $\Sigma$, and hence $q'$ as well. Hence, there is $k_{q}$, independent of $\FB$, such that 
\[\funCh{k_{q}}{\funCh{1}{\FB,\oadR}_{\mid\Sigma},\setR} \models \sigma_{\hat{q}}(q').\]

By Proposition\ \ref{prop:RTo1ad}, it holds that 
\[\funCh{2k_{q}}{\funCh{1}{\FB,\oadR}_{\mid\Sigma},1ad(\setR}) \models \sigma_{\hat{q}}(q').\]

As $\funCh{1}{\FB,\oadR}_{\mid \Sigma} \subseteq \funCh{1}{\FB,\oadR}$, we conclude that:
\[\funCh{2k_q+1}{\FB,\oadR} \models \sigma_{\hat{q}}(q')\]

Hence $\funCh{2k_q+1}{\FB,\oadR} \models q$, as $\FB \setminus \FB_{\mid \Sigma} \subseteq \funCh{2k_q+1}{\FB,\oadR}$.
\end{proof}

\begin{proposition}\label{prop:1adBDDP1adToR}
    For any rule set $\setR$, if $\oadR$ has the bounded derivation depth property, then so do $\setR$.
\end{proposition}

\begin{proof}
 Let us assume that $\setR$ does not have the bounded derivation depth property. There exists $q$ and $\{\FB_i\}_{i \in \mathbb{N}}$ such that for all $i$,
 $\FB_i, \setR \models q$
 and
 $\funCh{i}{\FB_i,\setR} \not \models q$.
 
 As for all $\FB_i$ on $\Sigma$,
 
 \[\funCh{i}{\FB_i,\setR} = \funCh{2i}{\FB_i,1ad(\setR})_{\mid \Sigma}\]
 
 if holds that
 
 \[\funCh{2i}{\FB_i,1ad(\setR}) \not \models q,\]
 
 proving that $\oadR$ does not have the bounded derivation depth property. 
 
\end{proof}

\section{Proofs of Section~\ref{sec:2ad}}


\subsection{Proofs of Proposition~\ref{prop:adce} and Theorems~\ref{thm:2adOSO}, \ref{thm:2adExR}}

\begin{repproposition}{prop:adce}
    The rule set \tadR is a conservative extension and a universal-conservative extension of \setR.
\end{repproposition}

\begin{proof}
    Let us first notice that for a knowledge base $\langle\setR, \FB\rangle$, every model of $\langle\tadR, \FB\rangle$ is a model of $\langle\oadR, \FB\rangle$, because $\oadR\subseteq\tadR$. As such, we only need to show that \tadR is a conservative extension of \setR to have the result.

    Let $\langle\setR, \FB\rangle$ be a knowledge base. For $(1)$, let \setN be a model of $\langle\tadR,\FB\rangle$. We denote its restriction to the predicates appearing in \setR by \setM. We want to show that \setM is a model of $\langle\setR, \FB\rangle$. First, $\FB\subseteq\setM$, because $\FB$ does not contain any fresh predicate. Let $R=B\to\bigwedge_i H_i$ be a rule and $\pi$ a homomorphism from $B$ to \setM. Since \setM is a restriction of \setN, $\pi$ is a homomorphism from $B$ to \setN. Thus, since every rule in \tadR is satisfied in \setN, the rule $B\to X_R(\vec{x})$ is too, so the atom $\pi^R(X_R(\vec{x}))$ is in \setN. Thus, since for every $i$, the rule $X_R\to H_i$ is satisfied, $\pi^R(H_i)$ is also in the database. As such, the homomorphism $\pi^R$ is an extension of $\pi$ such that forall $i$, $\hat{\pi}(H_i)\in\setM$. Since no $H_i$ features a fresh predicate, they are all in \setM too. Thus, every rule in \setR is satisfied in \setM. As such, \setM is a model of $\langle\setR, \FB\rangle$.

    For $(2)$, let \setM be a model of $\langle\setR, \FB\rangle$. We extend \setM to \setN using the following method: for every rule $R=B\to H\in\setR$, for every homomorphism $\pi$ from $H$ to \setM, we add the atom $\pi(X_R(\vec{x}))$ to \setM. With this definition, \setM and \setN share the same domain and agree on the predicates in \setR. In addition, \setN contains $\FB$. Then, let $R$ be a rule. Let us show that $R$ is satisfied by case analysis on the form of $R$:
    \begin{description}
        \item[If $R=B\to X_R(\vec{x})$: ] If there is a homomorphism $\pi$ from $B$ to \setN, then it is a homomorphism from $B$ to \setM. As such, there is an extension $\hat{\pi}$ of $\pi$ such that $\hat{\pi}(H)\in\setM$ because \setM is a model. Thus, $\hat{\pi}(X_R(\vec{x}))\in\setM$, so $R$ is satisfied.
        \item[If $R=X_R(\vec{x})\to H_i$: ] If there is a homomorphism $\pi$ from $X_R(\vec{x})$ to \setN, then $\pi$ is a homomorphism from $H$ to \setM (else we would not have added $\pi(X_R(\vec{x}))$ to construct \setN), so every atom in $\pi(H)$ is in \setN, which means in particular $\pi(H_i)\in\setN$, so $R$ is satisfied.
        \item[If $R=H\to X_R(\vec{x})$: ] If there is a homomorphism $\pi$ from $H$ to \setM, we added the atom $\pi(X_R(\vec{x}))$ to construct \setN, so $R$ is satisfied.
    \end{description}
    Thus, \setN is a model of $\langle\tadR,\FB\rangle$.

    Now that we have this, we can show that the Two-way atomic decomposition indeed has the property we want.

    Let $\langle\setR,\FB\rangle$ be a knowledge base. For $(1)$, let \setV be a universal model of $\langle\tadR,\FB\rangle$. We want to show that \setU, the restriction of \setV to the predicates appearing in \setR, is a universal model of $\langle\setR, \FB\rangle$. First, since $\langle\tadR,\FB\rangle$ is a conservative extension of $\langle\setR, \FB\rangle$, \setU is a model. To show its universality, we show that it can be homomorphically embedded in any other model of $\langle\setR, \FB\rangle$. Let \setM be another model of $\langle\setR,\FB\rangle$. We can extend \setM into a model \setN of $\langle\tadR,\FB\rangle$. Then, since \setV is a universal model, there is a homomorphism $h$ from \setV to \setN. Since \setM and \setN (resp. \setU and \setV) share the same domain, $h$ is a mapping from \setU to \setM. Let $P(\vec{x})$ be an atom in \setU. As such, $P$ is a predicate in \setR, and is also in \setV. Since $h$ is a homomorphism, $h(P(\vec{x}))\in\setN$. Since our restriction only removes atoms featuring predicates not in \setR, $h(P(\vec{x}))\in\setM$, proving \setU is a universal model.

    For $(2)$, let \setU be a universal model of $\langle\setR, \FB\rangle$. Consider the extension \setV of \setU defined in the context of \tadR being a conservative extension of \setR. It is a model of $\langle\tadR,\FB\rangle$. We want to show that it is a universal model. Let \setN be a model of $\langle\tadR,\FB\rangle$ and \setM the restriction of \setN that is a model of $\langle\setR,\FB\rangle$. Since \setU is a universal model, there is a homomorphism $h$ from \setU to \setM. We show that $h$ is also a homomorphism from \setV to \setN. It is a mapping from \setV to \setN, and for any atom $P(\vec{x})$ in \setV that features no fresh predicate, $h(P(\vec{x}))\in\setN$. Let $X_R(\vec{y})$ be an atom in \setV that features a fresh predicate, with $R=B\to \bigwedge_i H_i$ the rule such that $(B\to X_R)\in\tad(R)$. Since \setV is a model, it features every $H_i(\vec{y_i})$ with $\vec{y_i}$ the restriction of $\vec{y}$ to the variables of $H_i$ (because the rules $X_R\to H_i$ are all satisfied). As such, \setN also features those atoms. Since it is a model of $\langle\tadR,\FB\rangle$, it satisfies the rule $\bigwedge_i H_i\to X_R$, so $X_R(\vec{y})\in\setN$. Thus, \setV is a universal model.
\end{proof}


\begin{reptheorem}{thm:2adOSO}
    Both atomic decompositions preserve the termination of the oblivious and the semi-oblivious chase.
\end{reptheorem}

\begin{proof}
    First note that in the oblivious and the semi-oblivious chase, firing a rule cannot prevent another one from firing. As such, since $\oadR\subseteq\tadR$, if the oblivious (resp. semi-oblivious) chase terminates on \tadR, it also terminates on \oadR. We can thus prove the result only for the two-way atomic decomposition.

    Let $\X\in\ens{\Ob, \SO}$, \setR be a rule set and $\FB$ a factbase. Let us show by induction that for a derivation $\der=(\emptyset, \FB),(t_1,\FB_1),\ldots$ from $\langle\tadR, \FB\rangle$, there is an injective homomorphism $h$ such that $h(\funres{\der}_{\mid\Sigma})\subseteq\funCh{\X}{\langle\setR, \FB\rangle}$.
    \begin{description}
        \item[Step $0$:] $\FB\subseteq\funCh{\X}{\langle\setR, \FB\rangle}$.
        \item[Step $n$:] Assume the result up to step $n-1$. Thus, there is a homomorphism $h'$ such that $h'((\FB_{n-1})_{\mid\Sigma})\subseteq\funCh{\X}{\langle\setR, \FB\rangle}$. Depending on the trigger $t_n = (R^{ad}, \pi)$, with $R = B\to \bigwedge_i H_i$, we distinguish three cases:
        \begin{description}
            \item[If $R^{ad}=B\to X_R(\vec{x})$,] $(\FB_{n-1})_{\mid\Sigma}=(\FB_n)_{\mid\Sigma}$ so we have the result.
            \item[If $R^{ad}=X_R(\vec{x})\to H_i$,]  since $t_n$ is \X-applicable on $\FB_{n-1}$, its support is in $\FB_{n-1}$. In addition, since $\FB$ does not contain any fresh predicate, there is a $k<n$ such that $t_k=(B\to X_R(\vec{x}), \varphi)$ and $\pi = (\varphi^{R^{ad}})_{|\EI{\Vars}{R}}$ (if $t_n$'s support was introduced by the backwards rule, it would not be applicable). Therefore, $h'(\funsup{t_k}_{\mid\Sigma})\subseteq\funCh{\X}{\langle\setR, \FB\rangle}$. Since the support of $t_k$ is the body of the initial rule, we have $\funsup{t_k}_{\mid\Sigma}=\funsup{t_k}$, which implies that, if we set $t=(R, h\circ\varphi)$, $\funsup{t}\subseteq\funCh{\X}{\langle\setR, \FB\rangle}$. Similarily to the proof of Theorem~\ref{thm:sptOSO}, we can show that we can extend $h'$ to a $h$ such that $h(\funout{t_n})\subseteq\funCh{\X}{\langle\setR, \FB\rangle}$ (because in the \Ob-chase, $t$ was applied, and in the \SO-chase we can find a trigger that shared $t$'s frontier that was applied). We thus have the result.
            \item[If $R^{ad}=\bigwedge_i H_i\to X_R(\vec{x})$,] as in the case of a rule of the form $B\to X_R(\vec{x})$, $(\FB_{n-1})_{\mid\Sigma}=(\FB_n)_{\mid\Sigma}$ so we have the result by induction hypothesis.
        \end{description}
    \end{description}
    We conclude with the same argument of cardinality as in Theorem~\ref{thm:sptOSO}.
\end{proof}


\begin{reptheorem}{thm:2adExR}
    The two-way atomic decomposition preserves and may gain sometimes-termination of the \R-chase.
\end{reptheorem}

We call a factbase such that no rule in $\oadR\setminus\setR$ (i.e. rules of the form $X_R\to H_i$) is applicable \oad-free.
        
\begin{proof}
    We first prove the preservation of the sometimes-termination. Let $\K=\langle\setR, \FB\rangle$ be a knowledge base on which the restricted chase sometimes terminates, and $\der=(\emptyset, \FB),(\FB_1, t_1),\ldots,(\FB_n,t_n)$ a derivation from \K. Let $\K^{ad}=\langle\tadR, \FB\rangle$. We show by induction on $n$ that there is a derivation $\der^{ad}$ from $\K^{ad}$ such that $\funres{\der^{ad}}_{\mid\Sigma}=\funres{\der}$ and $\funres{\der^{ad}}$ is \oad-free.
    \begin{description}
        \item[$n=0$] If $\der=(\emptyset, \FB)$, we take $\der^{ad}=\der$ and get the result.
        \item[$n>0$] By applying the induction hypothesis to $\der_{|n-1}$, there is a derivation $\der^{ad}_{n-1}$ from $\K^{ad}$ such that $\funres{\der^{ad}_{n-1}}_{\mid\Sigma}=\FB_{n-1}$ and $\funres{\der^{ad}_{n-1}}$ is \oad-free. Assume $t_n=(R, \pi)$, with $R=B\to H_1\wedge\dots\wedge H_k$. 
        
        Then, we prove that $t=(B\to X_R, \pi)$ is applicable on $\funres{\der^{ad}_{n-1}}$, in three steps:
        \begin{itemize}
            \item Since $t_n$ is applicable on $\FB_{n-1}$, $\pi(B)\subseteq \FB_{n-1}=\funres{\der^{ad}_{n-1}}_{\mid\Sigma}$, so $\pi(B)\subseteq\funres{\der^{ad}_{n-1}}$.
            \item If there is a retraction $\sigma$ from $\funres{\der^{ad}_{n-1}}\cup\ens{\pi^R(X_R)}$ to $\funres{\der^{ad}_{n-1}}$, then due to the fact that $\funres{\der^{ad}_{n-1}}$ is \oad-free, there is a retraction $\sigma'$ from $\funres{\der^{ad}_{n-1}}\cup\ens{\sigma\circ\pi(H_i)\;|\;1\leq i\leq k}$ to $\funres{\der^{ad}_{n-1}}$. Then, $(\cdot_{\mid\Sigma})\circ\sigma'$ is a retraction from $(\funres{\der^{ad}_{n-1}}\cup\ens{\sigma\circ\pi(H_i)\;|\;1\leq i\leq k})_{\mid\Sigma}$ to $\funres{\der^{ad}_{n-1}}_{\mid\Sigma}$, i.e. from $\FB_{n-1}\cup\ens{\sigma\circ\pi(H_i)\;|\;1\leq i\leq k}$ to $\FB_{n-1}$, which contradicts the applicability of $t_n$ on $\FB_{n-1}$.
        \end{itemize}
        As such, $t$ is applicable on $\funres{\der^{ad}_{n-1}}$. Then, since $(\funres{\der^{ad}_{n-1}}\cup\ens{\pi^R(X_R)})_{\mid\Sigma}=\FB_{n-1}$, for every $i\leq k$, $\pi(H_i)\notin\funres{\der^{ad}_{n-1}}\cup\ens{\pi^R(X_R)}$. Thus, the triggers $t^{i}_n=(X_R\to H_i, \pi)$ are all applicable (as they feature Datalog rules). We thus define $\der^{ad}$ as the result of the following process: start with $\der^{ad}_{n-1}$, apply $t$, and every trigger $t^{i}_n$. Then, close the result under rules of the form $H'\to X_{R'}$ (with $R'=B'\to H'$). Thus, by construction, 
        \begin{flalign*}
            \funres{\der^{ad}}_{\mid\Sigma}&=\funres{\der^{ad}_{n-1}}_{\mid\Sigma}\cup\ens{\pi(H_i)\;|\;0\leq i\leq k}&\\
            &=\funres{\der^{ad}_{n-1}}_{\mid\Sigma}\cup\funout{t_n}&\\
            &=\FB_n&
        \end{flalign*}
        In addition, $\der^{ad}$ is \oad-free, as every $X_R$-predicate in $\funres{\der^{ad}_{n-1}}$ was saturated, and every rule that could have been applicable with an $X_R$-predicate in $\funres{\der^{ad}}\setminus\funres{\der^{ad}_{n-1}}$ is not, as every $H_i$ is in $\funres{\der^{ad}}$ already. 
    \end{description}
    As such, we can indeed create a derivation from $\K^{ad}$ with the properties we needed. We can now show that if \der is fair, then $\der^{ad}$ is too. First, since $\funres{\der^{ad}}_{\mid\Sigma}=\funres{\der}$ and by construction, $\funres{\der^{ad}}$ is the saturation of $\funres{\der}$ by rules of the form $H'\to X_R$, then by the fact that the two-way atomic decomposition is a conservative extension, it is a model of $\K^{ad}$, so $\der^{ad}$ is fair.

    Now, we show that it may gain sometimes-termination. Consider the rule set $\ens{(\ref{rule:2ad-nst-r-1}\text{--}\ref{rule:2ad-nst-r-5})}$ and the fact base $\ens{A(a)}$.
    \begin{align}
        A(x) \to~ &\exists y,z ~R(x,x,x)\wedge R(x,y,z) \tag{\ref{rule:2ad-nst-r-1}} \\
        R(x,y,z) &\to R(x,x,t) \tag{\ref{rule:2ad-nst-r-2}} \\
        R(x,x,y) &\to S(x,y,z) \tag{\ref{rule:2ad-nst-r-3}} \\
        R(x,x,y)\wedge S(x,y,z) &\to S(x,x,x) \tag{\ref{rule:2ad-nst-r-4}} \\
        A(x) \wedge S(x,x,y) &\to A(y) \tag{\ref{rule:2ad-nst-r-5}}
    \end{align}
    This rule set does not terminate w.r.t the restricted chase: the sequence $\eqref{rule:2ad-nst-r-1},\eqref{rule:2ad-nst-r-3},\eqref{rule:2ad-nst-r-4},\eqref{rule:2ad-nst-r-5}$ can be repeated indefinitely. $\eqref{rule:2ad-nst-r-4}$ can happen at anytime, but it does not matter.

    But this rule set terminates after normalisation: in the following, for rule $B\to H$ numbered $n$, we use the following notation:
    \begin{itemize}
        \item $n_X$ is the rule $B\to X_R$.
        \item $n_H$ is the rule $X_R\to H$ when the head is atomic.
        \item $n_{H_i}$ is the rule $X_R\to H_i$ when the head is non-atomic.
        \item $n_\leftarrow$ is the rule $H\to X_R$.
    \end{itemize}
    Then, the terminating sequence is: $\ref{rule:2ad-nst-r-1}_X,\ref{rule:2ad-nst-r-1}_{H_2},\ref{rule:2ad-nst-r-2}_X,\ref{rule:2ad-nst-r-2}_H,\ref{rule:2ad-nst-r-3}_X,\ref{rule:2ad-nst-r-3}_H,\ref{rule:2ad-nst-r-4}_X,\ref{rule:2ad-nst-r-4}_H,\ref{rule:2ad-nst-r-1}_{H_1},\ref{rule:2ad-nst-r-3}_\leftarrow$ then close that by $n_\leftarrow$.
\end{proof}


\subsection{Proof of Theorem~\ref{thm:2adDfR}}

\begin{reptheorem}{thm:2adDfR}
    The two-way atomic decomposition preserves the termination of the Datalog-first restricted chase.
\end{reptheorem}

\begin{proof}[Sketch]
    For this proof by contrapositive, we will write datalog-first derivations as an alternation of a trigger that introduces an existential variable, and a derivation closed under Datalog rules. This way, we start with a derivation from \tadR, and construct a derivation from \setR that has the same number of existential triggers. We then show that if the first one is fair, the second one is too. Then, if there is an infinite fair derivation from \tadR, we can construct one from \setR, showing the result.
\end{proof}

Let \setR be a rule set and $\FB$ be a factbase. We discriminate the rules in \setR into the ones that introduce at least one existential variable, $\setR^{\exists}$, and the Datalog ones, $\setR^{D}$. An \emph{existential trigger} is triggers that features a non-Datalog rule, and for a factbase $\FB$ and a rule set $R$, the \emph{Datalog closure of $\FB$ (under \setR)} is the derivation that applies every single applicable Datalog rule on $\FB$ until there is none left.

In the following we use the fact that any Datalog-first derivation \der can be decomposed following this schema: $(\emptyset, \FB), \der_0, (t_1, \FB_1), \der_1, \ldots, (t_n, \FB_n), \der_n,\ldots$ where forall $i$, $t_i$ is an existential trigger, and $\der^{ad}_i$ is the Datalog closure of $\FB_i$.

\begin{lemma}\label{lem:2adDfR}
    Let \der be a \DF{\R}-derivation from $\langle\setR, \FB\rangle$ and $\der^{ad}$ a \DF{\R}-derivation from $\langle\tadR, \FB'\rangle$. Assume that \der and $\der^{ad}$ are closed under Datalog, and that there is an isomorphism $h'$ from \funres{\der} to $\funres{\der^{ad}}_{\mid\Sigma}$. Let $R=B\to H$ a rule and $R^{ad}=B\to X_R\in\tad(R)$. Then
    \begin{enumerate}
        \item If $t=(R, \pi)$, is applicable on \der, define $t^{ad}=(R^{ad}, \phi)$, with $\phi=h'\circ\pi$. Then, $t^{ad}$ is applicable on $\der^{ad}$.
        \item If $t^{ad}=(R^{ad}, \phi)$ is applicable on $\der^{ad}$, define $t=(R, \pi)$, with $\pi=h'^{-1}\circ\phi$. Then, $t$ is applicable on \der.
    \end{enumerate}
    In either cases, set $\der_f$ and $\der_f^{ad}$ the Datalog closures of $\funres{\der}\cup\funout{t}$ and $\funres{\der^{ad}}\cup\funout{t^{ad}}$, respectively. We can then extend $h'$ to an isomorphism $h$ from \funres{\der_f} to $\funres{\der_f^{ad}}_{\mid\Sigma}$.
\end{lemma}

\begin{proof}
    Before anything else, we need to extend $h'$ to a bijection from $\EI{\Vars}{\funres{\der}\cup\funout{t}}$ to $\EI{\Vars}{\funres{\der^{ad}}\cup\funout{t^{ad}}}$. To do so, for every existential variable $z$ in $R$, we define $h(\pi^R(z))=\phi^R(z)$. Note that this definition makes sense regardless of the case we are in, and extends the equality $h\circ\pi=\phi$. In addition, $h$ is indeed a bijection (but not yet an isomorphism).

    We now prove the two points one after the other.
    \begin{itemize}
        \item To prove point 1, we proceed by contrapositive. First, assume $t^{ad}$ is not applicable on $\der^{ad}$. Then, there is a retraction from $\funres{\der^{ad}}\cup\funout{t^{ad}}$ to $\funres{\der^{ad}}$. Let us call it $\sigma^{ad}$, and consider $\sigma=h^{-1}\circ\sigma^{ad}\circ h$. We then show that $\sigma$ is a retraction from $\funres{\der}\cup\funout{t}$ to $\funres{\der}$.

        First, let us show that $\sigma$ is the identity on $\funres{\der}$. Let $u\in\EI{\Vars}{\funres{\der}}$. Then $h(u)=h'(u)$, and thus $h(u)\in\EI{\Vars}{\funres{\der^{ad}}}$. So, since $\sigma^{ad}$ is a retraction, we have the following:
        \begin{flalign*}
            && \sigma^{ad}\circ h(u)&=h(u)&\\
            &\text{As such,} & h^{-1}\circ\sigma^{ad}\circ h(u) &= h^{-1}\circ h(u)&\\
            &\text{i.e.} & \sigma(u) &= u&
        \end{flalign*}
    
        Thus, $\sigma$ is the identity on $\funres{\der}$. Let us now prove that $\sigma(\funres{\der}\cup\funout{t})=\funres{\der}$. Using our last result proven, we only have to show that $\sigma(\funout{t})\subseteq\funres{\der}$. Since $\sigma^{ad}$ is a retraction, $\sigma^{ad}\circ\phi^R(X_R)\in\funres{\der^{ad}}$ (with $X_R[\vec{x}]$ the only atom produced by $t^{ad}$). Since $\der^{ad}$ is closed under datalog, every rule of the form $X_R\to H_i$ is satisfied in $\der^{ad}$. 
        \begin{flalign*}
            &\text{Thus,} &\forall i\dotq & \sigma^{ad}\circ\phi^R(H_i)\in\funres{\der^{ad}}&\\
            &\text{Equivalently,} &\forall i\dotq & \sigma^{ad}\circ h\circ\pi^R(H_i)\in\funres{\der^{ad}}&\\
            &\text{Since } h=h'\text{ on } \funres{\der}\text{, }&\forall i\dotq & h^{-1}\circ\sigma^{ad}\circ h\circ\pi^R(H_i)\in\funres{\der}&\\
            &\text{i.e.}& \forall i\dotq& \sigma\circ\pi^R(H_i)\in\funres{\der}&
        \end{flalign*}
        Since every atom in $\funout{t}$ is of the form $\pi^R(H_i)$, $\sigma$ is indeed a retraction from $\funres{\der}\cup\funout{t}$ to $\funres{\der}$, which means $t$ is not applicable on \der. We thus have shown point $1$.

        \item The proof of point $2$ is very similar to point $1$, by contrapositive. Assuming that $t$ is not applicable on \der, we can find $\sigma$ a retraction from $\funres{\der}\cup\funout{t}$ to $\funres{\der}$. We then construct $\sigma^{ad}=h\circ\sigma\circ h^{-1}$. Again, we want to show that $\sigma^{ad}$ is a retraction from $\funres{\der^{ad}}\cup\funout{t^{ad}}$ to $\funres{\der^{ad}}$.

        We start by proving $\sigma^{ad}$ is the identity on $\funres{\der^{ad}}$. Let $u$ be an element of \EI{\Vars}{\funres{\der^{ad}}}. As such, $h^{-1}(u)\in\EI{\Vars}{\funres{\der}}$. We then use the fact that $\sigma$ is a retraction to get that $\sigma\circ h^{-1}(u)=h^{-1}(u)$, which leads us to $\sigma(u) = u$ and proves this partial result.
    
        As before, we then show that $\sigma^{ad}(\funout{t^{ad}})\subseteq\funres{\der^{ad}}$. First note that $\funout{t^{ad}}=\ens{\phi^{R}(X_R)}$. We thus only need to show that $\sigma^{ad}\circ\phi^{R}(X_R)\in\funres{\der^{ad}}$. Since $\sigma$ is a retraction, 
        \begin{flalign*}
            &&\forall i\dotq & \sigma\circ\pi^{R}(H_i)\in\funres{\der}&\\
            &h \text{ is a bijection, so} &\forall i\dotq & \sigma\circ h^{-1}\circ\phi^{R}(H_i)\in\funres{\der}&\\
            &\text{Since } h=h'\text{ on } \funres{\der}\text{, }&\forall i\dotq & \sigma^{ad}\circ\phi^{R}(H_i)\in\funres{\der^{ad}}&
        \end{flalign*}
        Due to the fact that $\der^{ad}$ is closed under Datalog, the rule $\bigwedge_i H_i\to X_R$ is satisfied in $\funres{\der^{ad}}$. As such, $\sigma^{ad}\circ\phi^{R}(X_R)\in\funres{\der^{ad}}$, which proves that $\sigma^{ad}$ is a retraction from $\funres{\der^{ad}}\cup\funout{t^{ad}}$ to $\funres{\der^{ad}}$, and point $2$.
    \end{itemize}

    We can finally prove that $h$ is an isomorphism from \funres{\der_f} to $\funres{\der_f^{ad}}_{\mid\Sigma}$. We will do both directions successively.
    \begin{itemize}
        \item We must first prove that $h$ is a homomorphism from \funres{\der_f} to $\funres{\der_f^{ad}}_{\mid\Sigma}$. Let $A$ be an atom in \funres{\der_f}. We now want to prove that $h(A)\in\funres{\der_f^{ad}}_{\mid\Sigma}$. 
        
        If $A\in\funres{\der}$, then since $h'$ is an isomorphism and $h$ its extension, $h(A)\in\funres{\der^{ad}}_{\mid\Sigma}\subseteq\funres{\der_f^{ad}}_{\mid\Sigma}$, so we have what we want. 
        
        Otherwise, if $A\in\funres{\der_f}\setminus\funres{\der}$, we must distinguish two cases:
        \begin{itemize}
            \item If $A$ has been produced by a rule in $\setR^\exists$: since the only existential trigger applied between \der and $\der_f$ is $t$, then there is an $i$ such that $A=\pi^{R}(H_i)$. Since $t^{ad}$ has been applied in $\der_f^{ad}$, $\phi^{R}(X_R)\in\funres{\der_f^{ad}}$. Then, since $\der_f^{ad}$ is closed under Datalog, the rule $X_R\to H_i$ has been applied. As such, $\phi^{R}(H_i)\in\funres{\der_f^{ad}}$, but $\phi^{R}(H_i)=h(\pi^{R}(H_i))=h(A)$ so $h(A)\in\funres{\der_f^{ad}}_{\mid\Sigma}$.
            \item If $A$ has been produced by a rule in $\setR^D$: In the restricted chase, the only way for a Datalog rule not to fire is for its output to already be there. Thus, if the trigger that produced $A$ fired and the rule that produced it is $R'=B'\to H'$, then the rule $B'\to X_{R'}$ and every rule $X_{R'}\to H'_i$ either will fire or already has its output in the factbase. Note that either of those two cases yield the same atoms, including $A$.
        \end{itemize}
        As such, $h$ is a homomorphism from \funres{\der_f} to $\funres{\der_f^{ad}}_{\mid\Sigma}$.

        \item We then have to prove that $h^{-1}$ is a homomorphism from $\funres{\der_f^{ad}}_{\mid\Sigma}$ to \funres{\der_f}. Let $A$ be an atom in $\funres{\der_f^{ad}}_{\mid\Sigma}$. Note that this excludes any atom using an $X_{R'}$ predicate. We want to show that $h^{-1}(A)\in\funres{\der_f}$. Again, if $A\in\funres{\der^{ad}}_{\mid\Sigma}$, then since $h'$ is an isomorphism, $h^{-1}(A)\in\funres{\der_f}$, so we assume that $A\in(\funres{\der_f^{ad}}\setminus\funres{\der^{ad}})_{\mid\Sigma}$. We again distinguish two cases:
        \begin{itemize}
            \item If $A$ has been produced by a rule in $\setR^\exists$: First note that the only existential trigger applied between $\der^{ad}$ and $\der_f^{ad}$ is $t^{ad}$. Then, since we consider $A\in\funres{\der_f^{ad}}_{\mid\Sigma}$, $A$ cannot use an $X_R$ predicate. As such, there is an $i$ such that $A=\phi^{R}(H_i)=h(\pi^{R}(H_i))$. As $t$ has been applied in $\der_f$, $\pi^{R}(H_i)\in\funres{\der_f}$, i.e. $h^{-1}(A)\in\funres{\der_f}$.
            \item If $A$ has been produced by a rule in $\setR^D$: Again, nothing prevents Datalog rules from being applicable in the restricted chase. Thus, if a rule in $\tad{R'=B'\to H'}$ has produced $A$, then $R'$ will be applicable too at some point between \der and $\der_f$.
        \end{itemize}
        Thus, $h^{-1}$ is also a homomorphism, which concludes the proof.\qedhere
    \end{itemize}
\end{proof}

\noindent We can now prove the theorem.

\begin{proof}
    First, we deal with termination. Let $\der^{ad}$ be a \DF{\R}-derivation from $\K^{ad}=\langle\tadR, \FB\rangle$. We show by induction over the number of existential triggers in $\der^{ad}$ that we can construct a derivation \der from $\K=\langle\setR, \FB\rangle$ that has the same number of existential triggers and such that there is an isomorphism from \funres{\der} to $\funres{\der^{ad}}_{\mid\Sigma}$.
    \begin{description}
        \item[$n=0$] If $\der^{ad}=(\emptyset, \FB),\der^{ad}_0$ with $\der^{ad}_0$ the Datalog closure of $\FB$ in \tadR. We can now consider $\der=(\emptyset, \FB),\der_0$ the derivation such that $\der_0$ is the Datalog closure of $\FB$ in \setR. Define $h=id_{\EI{\Vars}{\funres{\der}}}$. It is indeed a bijection, and even an isomorphism, since nothing in the restricted chase can prevent Datalog rules from firing. One can see (by induction on the number of Datalog rules applied) that any atom produced by one derivation will indeed be produced by the other.
        
        \item[$n>0$] If $\der^{ad}=(\emptyset, \FB),\der^{ad}_0,\ldots,(t^{ad}_{n+1}, \FB^{ad}_{n+1}), \der^{ad}_{n+1}$, then we can use the induction hypothesis on $\der^{ad}_n$ to construct $\der=(\emptyset, \FB),\der_0,\ldots,(t_n, \FB_n), \der_n$ and $h'$ an isomorphism from \funres{\der} to $\funres{\der_n^{ad}}_{\mid\Sigma}$. Since \der and $\der_n^{ad}$ are closed under Datalog and $t^{ad}_{n+1}$ is applicable on $\der_n^{ad}$, we can apply point 2 of Lemma~\ref{lem:2adDfR} to construct $\der'=(\emptyset, \FB),\der_0,\ldots,(t_{n+1}, \FB_{n+1}), \der_{n+1}$ and $h$ an isomorphism from \funres{\der'} to $\funres{\der^{ad}}_{\mid\Sigma}$.
    \end{description}
    We thus created a derivation \der that shares the same number of existential triggers as $\der^{ad}$, and such that $h$ is an isomorphism between the two.
    
    We still have to prove that if $\der^{ad}$ is fair then \der is too. Assume that $\der^{ad}$ is fair. As such, \funres{\der^{ad}} is a model of $\K^{ad}$. Thus, since $h$ is an isomorphism from \funres{\der} to $\funres{\der^{ad}}_{\mid\Sigma}$ and according to Proposition~\ref{prop:adce}, \funres{\der} is a model of \setR. Thus, \der is fair.

    As such, if the Datalog-first restricted chase does not terminate on $\K^{ad}$, we can find an infinite fair derivation from $\K^{ad}$, and from this we can construct an infinite fair derivation from \K, showing that the Datalog-first restricted chase does not terminate on \K. By contrapositive, the two-way atomic decomposition preserves the termination of the Datalog-first restricted chase.
    
    Let us now tackle non-termination. Let $\K=\langle\setR, \FB\rangle$ be a knowledge base on which the Datalog-first restricted chase sometimes terminates, and \der a terminating derivation from \K. Let $\K^{ad}=\langle\tadR, \FB\rangle$. Using point 1 of Lemma~\ref{lem:2adDfR}, we can construct a derivation $\der^{ad}$ from $\K^{ad}$ and $h$ an isomorphism between the two (the induction is almost identical as the one in Theorem~\ref{thm:2adDfR}). As such, since \der is fair and $h$ is an isomorphism between their results, and using the fact that the two-way atomic decomposition produces conservative extensions, we show that $\der^{ad}$ is fair. As such, the Datalog-first restricted chase is sometimes-terminating on $\K^{ad}$. Point 2 of Lemma~\ref{lem:2adDfR} proves the other direction similarily.
\end{proof}


\subsection{Proof of Theorem~\ref{thm:2adBDDP}}

\begin{reptheorem}{thm:2adBDDP}
    A rule set \setR is BDDP iff \tad(\setR) is BDDP.  
\end{reptheorem}

Similarily to how we defined $\setR_X$ previously, we define $\setR_X^{-1}$ as $\tadR\setminus\oadR$, or alternatively: \[\setR_X^{-1}\ens{H\to X_R[\vec{y}]\mid R=B\to H\in\setR}\]

\begin{proposition}\label{prop:RTo2ad}
    For any $\FB$, for any rule set $\setR$, for any integer $i$, it holds that
    
    \[\funchase{i}{\FB,\setR} \subseteq \funchase{2i}{\FB,\tadR}\]
\end{proposition}

\begin{proof}
    The proof is the exact same as the proof of Proposition~\ref{prop:RTo1ad}, the additional rule of the two-way atomic decomposition makes no difference.
\end{proof}

\begin{proposition}\label{prop:2adToR}
    For any $\FB$, for any rule set $\setR$, for any integer $i$, it holds that
    \[\funchase{2i}{\FB,\tadR} \subseteq \funchase{1}{\funchase{i}{\funchase{1}{\FB,\tadR}_{\mid\Sigma},\setR},\setR_X\cup\setR_X^{-1}} \cup (\FB \setminus \FB_{\mid \Sigma})\]
\end{proposition}

\begin{proof}
    We show the result by induction on $i$.
    
    For $i=0$, $\funchase{0}{\FB,\tadR} = \FB \subseteq \funchase{1}{\funchase{1}{\FB,\tadR}_{\mid\Sigma},\setR_X\cup\setR_X^{-1}} \cup (\FB \setminus \FB_{\mid \Sigma})$, as $\FB_{\mid \Sigma}$ is included in the first part of the union.
    
    Assume the result to be true for $i$. Again, atoms of $\funchase{2i+2}{\FB,\tadR} \setminus \funchase{2i}{\FB,\tadR}$ can be generated in the following way:
    \begin{enumerate}
        \item\label{item2ad-rx-2i} By applying a rule of $\setR_X\cup\setR_X^{-1}$ on $\funchase{2i}{\FB,\tadR}$.
        \item\label{item2ad-rset} By applying a rule of $\tadR \setminus (\setR_X\cup\setR_X^{-1})$.
        \item\label{item2ad-rx-2i+1} By applying a rule of $\setR_X\cup\setR_X^{-1}$ on $\funchase{2i+1}{\FB,\tadR}$.
    \end{enumerate}
    
    If an atom $A$ is generated by Case~\ref{item2ad-rx-2i}, by induction hypothesis, the body of the rule applied is in $\funchase{1}{\funchase{i}{\funchase{1}{\FB,\tadR}_{\mid\Sigma},\setR},\setR_X\cup\setR_X^{-1}}\cup (\FB \setminus \FB_{\mid \Sigma})$. The body of the rule uses only predicates in $\Sigma$, so it is in $\funchase{i}{\funchase{1}{\FB,\tadR}_{\mid\Sigma},\setR}$. Thus, it is also in $\funchase{i+1}{\funchase{1}{\FB,\tadR}_{\mid\Sigma},\setR}$, and the trigger that introduced $A$ in $\funchase{2i+1}{\FB,\tadR}$ is applicable on $\funchase{i+1}{\funchase{1}{\FB,\tadR}_{\mid\Sigma},\setR}$, so $A\in\funchase{1}{\funchase{i+1}{\funchase{1}{\FB,\tadR}_{\mid\Sigma},\setR},\setR_X\cup\setR_X^{-1}}$.
    \vspace{0.2cm}
    
    If an atom $A$ is generated by Case~\ref{item2ad-rset}, then the trigger that created it is of the form $(X_R\to H_i,\pi)$, with $A=\pi(H_i)$. We distinguish two cases:
    \vspace{-0.2cm}
    \begin{itemize}
        \item Either $\pi(X_R)\in \FB$. Then, $A\in\funchase{1}{\FB,\tadR}_{\mid\Sigma}$, so $A$ is in the set we consider.
        \item Or $\pi(X_R)\notin \FB$. Then, let $(B\to X_R, \pi')$ be the rule that introduced $\pi(X_R)$ (so the rule such that $\pi=\pi'^R$).
        Then, by induction hypothesis and the fact that $B$ only features predicates in $\Sigma$, $\pi'(B)\in\funchase{i}{\funchase{1}{\FB,\tadR}_{\mid\Sigma},\setR}$. 
        Thus, the trigger $(B\to H,\pi')$ is applicable on $\funchase{i}{\funchase{1}{\FB,\tadR}_{\mid\Sigma},\setR}$, so $\pi'^R(H)\subseteq\funchase{i+1}{\funchase{1}{\FB,\tadR}_{\mid\Sigma},\setR}$, and since $A\in\pi(H)$ and $\pi=\pi'^R$, $A$ is again in the set we want it to be.
    \end{itemize}
    
    If an atom $A$ is generated by Case~\ref{item2ad-rx-2i+1}, then the body of the rule used to generate $A$ has been generated using a rule in $\tadR \setminus (\setR_X\cup\setR_X^{-1})$. Thus, by the previous cases, the body of this rule is in $\funchase{i+1}{\funchase{1}{\FB,\tadR}_{\mid\Sigma},\setR}$, so $A\in\funchase{1}{\funchase{i}{\funchase{1}{\FB,\tadR}_{\mid\Sigma},\setR},\setR_X\cup\setR_X^{-1}}$.
    \vspace{0.2cm}
    
    As such, we indeed have the inclusion.
\end{proof}

\begin{proposition}\label{prop:2adBDDPRTo2ad}
    For any rule set $\setR$ having the bounded derivation depth property, $\tadR$ has the bounded derivation depth property.
\end{proposition}

\begin{proof}
    Let $\setR$ be a rule set that has the BDDP, $q$ a query and $\FB$ a factbase such that $\langle\tadR,\FB\rangle\models q$. Then, there is a $i$ such that
    \[\funchase{2i}{\FB,\tadR}\models q\]
    Thus, by Proposition~\ref{prop:2adToR},
    \[\funchase{1}{\funchase{i}{\funchase{1}{\FB,\tadR}_{\mid\Sigma},\setR},\setR_X\cup\setR_X^{-1}} \cup (\FB \setminus \FB_{\mid\Sigma})\models q\]
    We use the same technique we used in Theorem~\ref{prop:1adBDDPRTo1ad}. Let $\pi$ be a homomorphism witnessing this entailment, and $\hat{q}$ the subset of $q$ containing exactly the atoms that $\pi$ does not map in $\FB\setminus \FB_{\mid\Sigma}$, with as answer variables the terms in both $\hat{q}$ and $q\setminus\hat{q}$. By this definition, $\pi$ is a match of $\hat{q}$ in $\funchase{1}{\funchase{i}{\funchase{1}{\FB,\tadR}_{\mid\Sigma},\setR},\setR_X\cup\setR_X^{-1}}$. Consider $\sigma_{\hat{q}}$ the substitution that maps $\hat{q}$'s answer variables to their images by $\pi$ and other variables to themselves. Then, since $\setR_X\cup\setR_X^{-1}$ is acyclic, by rewriting, there is a finite set of conjunctive queries $\hat{\mathfrak{Q}}$ and $q'\in \hat{\mathfrak{Q}}$ such that:
    \[\funchase{i}{\funchase{1}{\FB,\tadR}_{\mid\Sigma},\setR}\models \sigma_{\hat{q}}(q')\]
    As $\funchase{1}{\FB,\tadR}_{\mid\Sigma}$ is on $\Sigma$, $\funchase{i}{\funchase{1}{\FB,\tadR}_{\mid\Sigma},\setR}$ and $q'$ are too, so by the BDDP, there is a $k_q$ independant of $\FB$ such that
    \[\funchase{k_q}{\funchase{1}{\FB,\tadR}_{\mid\Sigma},\setR}\models \sigma_{\hat{q}}(q')\]
    Using Proposition~\ref{prop:RTo2ad}, we get
    \[\funchase{2k_q}{\funchase{1}{\FB,\tadR}_{\mid\Sigma},\tadR}\models \sigma_{\hat{q}}(q')\]
    Then, as $\funchase{1}{\FB,\tadR}_{\mid\Sigma}\subseteq\funchase{1}{\FB,\tadR}$ and $\FB\setminus \FB_{\mid\Sigma}\subseteq\funchase{2k_q+1}{\FB,\tadR}$,
    \[\funchase{2k_q+1}{\FB,\tadR}\models q\]
    which concludes the proof.
\end{proof}

\begin{proposition}\label{prop:2adBDDP2adToR}
    For any rule set $\setR$, if $\tadR$ has the bounded derivation depth property, then so do $\setR$.
\end{proposition}

\begin{proof}
    Assume that \setR does not have the BDDP. Then, there is $q$ a query and $\ens{\FB_i}_{i\in\mathbb{N}}$ a family of factbases such that for any $i$, $\langle\setR, \FB\rangle\models q$ and $\funchase{i}{\FB_i, \setR}\not\models q$.
    
    Since, in a similar fashion to the proof of Theorem~\ref{prop:1adBDDP1adToR}, for any $\FB_i$ on $\Sigma$, 
Since, in a similar fashion to the proof of Theorem~\ref{prop:1adBDDP1adToR}, for any $\FB_i$ on $\Sigma$, 
    Since, in a similar fashion to the proof of Theorem~\ref{prop:1adBDDP1adToR}, for any $\FB_i$ on $\Sigma$, 
    \[\funchase{i}{\FB_i,\setR}=\funchase{2i}{\FB_i,\tadR}_{\mid\Sigma}\]
    then
    \[\funchase{2i}{\FB_i,\tadR}\not\models q\]
    which concludes the proof.
\end{proof}

Then, Theorem~\ref{thm:sptBDDP} is exactly Proposition~\ref{prop:2adBDDPRTo2ad} and Proposition~\ref{prop:2adBDDP2adToR}.
\section{Proofs of Section~\ref{sec:noNormR}}

\subsection{Proof of Proposition~\ref{prop-single-head-re}}
\begin{lemma}\label{lem:derBdd}
If $\setR \in \chaseterminst{\R}{\FB}{\forall}$, then there is some $k$ such that $\vert \der \vert \leq k$ for each $\R$-derivation \der\ from $\FB$.
\end{lemma}

\begin{proof}
Step-by-step argument:
\begin{enumerate}
\item Suppose for a contradiction that $\setR \in \chaseterminst{\R}{\FB}{\forall}$ and that, for each $k \geq 0$, there is some $\R$-derivation \der\ from $\Tuple{\setR, \FB}$ such that $\vert \der \vert \geq k$.
\item Consider the graph $G = (V, E)$ such that
\begin{itemize}
\item $V$ is the set of all factbases that occur in some \R-derivation from $\Tuple{\setR, \FB}$.
\item For each $\setG, \setG' \in V$, we have $\setG \to \setG' \in E$ if there is some trigger $t = (R, \pi)$ such that $R \in \setR$, $t$ is \R-applicable to \setG, and $\setG' = \setG \cup \funout{t}$.
\end{itemize}
\item By (2): the degree of each node in $G$ is finite.
\item By (2) and (3): by K\"onig's lemma, the graph $G$ features a simple path that is infinite.
\item By (2) and (4): since all vertices can be reached from its only root (i.e., \setF), there is a simple infinite path $\setF_0 = \setF, \setF_1, \setF_2, \ldots$ in $G$.
This infinite path corresponds to an infinite $\R$-derivation from \K.
\item By (5) and \cite{DBLP:conf/pods/GogaczMP20}, which states that for single-head rules, there exists an infinite $\R$-derivation if and only if there exists an infinite fair $\R$-derivation: $\FB$ admits an $\R$-derivation that is infinite and fair.
\item By (1) and (6): contradiction.\qedhere
\end{enumerate}
 \end{proof}

\begin{repproposition}{prop-single-head-re}
For any factbase $\FB$, the subset of $\chaseterminst{\R}{\FB}{\forall}$  containing only atomic-head rules is recognizable.
\end{repproposition}

\begin{proof}
    Step-by-step argument:
    \begin{enumerate}
        \item Consider the graph $G$ defined as in point $2$ of the proof of Lemma~\ref{lem:derBdd}.
        \item By Lemma~\ref{lem:derBdd}, the depth of $G$ is bounded exactly if $\setR \in \chaseterminst{\R}{\FB}{\forall}$.
        \item By definition, the degree of each node in $G$ is finite.
        \item By (2) and (3), $G$ is thus finite exactly if $\setR \in \chaseterminst{\R}{\FB}{\forall}$.
        \item From a factbase $\FB'$ in the graph, we can follow the edge between $\FB'$ and $\FB''$ by applying the trigger that yields $\FB''$ from $\FB'$.
        \item By (5), we can thus do a breadth-first search on the graph.
        \item By (4), the breadth-first search will terminate if and only if $\setR \in \chaseterminst{\R}{\FB}{\forall}$.\qedhere
    \end{enumerate}
\end{proof}

\subsection{Proof of Proposition~\ref{prop-hardness-termination} and Theorem~\ref{thm-no-nf-restricted}}
We will use (deterministic) Turing machines (TM), denoted as a tuple $\TM=\Tuple{\States, \Alphabet, \TransitionFunction}$,
with states $\States$, tape alphabet $\Alphabet$ with blank $\Blank\in\Alphabet$, and transition function $\TransitionFunction$.
$\TM$ has a distinguished initial state $\StartingState\in\States$, and accepting and rejecting halting states
$\AcceptingState,\RejectingState\in\States$. For all states $q\in\States\setminus \{\AcceptingState, \RejectingState\}$
and tape symbols $a\in\Alphabet$,
there is exactly one transition $(q, a)\mapsto (r, b, D) \in \TransitionFunction$, where $D$ can either be $L$ or $R$.
We assume that TM tapes are unbounded to the right but bounded to the left, and
that TMs will never attempt to move left on the first
position of the tape (this is w.l.o.g., since one can modify any TM to insert a marker at the tape start
to recognise this case).

We prove that given a Turing machine $\TM$, we can build a rule set $\setR_w \cup \setR_M$ and $\FB$ such that every restricted chase sequence from $\Tuple{\FB,\setR_w \cup \setR_M}$ is finite if and only if $\TM$ halts on every input. We can restrict ourselves w.l.o.g. to the case where the input alphabet of $\TM$ is unary. The construction works as follows:
\begin{itemize}
 \item build a rule set $\setR_w$ that generates representations of input tapes of length up to $k$, for arbitrary $k$. In order to ensure that restricted chase sequences are finite, we make use of the emergency brake technique from \cite{DBLP:conf/icdt/KrotzschMR19}
 \item build a rule set $\setR_M$ that simulate the run of a Turing machine in a terminating way if $\TM$ halts. This is classical , and we reuse a rule set provided in \cite{DBLP:conf/kr/BourgauxCKRT21}, recalled for self-containedness.
\end{itemize}

Let us consider $\FB$ containing the following atoms:
  \begin{enumerate}
 \item $\First(c_0^1), \Content{1}(c_0^1),\Next(c_0^1,c_1^1), \End(c_1^1), \Content{\Blank}(c_1^1)$, $\First(c_0^0), \End(c_0^0), \Content{\Blank}(c_0^0)$
 \item $\Init(a), \NonFinal(a,nf_1), \Real(nf_1), \NonFinal(nf_1,b), \Done(nf_1,b)$
 \item $\Brake(b),\Final(b,b),\NonFinal(b,b),\Done(b,b), \Next(b,b),\Last(b),\First(b)$
 \item $\HeadState{s}(b),\Content{l}(b),\End(b),\Step(b,b),\NextPlus(b,b)$
\end{enumerate}

$\Brake$ stands for \emph{brake}, $\Real$ for \emph{real}, $\NonFinal$ for \emph{non-final}, $\Done$ for \emph{done}, $\Final$ for \emph{final}.

\begin{figure}
\begin{align}
\Brake(b) \wedge \NonFinal(z,x) \wedge \Real(x) &\to \exists y . \NonFinal(x,y) \wedge \Real(y) \wedge \Done(y,b) \wedge \NonFinal(y,b) \label{rule-intermediate}\\
\Brake(b) &\to \Real(b) \label{rule-brake} \\
\NonFinal(x, y) &\to \exists z . \Final(y, z) \label{rule-final}\\
\Final(x, y) &\to \exists z . \Done(y, z) \wedge \End(z) \wedge \Content{\Blank}(z) \label{rule-tape-end}\\
\NonFinal(t,x) \wedge \Final(x,y) \wedge \Done(y,z) &\to \exists u . \Next(u,z) \wedge \Done(x,u)
\wedge \Content{1}(u) \\
\NonFinal(t,x) \wedge \NonFinal(x,y) \wedge \Done(y,z) &\to \exists u . \Next(u,z) \wedge \Done(x,u) \wedge \Content{1}(u)\label{rule-tape-2}\\
\Init(x) \wedge \NonFinal(x,y) \wedge \Done(y,z) &\to \exists u . \Next(u,z) \wedge \Done(x,u) \wedge \Content{1}(u) \wedge \First(u) \label{rule-tape-1} \\
\First(x) &\to \HeadState{q_I}(x) \label{rule-head-init}
\end{align}
\caption{Rules $\setR_w$ to Create the Initial Tapes}
\label{rule-tape-creation-app}
\end{figure}

\begin{figure}
\begin{align}
 \Next(x,y) &\to \NextPlus(x,y) \label{rule-next-plus-init}\\
 \NextPlus(x,y) \wedge \NextPlus(y,z) &\to \NextPlus(x,z) \label{rule-transitive-next-plus}\\
 \Next(x,y) \wedge \Step(x,z) \wedge \Step(y,w) &\to \Next(z,w) \label{rule-next-step-next}\\
 \End(x) \wedge \Step(x,z) &\to \exists v. \Next(z,v) \wedge \Content{\Blank}(v) \wedge \End(v) \label{rule-extend-tape}\\
 \HeadState{q}(x) \wedge \NextPlus(x,y) \wedge \Content{c}(y) &\to \exists z . \Step(y,z) \wedge \Content{c}(z) \label{rule-right-inertia}\\
  \HeadState{q}(x) \wedge \NextPlus(y,x) \wedge \Content{c}(y) &\to \exists z . \Step(y,z) \wedge \Content{c}(z) \label{rule-left-inertia}\\
  \HeadState{q}(x) \wedge \Content{a}(x) &\to \exists z . \Step(x,z) \wedge \Content{b}(z) \label{rule-change-cell-content}\\
  \HeadState{q}(x) \wedge \Content{a}(x) \wedge \Step(x,z) \wedge \Next(z,w) &\to \HeadState{r}(w) \label{rule-change-head-state-right}\\
 \HeadState{q}(x) \wedge \Content{a}(x) \wedge \Step(x,z) \wedge \Next(w,z) &\to \HeadState{r}(w) \label{rule-change-head-state-left}
\end{align}
\caption{Rules $\setR_M$ for the Turing Machine Simulation: for a rule $(q, a)\mapsto (r, b, R) \in \TransitionFunction$, instantiations of Rules~(\ref{rule-change-cell-content}) and (\ref{rule-change-head-state-right}) are created; for a rule $(q, a)\mapsto (r, b, L) \in \TransitionFunction$,  instantiations of Rules~(\ref{rule-change-cell-content}) and (\ref{rule-change-head-state-left})}
\label{rule-tm-simulation-app}
\end{figure}

We claim that $\setR_w \cup \setR_M \in \chaseterminst{\R}{\FB}{\forall}$ if and only if $M$ halts on every input.

\begin{lemma}
\label{prop-tape-creation}
The result of any restricted chase sequence from $\Tuple{\setR_w,\FB}$ is of the following shape, for some $n$:

\begin{align*}
 &\FB   \\
 \cup& \quad \{\NonFinal(nf_i,nf_{i+1}),\Real(nf_{i+1}) \mid \{i \in \{1,\ldots,n\}),\Done(nf_{i+1},b),\NonFinal(nf_{i+1},b)\} \\
 \cup& \quad\{\Final(nf_{i},f_{i+1}) \mid i \in \{1,\ldots,n+1\}\} \\
  \cup& \quad\{\First(c_0^i),\Next(c_0^i,c_1^i),\Done(a,c_i^0),\Done(f_i,c_i^i),\End(c_i^i) \mid i \in \{2,\ldots, n+1\}\} \\
 \cup& \quad\{\Next(c_j^i,c_{j+1}^i),\Done(nf_j,c_j^i), \Content{1}(c_j^i)\mid i \in \{2,\ldots,n\}, j \in \{1,\ldots,i-1\}\} \\
 \cup& \quad\{\Real(b)\}
\end{align*}

Moreover, for any $n \geq 2$ there exists a restricted chase sequence whose result is described above.

\end{lemma}

\begin{proof}
 Let us first notice that a trigger of Rule~(\ref{rule-intermediate}) can only be blocked by the only possible application of Rule~(\ref{rule-brake}). Once Rule~(\ref{rule-brake}) has been applied, no trigger of Rule~(\ref{rule-intermediate}) is active, hence there are finitely many such triggers applied during any restricted fair derivation. This implies that any restricted chase sequence can be reordered in order to start with $m \in \mathbb{N}$ applications of Rule~(\ref{rule-intermediate}) followed by one application of Rule~(\ref{rule-brake}), followed by all the other rule applications in the original order. 
 
 As Rule~(\ref{rule-final}) can be applied exactly once per atom of predicate $\NonFinal$ (as they all have distinct second argument), and such a rule application cannot block any other rule, one can reorder the fair restricted derivation by applying all these triggers right after Rule~\ref{rule-brake}).
 
 The same reasoning applies to Rule~(\ref{rule-tape-end}), (\ref{rule-tape-2}) and (\ref{rule-tape-1}), which we thus order in that way.
 
 By choosing the following naming convention of the nulls, we obtain the claimed shape of the result:
 \begin{itemize}
  \item Rule~(\ref{rule-intermediate}) mapping $x$ to $nf_i$ instantiates $y$ to $nf_{i+1}$
  \item Rule~(\ref{rule-final}) mapping $y$ to $nf_i$ instantiates $z$ to $f_{i+1}$
  \item Rule~(\ref{rule-tape-end}) mapping $y$ to $f_i$ instantiates $z$ to $c_i^i$
  \item Rule~(\ref{rule-tape-2}) mapping $x$ to $nf_j$ and $z$ to $c_{j+1}^i$ instantiates $u$ to $c_{j}^i$
  \item Rule~(\ref{rule-tape-1}) mapping $x$ to $a$ and $z$ to $c_1^i$ instantiates $u$ to $c_0^i$.
 \end{itemize}

 Note that no rule (except for Rule~(\ref{rule-brake})) is ever applicable by mapping a frontier term to $b$: Rule~\ref{rule-intermediate} is blocked by $b$ whenever $R(b)$ is derived, which is necessary to map its body, while all the other rules are blocked due to atoms in $\FB$. 
 

\end{proof}

\begin{lemma}
\label{proposition-tape-creation-then-simulation}
 The result of any restricted chase sequence from $\Tuple{\FB,\setR_w\cup\setR_M}$ is isomorphic to the result of a restricted chase sequence where rules of $\setR_w$ are all applied before rules of $\setR_M$..
\end{lemma}

\begin{proof}
Let us notice that no rule of Figure~\ref{rule-tm-simulation} can either trigger or block a rule from Figure~\ref{rule-tape-creation-app} (due to stratification). As rules from Figure~\ref{rule-tape-creation-app} are applied finitely many times (direct consequence of Lemma~\ref{prop-tape-creation}), they can all be applied first while preserving fairness. 
\end{proof}

\begin{lemma}
\label{proposition-tm-simulation}
 The chase of $(\setR_M,Tp_i)$ where $Tp_i$ is defined as 
 
  \[\{\Next(c_j^i,c_{j+1}^i),\First(c_0^i), \Content{1}(c_j^i) \mid j \in \{0,\ldots,i-i\}\} \cup\{\End(c_i^i),\Content{\Blank}(c_i^i)\}\]
  
w.r.t. rules of Figure~\ref{rule-tm-simulation-app} is finite if and only if $M$ halts on the input of length $i \geq 1$.
\end{lemma}

\begin{proof}
 This is the same rule set at that used in \cite{DBLP:conf/kr/BourgauxCKRT21}, and designed specifically to simulate a terminating Turing machine run in a terminating way.
\end{proof}

%
%
%

\begin{lemma}
\label{proposition-union-chases}
 The result of any chase sequence w.r.t. rules of Figure~\ref{rule-tm-simulation} is $I_n$ union the chase of $Tp_k$ w.r.t. rules of Figure~\ref{rule-tm-simulation} for any $k \leq n$.
\end{lemma}


\begin{proof}
 Let us notice that the result of any restricted chase sequence $\der$ from $\Tuple{\setR_w,\FB}$ is of the shape $\FB' \bigcup_{i\in \{2,\ldots,n\}} Tp_i$ for some $n$, where $\FB'$ does not contain any predicate appearing in $\setR_w$. Moreover, each $Tp_i$ is in its own connected component of \funres{\der} w.r.t. the predicates that appear in $\setR_w$. As these rules are both body and head connected, with a non-empty frontier, rules are applied by mapping their frontier to one connected component, and thus cannot merge connected components, or prevent a trigger occuring in another connected component to be applied. The result of any chase sequence from $\Tuple{\setR_w,\funres{\der}}$ is equal to $\FB' \bigcup_{i \in \{2,\ldots,n\}} \funres{\der_i}$, where $\der_i$ is a restricted chase sequence from $\Tuple{\setR_M,Tp_i}$.
\end{proof}

We can finally conclude the proof. Let us assume that there the restricted chase does not terminate. By Lemma~\ref{proposition-union-chases}, this implies that there exists $i$ such that the chase from $\Tuple{I_i,\setR_\TM}$ does not terminate, which by Lemma~\ref{proposition-tm-simulation} implies that $M$ does not halt on the input of length $i$.

Conversely, let us assume that $\TM$ does not halt on the input of length $n$. We build an infinite fair restricted chase sequence. We first generate $Tp_n$ by rule applications of Figure~\ref{rule-tape-creation-app} (and one application of Rule~(\ref{rule-head-init}), which is possible by Lemma~\ref{prop-tape-creation}). At this step, no trigger of a rule of Figure~\ref{rule-tape-creation-app} is left active. By Lemmas~\ref{proposition-tm-simulation} and \ref{proposition-union-chases}, applying rules of Figure~\ref{rule-tm-simulation-app} halts if and only if $\TM$ halts on any input of length less or equal to $n$. Hence it does not halt, which concludes the proof.

\end{tr}

\end{document}